\documentclass[11pt]{article}
\usepackage{amsmath,amsfonts,amsthm}
\usepackage{graphicx}
\usepackage{epstopdf}
\usepackage{color}
\usepackage{float}
\usepackage{pifont}
\usepackage{bm}
\usepackage[letterpaper]{geometry}
\usepackage{algorithm}
\usepackage{algpseudocode}
\usepackage{thmtools}
\usepackage{thm-restate}
\usepackage{titletoc}
\usepackage{tikz}
\usetikzlibrary{automata, positioning, arrows}
\usepackage{stackengine,amssymb,scalerel}

\newcommand\CircArrowRight[1]{\stackengine{-.2ex}{\scalebox{.5}{#1}}{\CAR}{O}{c}{F}{F}{L}}
\newcommand\CAR{\scaleto{\circlearrowright}{1.8ex}}



\usepackage{endnotes}
\usepackage{setspace}
\usepackage{multirow}
\usepackage{url}
\usepackage{mathtools}

\usepackage{multicol}
\usepackage[table]{colortbl}
\usepackage{wrapfig}
\usepackage{footnote}
\makesavenoteenv{table}

\usepackage{natbib}
\allowdisplaybreaks 

\newtheorem{theorem}{Theorem}[section]

\newtheorem{claim}[theorem]{Claim}

\newtheorem{proposition}[theorem]{Proposition}
\newtheorem{lemma}[theorem]{Lemma}

\theoremstyle{definition}
\newtheorem{definition}[theorem]{Definition}

\newtheorem{assumption}[theorem]{Assumption}

\numberwithin{equation}{section}

\newcommand{\AAA}{\mathcal{A}}

\newcommand{\FFF}{\mathcal{F}}

\newcommand{\SSS}{\mathcal{S}}

\newcommand{\primal}{\mathsf{P}}
\newcommand{\dual}{\mathsf{D}}
\newcommand{\linearprimal}{\mathsf{lin}\text{-}\mathsf{P}}
\newcommand{\lineardual}{\mathsf{lin}\text{-}\mathsf{D}}
\newcommand{\rad}{\mathsf{rad}}

\newcommand{\FW}{\mathsf{FW}}
\newcommand{\TGD}{\mathsf{TGD}}
\newcommand{\TMD}{\mathsf{TMD}}
\newcommand{\refe}{\mathsf{ref}}

\title{Exploration-Exploitation Trade-off in Reinforcement Learning on Online Markov Decision Processes with Global Concave Rewards}
\author{Wang Chi Cheung\thanks{Department of Industrial Systems Engineering and Management, National University of Singapore, isecwc@nus.edu.sg}}
\begin{document}
\vspace{-6.5em}
\date{}
\maketitle
\vspace{-3em}

\begin{abstract}
We consider an agent who is involved in a Markov decision process and receives a vector of outcomes every round. Her objective is to maximize a global concave reward function on the average vectorial outcome. The problem models applications such as multi-objective optimization, maximum entropy exploration, and constrained optimization in Markovian environments. In our general setting where a stationary policy could have multiple recurrent classes, the agent faces a subtle yet consequential trade-off in alternating among different actions for balancing the vectorial outcomes. In particular, stationary policies are in general sub-optimal. We propose a no-regret algorithm based on online convex optimization (OCO) tools (Agrawal and Devanur 2014) and UCRL2 (Jaksch et al. 2010). Importantly, we introduce a novel gradient threshold procedure, which carefully controls the switches among actions to handle the subtle trade-off. By delaying the gradient updates, our procedure produces a non-stationary policy that diversifies the outcomes for optimizing the objective. The procedure is compatible with a variety of OCO tools. 
\end{abstract}

\vspace{-0.5cm}

\section{Introduction}
Markov Decision Processes (MDPs) model sequential optimization problems in Markovian environments. At each time, an agent performs an action, contingent upon the state of the environment. Her action influences the environment through the resulting state transition.  
In many situations, an action at a state leads a vector of different and correlated outcomes, and the agent desires to optimize a complex and global objective that involves all these outcomes across time. Motivated by these situations, we consider online \underline{MDP}s \underline{w}ith \underline{G}lobal concave \underline{R}eward functions (MDPwGR). In the MDPwGR problem, an agent seeks to optimize a  concave reward function, which is generally non-linear, over the average vectorial outcome generated by a latent MDP. 

For online optimization with global concave rewards and MDPwGR in particular, an agent is required to alternate among different actions in order to balance the vectorial outcomes.
The setting of MDPwGR presents the following subtle challenges.
To alternate between two actions, the agent has to travel from one state to another, which could require visiting sub-optimal states and compromises her objective. Thus, the alternations have to be carefully controlled, in order to balance the outcomes while maintaining near-optimality, on top of her simultaneous exploration and exploitation on the latent model.

We shed light on the mentioned trade-off by proposing {\sc Toc-UCRL2}, a near-optimal online algorithm for MDPwGR. The algorithm is built upon a dual based approach using gradient updates, which facilitate the balancing of outcomes, as well as UCRL2, which solves MDPs with certain scalar rewards. In order to handle the mentioned trade-off in action alternations, we introduce a novel gradient threshold procedure that delays the gradient updates. The delay is finely tuned so that the balancing mechanism is still intact while the objective is not severely compromised, leading to a no-regret and non-stationary policy.

\textbf{Related Literature.} MDPwGR is a common generalization of the \underline{B}andits \underline{w}ith \underline{G}lobal concave \underline{R}ewards (BwGR) 
and online \underline{MDP}s \underline{w}ith \underline{S}calar \underline{R}ewards (MDPwSR). BwGR is first studied by \citep{AgrawalD14}, who establish important connections between online convex optimization and upper-confidence-bound (UCB) based algorithms for BwGR and its generalization. The work of \citep{AgrawalD14} focus on stochastic $K$-armed bandits.
Subsequently, BwGR is studied 
contextual $K$-armed bandits \citep{AgrawalDL16}. \citep{Busa-feketeSWM17} consider $K$-armed bandits with generalized fairness objectives, which require special cares different from BwGR.  \citep{BerthetP17} consider the combination of Frank-Wolfe algorithm and UCB algorithms (which is also considered in \citep{AgrawalD14}), and \citep{BerthetP17} demonstrate fast rate convergence in cases when the concave reward functions are not known but satisfy certain smoothness property. 
 
 BwGR is closely related to \underline{B}andits \underline{w}ith \underline{K}napsacks (BwK), which precedes BwGR. BwK is first studied under $K$-armed bandits by \citep{BadanidiyuruKS13}. Subsequently, BwK is studied under $K$-armed bandits with concave rewards and convex constraints \citep{AgrawalD14}, contextual $K$-armed bandits \citep{BadanidiyuruLS14, AgrawalDL16} and linear bandits \citep{AgrawalD16}.  
The works on BwGR and BwK focus on stationary stochastic environments, and provide online algorithms with global rewards converging to the offline optimum when the number of time steps $T$ grows.  
 
 Recently, \citep{ImmorlicaSSS18} study the adversarial BwK, and show that it is impossible to achieve a competitive ratio of less than $T / B^2$ compared to the offline optimum, where $B$ is the budget. They propose algorithms with competitve ratios of $O(\log T)$ compared to a certain offline benchmark. Our positive results, which are on MDPwGR and MDPwK (see Appendix \ref{app:MDPwK}), walk a fine line between the negative results for adversarial BwK and positive results for stochastic BwGR and BwK. 
Finally, online optimization with global rewards are also studied in adversarial settings with full feedback \citep{Even-DarKMM09,AzarFFT14}.

MDPwSR on communicating MDPs is studied by \citep{AuerO06,JakschOA10}. Subsequently, \citep{AgrawalJ17} provide improved regret bounds by optimistic posterior sampling.  
\citep{Ortner18} derive regret bounds in terms of a mixing time parameter.
\citep{BartlettT09} consider a more general case of weakly communicating MDPs, and regret bounds are derived in \citep{FruitPLO18}. \citep{FruitPL18} study an even more general case of non-communicating MDPs. These works focus on optimizing under scalar rewards, but do not consider optimizing under vectorial outcomes. For a review on MDPs, please consult \citep{Puterman94}.

Reinforcement learning on multi-objective MDPs and MDPs with resource constraints are studied in discounted reward settings \citep{GaborKS98,BarrettS08,VanMoffaertN14}. \citep{NatarajanT05,LizotteBM12} design algorithms for average reward settings. \citep{MannorS04,MannorTY09} consider optimizing the average rewards in asymptotic settings, and demonstrate convergence of their algorithms. We study MDPwGR in an non-asymptotic average reward setting. 
Recently, \citep{HazanKSS18} study exploration problems on MDPs in offline settings, which can be modeled as MDPs with global rewards. Another recent work by \citep{TarbouriechL19} consider active exploration in Markov decision processes, which involves maximizing a certain concave function on the frequency of visiting each state-action pair. \citep{TarbouriechL19} assumes that the underlying transition kernel is known to the agent, and they also make certain mixing time assumptions that hold for every stationary policy. Different from \citep{TarbouriechL19}, our model allows the underlying transition kernel to be not known to the agent. Moreover, we only assume the underlying MDP to be \emph{communicating} (see Assumption \ref{ass:communicating} in Section \ref{sec:def}), which is less restrictive than the mixing time assumption \citep{TarbouriechL19}. \citep{JakschOA10} also provide a discussion on the relationship between mixing time assumptions and the communicating assumption. Constrained MDPs are reviewed in \citep{Altman99}, and multi-objective reinforcement learning is surveyed in \citep{RoijersVWD13,LiuXH15}. 

\textbf{Organization of the Paper. }In Section \ref{sec:def}, we provide the problem definition of MDPwGR, and define the offline benchmark for the regret analysis oif MDPwGR. In Section \ref{sec:alg}, we discuss the challenges in MDPwGR, and explain why existing works on BwGR and MDPwSR fail to solve MDPwGR to near-optimality. Then we introduce our algorithm {\sc Toc-UCRL2} which solves MDPwGR to near-optimality. In Section \ref{sec:ana}, we analyze  {\sc Toc-UCRL2} in the case of Frank-Wolfe oracle, assuming that the reward function is $\beta$-smooth. In Section \ref{sec:applications_model}, we discuss in details the applications of the problem model of MDPwGR, and demonstrate the near-optimality of {\sc Toc-UCRL2} in all these applications. Finally, we conclude in Section \ref{sec:conclusion}. Supplementary details to the discussions and more proofs are provided in the Appendix.

\textbf{Notation.} For $1\leq p\leq \infty$, $\|\cdot\|_p$ is the $\ell_p$ norm on $\mathbb{R}^K$, defined as $\|w\|_p := (\sum^K_{k=1}|w_k|^p)^{1/p}$. The vectors $\mathbf{1}_K, \mathbf{0}_K \in \mathbb{R}^K$ are the all-one and all-zero vectors, and $\mathbf{e}_k\in \mathbb{R}^K $ is the $k$th standard basis vector in $\mathbb{R}^K$. All vectors are column vectors by default. The inner product between $\theta, w \in \mathbb{R}^K$ is $\theta^\top w$. Denote $B(L,\|\cdot \|) := \{v\in \mathbb{R}^K : \|v\| \leq L\}$. For a norm $\| \cdot \|$ on $\mathbb{R}^K$, denote its dual norm as $\| \cdot \|_*$, where $\|\theta\|_* = \max_{w\in B(1, \|\cdot\|)} \theta^\top w$. For a finite set ${\cal U}$, denote $\Delta^{\cal U} := \{x\in \mathbb{R}^{\cal U} : p \geq 0, \sum_{u\in {\cal U}}p(u) = 1\}$ as the set of probability distributions over ${\cal U}$. For an event $E$, $\mathsf{1}(E) = 1$ if $E$ holds, and $\mathsf{1}(E) =0$ otherwise. Finally, ``w.r.t.'' stands for ``with respect to''.

\section{Problem Definition of MDPwGR}\label{sec:def}
An instance of MDPwGR is specified by the tuple $(\SSS, s_1, \AAA, p, V, g)$. 
The set $\SSS$ is a finite state space, and $s_1\in \SSS$ is the starting state. 
The collection $\AAA = \{\AAA_s\}_{s\in \SSS}$ contains a finite set of actions $\AAA_s$ for each state $s$. We say that $s, a$ is a state-action pair if $a\in \AAA_s$. 
The quantity $p$ is the transition kernel. For each $s, a$, we denote $p(\cdot | s, a)\in \Delta^\SSS$ as the probability distribution on the subsequent state when the agent takes action $a$ at state $s$.

For each $s, a$, the $K$-dimensional random variable $V(s, a) = (V_k(s, a))^K_{k=1} \in [0, 1]^K$ represents the stochastic outcomes. The mean is denoted as $\mathbb{E}[V(s,a)] = v(s, a) = (v(s, a))^K_{k=1}$. 
The reward function $g :  [0, 1]^K \rightarrow [0, 1]$ is concave, and is to be maximized. The function $g$ is $L$-Lipschitz continuous on $ [0, 1]^K$ w.r.t. a norm $\|\cdot\|$, i.e. $| g(u) - g(w) | \leq L\|u - w\|$ for all $u, w\in  [0, 1]^K$.\footnote{We also assume $g$ to be closed, i.e. $\{(w, u) \in  [0, 1]^K\times [0, 1]: u \geq g(w)\}$ is closed. This ensures  $g(w) = \min_{\theta \in B(L, \|\cdot\|_*)} \{g^*(\theta) - \theta^\top w \}$, where $g^*(\theta) := \max_{w\in  [0, 1]^K}\{ g(w) + \theta^\top w \}$ is the Fenchel dual of $g$. } The function $g$ needs not be monotonic in any of the $K$ dimensions. 

\textbf{Dynamics. }An agent, who faces an MDPwGR instance ${\cal M} = (\SSS, s_1, \AAA, p, V, g)$, starts at state $s_1\in \SSS$. At time $t$, three events happen. First, the agent observes his current state $s_t$. Second, she takes an action $a_t\in \AAA_{s_t}$. Third, she stochastically transits to another state $s_{t+1}$, and observes a vector of stochastic outcomes $V_t(s_t, a_t)$. In the second event, the choice of $a_t$ is based on a \emph{non-anticipatory} policy. That is, the choice only depends on the current state $s_t$ and the previous observations $H_{t-1} := \{s_q, a_q, V_q(s_q, a_q)\}^{t-1}_{q =1}$. When $a_t$ only depends on $s_t$, but not on $H_{t-1}$, we say that the corresponding non-anticipatory policy is \emph{stationary}.

At each time step $t$, the subsequent state $s_{t+1}$ and the outcomes $V_t(s_t, a_t)$ are generated in a Markovian manner. Conditional on $s_t, a_t$, we suppose four properties on $s_{t+1}, V_t(s_t, a_t)$. First, $s_{t+1}, V_t(s_t, a_t)$ are independent of $H_{t-1}$. Second, the subsequent state $s_{t+1}$ is distributed according to $p(\cdot | s_t, a_t)$, or in short $s_{t+1}\sim p(\cdot | s_t, a_t)$. Third, the outcome $V_t(s_t, a_t)$ is identically distributed as $V(s_t, a_t)$.  Fourth, $s_{t+1}, V_t(s_t, a_t)$ can be arbitrarily correlated. 

\textbf{Objective. }The MDPwGR instance ${\cal M}$ is latent. While the agent knows $\SSS, s_1, \AAA, g$, she does not know $v, p$. To state the objective, define $\bar{V}_{1:t} := \frac{1}{t}\sum^t_{q = 1} V_q(s_q, a_q)$. For any horizon $T$ not known \emph{a priori}, the agent aims to maximize $g (\bar{V}_{1:T} )$, by selecting actions $a_1, \ldots, a_T$ with a non-anticipatory policy. The agent faces a dilemma between exploration and exploitation. She needs to learn $v, p$ while optimizing in a Markovian environment. 

MDPwGR models a variety of online learning problems in Markovian environments, such as multi-objective optimization (MOO), maximum entropy exploration (MaxEnt), and MDPwSR with knapsack constraints in the large volume regime (MDPwK). We elaborate on these applications in Section \ref{sec:applications_model}. Finally, if $g$ is a linear function, we recover MDPwSR \citep{JakschOA10}; if we specialize $\SSS = \{s\}$, we recover BwGR \citep{AgrawalD14}.

\textbf{Reachability of ${\cal M}$. }
To ensure learnability, we suppose in Assumption \ref{ass:communicating} that the instance ${\cal M}$ is \emph{communicating}. 
For any $s, s'\in \SSS$ and any stationary policy $\pi$, the travel time from $s$ to $s'$ under $\pi$ is equal to the random variable $\Lambda(s' | \pi, s) := \min\left\{ t : s_{t+1} = s' , s_1 = s, s_{\tau + 1} \sim p(\cdot | s_\tau, \pi(s_\tau)) \text{ $\forall\tau$}\right\}$. 
\begin{assumption}\label{ass:communicating}
The latent MDPwGR instance ${\cal M}$ is communicating, that is, the quantity $
D := \max_{s, s'\in \SSS}\min_{\text{stationary }\pi} \mathbb{E}\left[\Lambda(s' | \pi, s) \right] 
$
is finite. We call $D$ the diameter of ${\cal M}$.
\end{assumption}
The same reachability assumption is made in \citep{JakschOA10}. Since the instance ${\cal M}$ is latent, the corresponding diameter $D$ is also not known to the agent. Assumption \ref{ass:communicating} is weaker than the \emph{unichain} assumption, where every stationary policy induces a single recurrent
class on $\SSS$. 

\textbf{Offline Benchmark and Regret. }To measure the effectiveness of a policy, we rephrase the agent's objective as the minimization of regret: $
\text{Reg}(T) := \text{opt}(\primal_{\cal M}) - g(\bar{V}_{1:T})$. The offline benchmark $\text{opt}(\primal_{\cal M})$ is the optimum of the convex optimization problem $(\primal_{\cal M})$, which serves as a fluid relaxation \citep{Puterman94,Altman99} to the MDPwGR problem. 
\begin{subequations}
\begin{alignat}{2}
(\primal_{\cal M})\text{: \;} & \max_x ~ g\left(\sum_{s\in \SSS, a\in \AAA_s} v(s, a) x(s, a)\right) & \nonumber\\
\text{s.t. } &\sum_{a \in \AAA_s} x(s, a) = \sum_{s'\in\SSS, a'\in \AAA_{s'}} p(s | s',a') x(s', a')&\quad &\forall s \in \SSS \label{eq:P.1}\\
               &\sum_{s\in \SSS, a\in \AAA_s}x(s, a) = 1       &\quad & \label{eq:P.2}\\
               &x(s, a)\geq 0      &\quad &\forall s\in \SSS, a\in \AAA_s \label{eq:P.3}
\end{alignat}
\end{subequations}
In $(\primal_{\cal M})$, the variables $\{x (s, a)\}_{s, a}$ form a probability distribution over the state-action pairs. The set of constraints (\ref{eq:P.1}) requires the rates of transiting into and out of each state $s$ to be equal. 

To achieve near-optimality, we aim to design a non-anticipatory policy with an \emph{anytime regret bound} $\text{Reg}(T)=O(1/T^\alpha)$ for some $\alpha >0$. That is, for all $\delta > 0$, there exist constants $c, C$ (which only depend on $K, S, A, g, \delta$), so that the policy satisfies $\text{Reg}(T) \leq c T^{-\alpha}$ for all $T \geq C$ with probability at least $1-\delta$. Our offline benchmark $\text{opt}(\primal_{\cal M})$ is justified as follows:
\begin{restatable}{theorem}{thmbenchmark}\label{thm:benchmark}
Consider an MDPwGR instance ${\cal M}$ that satisfies Assumption \ref{ass:communicating} with diameter $D$. For any non-anticipatory policy, it holds that
$$
\mathbb{E} [g (\bar{V}_{1:T} )] \leq \text{opt}(\primal_{\cal M}) + 2 L\|\mathbf{1}_K\| D / T.
$$
\end{restatable}
Theorem \ref{thm:benchmark} is proved in Appendix \ref{app:pfthmbenchmark}. Interestingly, the proof requires inspecting a dual formulation of $(\primal_{\cal M})$, and it appears hard to analyze $(\primal_{\cal M})$ directly. We could have $
\mathbb{E} [g (\bar{V}_{1:T}) ] > \text{opt}(\primal_{\cal M})$ when $T$ is small (see Appendix \ref{app:offlineremark}), thus an additive term in the upper bound is necessary.


\section{Challenges of MDPwGR, and Algorithm {\sc Toc-UCRL2}}\label{sec:alg}

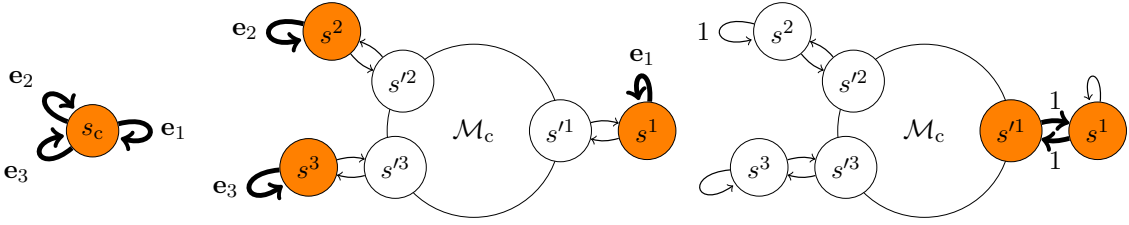
\begin{figure}
\centering
\resizebox{0.9\textwidth}{!}{%
\begin{tikzpicture}
\node [circle,draw,fill=orange] (sc) at (0,0) {$s_\text{c}$};
\path[->] (sc) edge [in=-15,out=15,loop, line width=2pt] node [right] {$\textbf{e}_1$} ();
\path[->] (sc) edge [in=130,out=160,loop, line width=2pt] node [above left] {$\textbf{e}_2$} ();
\path[->] (sc) edge [in=190,out=220,loop, line width=2pt] node [below left] {$\textbf{e}_3$} ();
\node [circle,draw,minimum size=2.5cm] (mcgr) at (5.5cm, 0cm) {$\mathcal{M}_\text{c}$};
\node [circle,fill=white,draw=black,overlay] (sprime1) at (6.75,0) {$s'^1$};
\node [circle,fill=orange,draw=black,overlay] (s1) at (8,0) {$s^1$};
\path[->] (s1) edge [in=105,out=85,loop, line width=2pt] node [above] {$\textbf{e}_1$} ();
\path[->] (sprime1) edge [in=165, out=15, line width=0.5pt] (s1);
\path[->] (s1) edge [in=-15, out=195, line width=0.5pt] (sprime1);
\node [circle,fill=white,draw=black,overlay] (sprime2) at (4.47606,0.71697) {$s'^2$};
\node [circle,fill=orange,draw=black,overlay] (s2) at (3.45211,1.43394) {$s^2$};
\path[->] (s2) edge [in=195,out=165,loop, line width=2pt] node [left] {$\textbf{e}_2$} ();
\path[->] (sprime2) edge [in=-20, out=130, line width=0.5pt] (s2);
\path[->] (s2) edge [in=160, out=-50, line width=0.5pt] (sprime2);
\node [circle,fill=white,draw=black,overlay] (sprime3) at (4.36712,-0.52827) {$s'^3$};
\node [circle,fill=orange,draw=black,overlay] (s3) at (3.11712,-0.52827) {$s^3$};
\path[->] (s3) edge [in=215,out=185,loop, line width=2pt] node [left] {$\textbf{e}_3$} ();
\path[->] (sprime3) edge [in=-15, out=195, line width=0.5pt] (s3);
\path[->, draw] (s3) edge [in=165, out=15, line width=0.5pt] (sprime3);
\node [circle,draw,minimum size=2.5cm] (mcgrs) at (12cm, 0cm) {$\mathcal{M}_\text{c}$};
\node [circle,fill=orange,draw=black,overlay] (sprime1s) at (13.25,0) {$s'^1$};
\node [circle,fill=orange,draw=black,overlay] (s1s) at (14.5,0) {$s^1$};
\path[->] (s1s) edge [in=105,out=85,loop, line width=0.5pt] node [above] {} ();
\path[->] (sprime1s) edge [in=165, out=15, line width=2pt] node [above] {$1$} (s1s);
\path[->] (s1s) edge [in=-15, out=195, line width=2pt] node [below] {$1$}(sprime1s);
\node [circle,fill=white,draw=black,overlay] (sprime2s) at (10.97606,0.71697) {$s'^2$};
\node [circle,fill=white,draw=black,overlay] (s2s) at (9.95211,1.43394) {$s^2$};
\path[->] (s2s) edge [in=195,out=165,loop, line width=0.5pt] node [left] {1} ();
\path[->] (sprime2s) edge [in=-20, out=130, line width=0.5pt] (s2s);
\path[->] (s2s) edge [in=160, out=-50, line width=0.5pt] (sprime2s);
\node [circle,fill=white,draw=black,overlay] (sprime3s) at (10.86712,-0.52827) {$s'^3$};
\node [circle,fill=white,draw=black,overlay] (s3s) at (9.61712,-0.52827) {$s^3$};
\path[->] (s3s) edge [in=215,out=185,loop, line width=0.5pt] node [left] {} ();
\path[->] (sprime3s) edge [in=-15, out=195, line width=0.5pt] (s3s);
\path[->] (s3s) edge [in=165, out=15, line width=0.5pt] (sprime3s);
\end{tikzpicture}
}%
\caption{Instances ${\cal M}_{\text{BG}}, {\cal M}_{\text{MG}}, {\cal M}_{\text{MS}}$ with $K=3$, from left to right. Optimal actions are bolded.}\label{fig:trouble_inst}
\end{figure}

\textbf{Challenges.} While MDPwGR is a common generalization of BwGR and MDPwGR, we identify unique challenges in MDPwGR for alternating among different actions, which is crucial for balancing the outcomes and achieving near-optimality.


We showcase these challenges in Fig. \ref{fig:trouble_inst}. An arc from state $s$ to state $s'$ represents an action $a$, with $p(s' | s, a) = 1$. Instance ${\cal M}_{\text{BG}}$, which can be seen as a BwGR instance, consists of a single state $s_\text{c}$ and $K$ actions $\{\CircArrowRight{}^k\}^K_{k=1}$.  Instances ${\cal M}_{\text{MG}}$, ${\cal M}_{\text{MS}}$ are respective instances for MDPwGR, MDPwSR. These instances share the same $\SSS, \AAA, p$. The center node ${\cal M}_\text{c}$ is a communicating MDP. Each peripheral node $s^1, \ldots, s^K$ is a distinct state, disjoint from ${\cal M}_\text{c}$. Each $s^k$ has a self-loop $\CircArrowRight{}^k$; there is an arc ${}^k\hspace{-1mm}\rightarrow_\text{c}$ from $s^k$ to $s'^k\in {\cal M}_\text{c}$, as well as an arc ${}_\text{c}\hspace{-1mm}\rightarrow^k$ back. Thus, ${\cal M}_{\text{MG}}, {\cal M}_{\text{MS}}$ are communicating. 

Let's focus on ${\cal M}_{\text{BG}}, {\cal M}_{\text{MG}}$, both with $ g(w) = 1  -\sum^K_{k=1} |w_k  -1/K | / 2$. 
For ${\cal M}_{\text{BG}}$, we set $V(s_\text{c}, \CircArrowRight{}^k) = \textbf{e}_k$ for each $1\leq k\leq K$. For ${\cal M}_{\text{MG}}$, we set $V(s^k, \CircArrowRight{}^k) = \textbf{e}_k$ for each $1\leq k\leq K$, and set $V(s, a) = \textbf{0}_K$ for all other $(s, a)$. Now, $\text{opt}(\mathsf{P}_{{\cal M}_\text{BG}}) = \text{opt}(\mathsf{P}_{{\cal M}_\text{MG}}) = 1$. In the case of ${\cal M}_\text{BG}$, the agent achieves a $O(K/T)$ anytime regret by choosing $\CircArrowRight{}^{k_t}$ at time $t$, where $k_t \equiv t(\text{mod } K)$. 

In ${\cal M}_\text{MG}$, an optimal policy has $K$ recurrent classes $\{\{s^k\}\}^K_{k=1}$, and each action $\CircArrowRight{}^k$ should be chosen with frequency $1/K$ for optimality. To alternate from $\CircArrowRight{}^k$ to $\CircArrowRight{}^{k'}\neq \CircArrowRight{}^k$, the agent has to travel from state $s^k$ to $s^{k'}$, which forces her to visit ${\cal M}_\text{c}$ and compromises her objective. This presents a more difficult case than ${\cal M}_\text{BG}$, where she can freely alternate among $\{\CircArrowRight{}^k\}_{k=1}^K$. 

The agent has to explore ${\cal M_\text{c}}$ and seek shortest paths among $s'^k$s to alternate among $\{ \CircArrowRight{}^k\}_{k=1}^K$. Importantly, her frequency of alternations has to be finely controlled. To elaborate, define $N_\text{alt}(T) = \sum^T_{s=1}\mathsf{1}(a_t \in\{  {}^k\hspace{-1mm}\rightarrow_\text{c} \}^K_{k=1})$, which is the number of alternations among $\{ \CircArrowRight{}_k\}_{k=1}^K$ in the first $T$ time steps (An alternation from $\CircArrowRight{}^k$ requires visiting  ${}^k\hspace{-1mm}\rightarrow_\text{c}$ once). The agent's anytime regret depends on $N_\text{alt}(T)$ delicately:
\begin{restatable}{claim}{claimsmalleg}\label{claim:small_eg}
Let $\alpha \in (0, 1)$ be arbitrary, and $K \geq 2$. There is an ${\cal M}_\text{c}$ such that: If $N_\text{alt}(T) \geq  T^\alpha$, then $\text{Reg}(T) = \Omega(1 / T^{1 - \alpha})$.  If $N_\text{alt}(T) < T^\alpha$, then either $\text{Reg}(T) = \Omega(1 / T^{1 - \alpha})$, or $\text{Reg}(\tau) = \Omega(D / \tau^{\alpha})$ for some $T / 2 \leq \tau\leq T$, where $D\geq 2$ is the diameter of ${\cal M}_\text{MG}$.
\end{restatable}
Claim \ref{claim:small_eg}, proved in Appendix \ref{app:pfclaimsmalleg}, holds even when the agent knows $v, p$. In the claim, the first ``if'' case is when the agent alternates too often, and compromises the objective by visiting the sub-optimal ${\cal M}_\text{c}$ too many times. The second ``if'' case is essentially when the agent alternates too infrequently by staying at a loop $\CircArrowRight{}^k$ for too long, leading to an imbalance in the outcomes. The frequency of alternation becomes an even subtler issue when the agent has to maintain simultaneous exploration and exploitation on ${\cal M}$.

The trade-off in Claim \ref{claim:small_eg} is absent in MDPwSR, where the agent follows a single stationary policy and alternates within a recurrent class of optimal states. Since the rewards are scalar, the agent does not need to balance the outcomes, unlike in MDPwGR. For example, in instance ${\cal M}_\text{MS}$, state-action pairs $(s^1, {}^1\hspace{-1mm}\rightarrow_\text{c})$, $(s'^1, {}_\text{c}\hspace{-1mm}\rightarrow^1), (s^2, \CircArrowRight{}^2)$ have scalar reward 1, while other state-action pairs have scalar reward 0. The agent achieves a $\text{Reg}(T) = O(D/T)$ anytime regret by traveling from a starting state to $s'^1$, and then alternating solely between $s'^1, s^1$ by actions ${}_\text{c}\hspace{-1mm}\rightarrow^1, {}^1\hspace{-1mm}\rightarrow_\text{c} $ indefinitely.


Altogether, the trade-off in Claim \ref{claim:small_eg} occurs when a policy can have multiple recurrent classes, which is possible in communicating MDPs
, but not unichain MDPs. In fact, stationary policies are in general sub-optimal for MDP-wGR:
\begin{restatable}{claim}{claimcontsmalleg}\label{claim:cont_small_eg}
There exists ${\cal M}_\text{MG}$ under which any stationary policy incurs an $\Omega(1)$ anytime regret.
\end{restatable}
The Claim is proved in Appendix \ref{app:pfclaimsmalleg}. Claim \ref{claim:cont_small_eg} is in stark contrast to the optimality of stationary policies in the unichain case \citep{Altman99}, or the scalar reward case \citep{JakschOA10}, or the discounted case \citep{Altman99}. How should the agent manage her exploration-exploitation trade-off, in face of the trade-off in alternating among actions (cf. Claim \ref{claim:small_eg}), while avoiding converging to a stationary policy? 




\begin{algorithm}[t]
\caption{{\sc Toc-UCRL2}}\label{alg:oco-ucrl2}
\begin{algorithmic}[1]
\State Inputs: Parameter $\delta\in (0, 1)$, gradient $\theta_1 \in B(L, \|\cdot\|_*)$, grad. threshold $Q>0$, initial state $s_1$.
\State Input oracles: OCO oracle $\textsf{OCO}$, EVI oracle $\textsf{EVI}$.
\State Initialize $t = 1$. \;

\State \textbf{for} episode $m = 1, 2, \ldots$ \textbf{do}
	\State \hspace{0.5cm} Set $\tau(m) = t$, and initialize $N_m(s, a)$ according to Eq. (\ref{eq:Nm}) for all $s, a$.\; \label{alg:oco-ucrl2-tau}
 	\State \hspace{0.5cm} Compute the confidence regions $H^v_m$, $H^p_m$ respectively for $v, p$, according to Eqns (\ref{eq:Hvm}, \ref{eq:Hpm}).\;
	\State \hspace{0.5cm} Compute the optimistic reward $\tilde{r}_m = \{\tilde{r}_m(s, a)\}_{s\in \SSS, a\in \AAA_s}$:\;
	\begin{equation}\label{eq:opt_reward}
		 \tilde{r}_m(s, a) = \max_{\bar{v}(s, a)\in H^v_m(s, a) } (- \theta_{\tau(m)})^\top \bar{v}(s, a).
	\end{equation}\label{alg:opt_reward}
	\State \hspace{0.5cm} Compute a $(1/\sqrt{\tau(m)})$-optimal optimistic policy $\tilde{\pi}_m$: \;
	\begin{equation}\label{eq:evi_m}
		 \tilde{\pi}_m, (\tilde{\phi}_m, \tilde{\gamma}_m)\leftarrow \textsf{EVI}(\tilde{r}_m, H^p_m ; 1/\sqrt{\tau(m)}).
	\end{equation}
	\State \hspace{0.5cm} Initialize $\nu_m(s, a) = 0$ for all $s, a$.\;
	\State \hspace{0.5cm} Initialize reference gradient $\theta^\refe = \theta_{\tau(m)}$, and $\Psi = 0$\; \label{alg:oco-ucrl2-initial}
	\State \hspace{0.5cm} \textbf{while} $\Psi \leq Q$ and $\nu_m(s_t, \tilde{\pi}_m(s_t)) < N^+_m(s_t, \tilde{\pi}_m(s_t))$ \text{do} \label{alg:oco-ucrl2-while}
	\State	\hspace{0.5cm} \hspace{0.5cm} Choose action $a_t = \tilde{\pi}_m(s_t)$. \;\label{alg:oco-ucrl2-action}
	\State	\hspace{0.5cm} \hspace{0.5cm} Observe the outcomes $V_t(s_t, a_t)$ and the next state $s_{t+1}$.\; \label{alg:oco-ucrl2-receive}
	\State	\hspace{0.5cm} \hspace{0.5cm} Compute gradient $\theta_{t + 1}$ based on  $\textsf{OCO}$ and the observation history.\; \label{alg:update_grad}
	\State	\hspace{0.5cm} \hspace{0.5cm} Update $\Psi \leftarrow \Psi + \| \theta_{t+1} - \theta^\refe\|_*$.\;
	\State	\hspace{0.5cm} \hspace{0.5cm} Update $\nu_m(s_t, a_t) \leftarrow \nu_m(s_t, a_t) + 1$.
	\State	 \hspace{0.5cm} \hspace{0.5cm} Update $t \leftarrow t+1$.\; \label{alg:oco-ucrl2-end}
\end{algorithmic}
\end{algorithm}

\textbf{Algorithm. }We propose Algorithm {\sc Toc-UCRL2}, displayed in Algorithm \ref{alg:oco-ucrl2}, for MDPwGR. The algorithm runs in episodes, and it overcomes the discussed challenges by a novel gradient threshold procedure. During episode $m$, which starts at time $\tau(m)$, it runs a certain stationary policy $\tilde{\pi}_m$, until the end of the episode at time $\tau(m+1) -1$. The start times $\{\tau(m)\}^\infty_{m=1}$ and policies $\{\tilde{\pi}_m\}^\infty_{m=1}$ are decided adaptively, as discussed later. We maintain confidence regions $H^v_m = \{H^v_m(s, a)\}_{s, a}$, $H^p_m = \{H^p_m(s, a)\}_{s, a}$ on the latent $v$, $p$ across episodes, by first defining
\begin{equation}\label{eq:Nm}
N_m(s, a) = \sum^{\tau(m) - 1}_{t = 1} \mathsf{1}(s_t = s, a_t = a),\quad  N^+_m(s, a) = \max\{1, N_m(s, a)\}.
\end{equation}
Define $(\text{log-$v$})_m := \log (12 K SA \tau^2(m) / \delta )$. The estimates and confidence regions for $v$ are:
\begin{align}
\hat{v}_m(s, a) &:= \frac{1}{N^+_m(s, a)} \sum^{\tau(m) - 1}_{t = 1} V_\tau(s_t , a_t )\mathsf{1}(s_t = s, a_t = a),\nonumber\\
\rad^{v}_{m, k}(s, a) &:= \sqrt{\frac{2\hat{v}_{m, k}( s, a) \cdot (\text{log-$v$})_m }{N^+_m(s, a)}} + \frac{3\cdot (\text{log-$v$})_m }{N^+_m(s, a)}, \nonumber\\
H^v_m(s,  a) &:= \left\{\bar{v} \in  [0, 1]^K : \left| \bar{v}_k - \hat{v}_{m, k}(s, a)\right| \leq \rad^v_{m, k}( s, a) \; \forall k\in [K]\right\} \label{eq:Hvm}.
\end{align}
Define $(\text{log-$p$})_m := \log ( 12 S^2 A \tau^2(m) / \delta )$. The estimates and confidence regions for $p$ are:
\begin{align}
\hat{p}_m(s' | s, a) &:= \frac{1}{N^+_m(s, a)} \sum^{\tau(m) - 1}_{t = 1} \mathsf{1}(s_t = s, a_t = a, s_{t + 1} = s'),\nonumber\\
\rad^p_{ m}(s' | s, a) &:= 
\sqrt{\frac{2\hat{p}_m(s' | s, a) \cdot (\text{log-$p$})_m }{N^+_m(s, a)}} 
+ 
\frac{3\cdot (\text{log-$p$})_m}{N^+_m(s, a)},\nonumber\\
H^p_m(s,  a) &:= \left\{ \bar{p}\in \Delta^\SSS :  \left| \bar{p}(s') - \hat{p}_m(s' | s, a)\right| \leq \rad^p_{ m}(s' | s, a) \;\forall s'\in \SSS\right\} \label{eq:Hpm}.
\end{align}

\textbf{OCO Oracle.} We balance the contributions from each of the $K$ outcomes by an Online Convex Optimization (OCO) oracle $\textsf{OCO}$. The applications of OCO tools with UCB algorithms are first studied in bandit settings by \citep{AgrawalD14}, and are subsequently studied in different settings by \citep{AgrawalDL16,Busa-feketeSWM17,BerthetP17}. An OCO oracle is typically based on a gradient descent algorithm. At the end of time $t$, the oracle $\textsf{OCO}$ computes a sub-gradient $\theta_{t+1}\in B(L, \|\cdot\|_*)$ that depends on $g, \{V_t(s_q, a_q), \theta_q\}^t_{q=1}$. For each $s, a$, the scalar reward $(-\theta_{t+1})^\top v(s, a)\in \mathbb{R}$ reflects how well $(s, a)$ balances the outcomes. To illustrate, we provide the definition of the Frank-Wolfe oracle based on \citep{FrankW56}, which is defined for $\beta$-smooth reward functions (see later in Defintion \ref{def:smooth}). The initial gradient is $\theta_1 = - \nabla g(\textbf{0}_K)$. 
To prepare for time $t+1$, at the end of time $t$ the oracle $\FW$ outputs gradient
$$
\theta_{t+1} = -\nabla g(\bar{V}_{1:t}  ).
$$
For an even more concrete example, consider instance ${\cal M_{\text{MG}}}$ with  $g(w) = 1 - \sum^K_{k=1}(w_k - 1 / K)^2 / 2$. 
The $\FW$ oracle outputs $\theta_{t+1} = \bar{V}_{1:t} - \mathbf{1}_K / K $. The resulting scalar reward for $(s_k, \CircArrowRight{}_k)$ is $1 / K - \bar{V}_{1:t, k}$, confirming the intuition that the agent should choose those $\CircArrowRight{}_k$s with $\bar{V}_{1:t, k} \leq 1/K$, but not those $\CircArrowRight{}_k$s with $\bar{V}_{1:t, k} > 1/K$. 
  
\textbf{EVI Oracle.} Despite the uncertainty on $v, p$, we aim for an optimal policy for $\text{MS}(r_m, p)$, the MDPwSR with scalar rewards $r_m := \{(-\theta_{\tau(m)})^\top v(s, a)\}_{s,a}$ and transition kernel $p$. Problem $\text{MS}(r_m, p)$ is easier than the original MDPwGR, since $\text{MS}(r_m, p)$ is optimized by a stationary policy. 
Denote the optimal average reward as $\text{ave-opt}(\text{MS}(r_m, p)) = \max_{x\in R(\mathsf{P}_{\cal M})} \sum_{s, a} r_m(s, a)x(s, a)$, where $R(\mathsf{P}_{\cal M})$ is the feasible region of $(\mathsf{P}_{\cal M})$ that is defined by $p$. 

To learn $v, p$ and while optimizing $\text{MS}(r_m, p)$, we follow the optimistic approach in UCRL2 \citep{JakschOA10}, and employ an Extended Value Iteration (EVI) oracle $\textsf{EVI}$ in (\ref{eq:evi_m}). An EVI oracle computes a near-optimal and stationary policy $\tilde{\pi}_m$ for $\text{MS}(\tilde{r}_m, \tilde{p}_m)$, where $\tilde{r}_m$, $\tilde{p}_m$ are optimistic estimates of $r_m$, $p$, i.e. $\text{ave-opt}(\text{MS}(\tilde{r}_m, \tilde{p}_m)) = \max_{\bar{v} \in H^v_m, \bar{p}\in H^p_m} \text{ave-opt}(\text{MS}((-\theta_{\tau(m)})^\top\bar{v}, \bar{p}))$. The oracle also outputs $\tilde{\phi}_m\in \mathbb{R}$, an optimistic estimate of $\text{ave-opt}(\text{MS}(\tilde{r}_m, \tilde{p}_m))$, as well as $\tilde{\gamma}_m\in \mathbb{R}^\SSS$, a certain bias associated with each state. These outputs are useful for the analysis. Finally, $(1/\sqrt{\tau(m)})$ is a certain  prescribed error parameter for $\tilde{\pi}_m$. We extract an EVI oracle from \citep{JakschOA10}, displayed in Appendix \ref{app:evi}.

\textbf{Gradient Threshold.} While the OCO and EVI oracles are vital for solving MDPwGR, they are yet to be sufficient for solving MDPwGR. Let's revisit instance ${\cal M}_\text{MG}$ and Claim \ref{claim:small_eg}. An OCO oracle could potentially recommend alternating among $ \{ \CircArrowRight{}_k\}_{k=1}^K$ for $\Omega(T)$ times in $T$ time steps, leading to the first ``if'' case of a large $N_\text{alt}$. UCRL2 recommends alternating among $ \{ \CircArrowRight{}_k\}_{k=1}^K$ for only $O(SA\log T)$ times in $T$ time steps, leading to the second ``if'' case of a small $N_\text{alt}$. 

We introduce a novel gradient threshold procedure (starting from Line \ref{alg:oco-ucrl2-while}) to overcome the discussed challenges. 
The procedure maintains a distance measure $\Psi$ on the sub-gradients generated during each episode, and starts the next episode if the measure $\Psi$ exceeds a threshold $Q$. A small $Q$ makes the agent alternate among different stationary policies frequently and balance the outcomes, while a large $Q$ facilitates learning and avoid visiting sub-optimal states. It is interesting to note that {\sc Toc-UCRL2} does not converge to a stationary policy, except when we force $Q = \infty$. 
A properly tuned $Q$ paths the way to obtain near optimality for the MDPwGR problem. In the context of ${\cal M}_{\text{MG}}$, the threshold $Q$ can be tuned to optimize $N_\text{alt}$ for the regret bound, and to ensure that the agent alternates among $\{\CircArrowRight{}^k\}_{k=1}^K$ sufficiently often.   
   
While the procedure overcomes the challenges, it dilutes the balancing effect of the underlying OCO oracle by delaying gradient updates, and interferes with the learning of $v, p$. This makes the analysis of {\sc Toc-UCRL2} challenging. Despite these apparent obstacles, we still show that {\sc Toc-UCRL2} achieves an anytime regret that diminishes with $T$.

\textbf{Main Results.} We first focus on $\beta$-smooth $g$, then consider general $g$ in Section \ref{sec:general_cr}. 
\begin{definition}[$\beta$-smooth]\label{def:smooth}
For $\beta > 0$, a concave function $f : [0, 1]^K\rightarrow [0, 1]$ is $\beta$-smooth w.r.t. norm $\|\cdot\|$, if $f$ is differentiable on $[0, 1]^K$, and it holds for all $u, w\in [0, 1]^K$ that
\begin{equation}\label{eq:def_smooth}
\left\|\nabla f(u) - \nabla f(w ) \right\|_* \leq \beta \left\| u - w \right\|.
\end{equation}
\end{definition}
We provide regret bounds for {\sc Toc-UCRL2} under $\FW$. Denote $S := |\SSS |$, $A := \frac{1}{S}\sum_{s\in \SSS} |\AAA_s| $, so   $SA$ is the number of state-action pairs. Denote $\Gamma := \max_{s\in \SSS ,a\in \AAA_s} \| p(\cdot | s, a) \|_0$, which is the maximum number of states from which a state-action pair can transit to. We employ the $\tilde{O}(\cdot )$ notation, which hides additive terms which scales with $\log (T/\delta) / T$ as well as multiplicative $\log (T/\delta)$ factors. The full $O(\cdot)$ bounds for the Theorems and the analyses are provided in the Appendix.
\begin{restatable}{theorem}{thmfw}\label{thm:fw_brief}
Consider {\sc Toc-UCRL2} with OCO oracle $\FW$ and gradient threshold $Q > 0$, applied on a communicating MDPwGR instance ${\cal M}$ with diameter $D$. Suppose $g$ is $L$-Lipschitz continuous and $\beta$-smooth w.r.t the norm $\|\cdot \|$. With probability $1 - O(\delta)$, we have anytime regret bound
\begin{equation}\label{eq:fw_bd_brief}
\text{Reg}(T) = \tilde{O}\left(   \sqrt{\beta}  \left[ \sqrt{Q} + LD / \sqrt{Q} \right] \|\mathbf{1}_K\|^{3/2} \middle/\sqrt{T}\right)  + \tilde{O}\left( L  \| \mathbf{1}_K\| D\sqrt{\Gamma S A}  \middle/ \sqrt{T}\right).
\end{equation}
In particular, setting $Q = L$ gives $\tilde{O} ( \sqrt{\beta L } \|\mathbf{1}_K\|^{3/2} D]/\sqrt{T})  +  \tilde{O}( L \| \mathbf{1}_K\| D\sqrt{\Gamma S A} ]/\sqrt{T}) $.
\end{restatable}
In the first regret term, the summand with $\sqrt{Q}$ represents the regret due to the delay in gradient updates, and the summand with $1 / \sqrt{Q}$ represents the regret due to the interference of the gradient threshold procedure with the learning of $v, p$, as well as the regret in switching stationary policies, which could require visiting sub-optimal states. The second regret term is the regret due to the simultaneous exploration-exploitation using an EVI oracle. The $L\|\mathbf{1}_K\|$ factor scales with the magnitude of contribution from the outcomes at each time to the global reward. The same factor appears in related bandit settings \citep{AgrawalD14,AgrawalDL16}. 


Applying Theorem \ref{thm:fw_brief} on an MDPwSR instance, we recover the regret bound by \citep{JakschOA10}. Indeed, we recover UCRL2 (up to the difference in $H^v_m, H^p_m$) when we specialize {\sc Toc-UCRL2} with OCO oracle $\FW$ to linear $g$. Nevertheless, when we specialize 
{\sc Toc-UCRL2} with $\FW$ to BwGR problems with smooth $g$, we do not recover the Frank-Wolfe based algorithm (Algorithm 4 in \citep{AgrawalD14}) for BwGR, due to the gradient threshold procedure. The resulting regret bound is also different from \citep{AgrawalD14}, see their Theorem 5.2. Nevertheless, the procedure is crucial for MDPwGR. 
A direct combination of Frank-Wolfe Algorithm and UCRL2, which is equivalent to using OCO oracle $\FW$ and setting $Q =0 $, is insufficient for solving MDPwGR, see Appendix \ref{app:pfclaimegQ0}. 


\section{Analysis of {\sc Toc-UCRL2}, with Focus on Oracle $\FW$}\label{sec:ana}
In this Section, we provide an analytic framework for analyzing {\sc Toc-UCRL2} under general OCO oracles. In particular, we prove Theorem \ref{thm:fw_brief} to demonstrate our framework. To start, we consider events ${\cal E}^v , {\cal E}^p$, which quantify the accuracy in estimating $v, p$:
\begin{equation}\label{eq:event_v}
{\cal E}^v  := \left\{ v(s, a)\in H^v_m( s, a) \text{ for all $m\in \mathbb{N}$, $ s\in \SSS$, $a\in \AAA_s$}\right\},
\end{equation}
\begin{equation}\label{eq:event_p}
{\cal E}^p  := \left\{ p(\cdot | s, a) \in H^p_m(s, a) \text{ for all $m\in \mathbb{N}$, $s\in \SSS$, $a\in \AAA_s$}\right\}.
\end{equation}
\begin{restatable}{lemma}{lemmaconc}\label{lemma:conc}
Consider an execution of {\sc Toc-UCRL2} with a general OCO oracle. It holds that
$
\mathbb{P}[{\cal E}^v] \geq 1 - \delta/2, \mathbb{P}[{\cal E}^p] \geq 1 - \delta/2.
$
\end{restatable}
Lemma \ref{lemma:conc} is proved in Appendix \ref{app:pflemmaconc}. We analyze $\text{Reg}(T)$ by tracing the sequence of stochastic outcomes and quantifying their contributions to the global reward. The tracing bears similarity to the analysis of the Frank-Wolfe algorithm (see \cite{Bubeck15}).  Let's define the shorthand $v^* := \sum_{s,a} v(s, a)x^*(s, a)$, where $x^*$ is an optimal solution of $(\primal_{\cal M})$.
\begin{align}
g (\bar{V}_{1:t} ) &\geq g( \bar{V}_{1:t-1} ) + \nabla g(\bar{V}_{1:t-1} )^\top [ \bar{V}_{1:t} - \bar{V}_{1:t-1}] - \frac{\beta}{2} \|\bar{V}_{1:t} - \bar{V}_{1:t-1}\|^2 \label{eq:fw_intermediate_bound0}\\
& = g(\bar{V}_{1:t-1} ) + \frac{1}{t} \nabla g(\bar{V}_{1:t-1} )^\top [ V_t(s_t, a_t) - \bar{V}_{1:t-1} ] - \frac{\beta}{2 t^2} \|V_t(s_t, a_t) - \bar{V}_{1:t-1}\|^2\nonumber\\
&\geq g(\bar{V}_{1:t-1} ) + \frac{1}{t} \nabla g(\bar{V}_{1:t-1} )^\top [ v^*- \bar{V}_{1:t-1} ] + \frac{1}{t} \nabla g(\bar{V}_{1:t-1} )^\top [V_t(s_t, a_t) -v^* ] - \frac{\beta\|\mathbf{1}_K\|^2 }{2 t^2}  \nonumber\\
&\geq g( \bar{V}_{1:t-1} ) + \frac{1}{t}\left[\text{opt}(\primal_{\cal M}) - g( \bar{V}_{1:t-1})\right]+ \frac{1}{t} (-\theta_t)^\top  [ V_t(s_t, a_t) -v^* ] - \frac{\beta\|\mathbf{1}_K\|^2 }{2 t^2} \label{eq:fw_intermediate_bound2}.
\end{align}
Step (\ref{eq:fw_intermediate_bound0}) uses the $\beta$-smoothness of $g$. 
Rearranging (\ref{eq:fw_intermediate_bound2}) gives 
\begin{align}
t \cdot \text{Reg}(t) & \leq (t-1)\cdot \text{Reg}(t-1) + \frac{\beta \| \mathbf{1}_K \|^2}{2t}  + ( - \theta_t)^\top  [ v^* - V_t(s_t, a_t) ] \label{eq:fw_intermediate_bound3}.
\end{align}
Apply the inequality (\ref{eq:fw_intermediate_bound3}) recursively for $t = T, \ldots, 1$, we obtain the following regret bound :
\begin{align}
\text{Reg}(T) & \leq  \frac{\beta \left\|\mathbf{1}_K\right\|^2 \log T}{T} + \frac{1}{T}\sum^T_{t=1} ( - \theta_t)^\top  [ v^* - V_t(s_t, a_t) ] \label{eq:fw_intermediate_bound3.5} .
\end{align}
To proceed, we provide the following novel analysis that allows us to compare the online output and the offline benchmark, and help us analyze the effect of the gradient threshold procedure. For a time step $t$, we denote random variable $m(t)$ as the index of the episode that contains $t$. When the underlying OCO oracle is specified, we decorate $m(t)$ with the corresponding superscript, for example  $m^\FW(t)$ is the above-mentioned episode index under $\FW$. We provide the following Proposition that helps us analyze the second term in (\ref{eq:fw_intermediate_bound3.5}):
\begin{proposition}\label{prop:Err_bd}
Consider an execution of {\sc Toc-UCRL2} with a general OCO oracle, over a communicating MDPwGR instance {\cal M} with diameter $D$. For each $T \in \mathbb{N}$, suppose that there is a deterministic constant $M(T)$ s.t. $\Pr[m(T)\leq M(T)] = 1$. Conditioned on events ${\cal E}^v, {\cal E}^p$, with probability at least $1-O(\delta)$ we have
\begin{equation*}
\sum^T_{t=1} (-\theta_t)^\top  [ v^* - V_t(s_t, a_t) ] = \tilde{O}\left(  (L  D +  Q) \|\mathbf{1}_K\| M(T)\right) + \tilde{O}\left( L \|\mathbf{1}_K\| D\sqrt{\Gamma SAT } \right).
\end{equation*}
\end{proposition}
The Proposition is proved in Section \ref{sec:pf_Err_bd}, and can be understood by what follows. 
Since the action $a_t$ is chosen based on policy $\tilde{\pi}_m$, it turns out that $( - \theta_{\tau(m(t))})^\top [ v^* - v(s_t, a_t) ]$ can be upper bounded. 
By the threshold procedure, we essentially know that $\|\theta_{\tau(m(t))} - \theta_t\|_*$ cannot be too large. Consequently, we can bound  $(-\theta_t)^\top  [ v^* - V_t(s_t, a_t) ]$ from above.

We next provide a deterministic upper bound $M^\FW(T)$ for  $m^\FW(T)$: 
\begin{restatable}{lemma}{lemmafwbdMbrief}\label{lemma:fw_bd_M_brief}
Consider an execution of {\sc Toc-UCRL2} with OCO oracle $\FW$ and gradient threshold $Q>0$. With certainty, for every $T\in \mathbb{N}$ we have 
$$m^\FW(T) \leq M^\FW(T) = \tilde{O}\left( \sqrt{ \beta \|\mathbf{1}_K\| T / Q} \right). $$
\end{restatable}
Altoghether, Theorem \ref{thm:fw_brief} is proved by applying Proposition \ref{prop:Err_bd} and Lemma \ref{lemma:fw_bd_M_brief} to   (\ref{eq:fw_intermediate_bound3.5}). To complete the argument, we provide the proof of Lemma \ref{lemma:fw_bd_M_brief} in Section \ref{sec:pf_lemma_fw_bd_M_brief} and a high level view on the proof of Proposition \ref{prop:Err_bd} in Section \ref{sec:pf_Err_bd}.


\subsection{Proof of Lemma \ref{lemma:fw_bd_M_brief} } \label{sec:pf_lemma_fw_bd_M_brief}
Consider the following two sets of episode indexes:
\begin{align}
{\cal M}^\FW_\Psi(T) &:= \left\{m\in \mathbb{N}: \tau(m) \leq T\text{, episode $m+1$ is started due to } \Psi \geq Q\right\}, \nonumber\\
{\cal M}^\FW_\nu(T) & := \left\{m\in \mathbb{N}: \tau(m) \leq T\text{, episode $m+1$ is started due to } \right. \nonumber\\
&\qquad\qquad\quad\;\, \left. \nu_m(s_t, \tilde{\pi}_m(s_t)) \geq N^+_m(s_t, \tilde{\pi}_m(s_t))\text{ for some $t\geq \tau(m)$}\right\} \nonumber. 
\end{align}
The set ${\cal M}^\FW_\Psi(T)$ consists of indexes $m$'s for which the next episode is started 
because of $\Psi\geq Q$, indicating the overflow in the cumulative distance to the reference gradient. The set ${\cal M}^\FW_\nu(T)$ consists of indexes $m$'s for which the next episode is triggered because some state-action pair is observed for sufficiently many times. 

We have ${\cal M}^\FW_\Psi(T)\cup {\cal M}^\FW_\nu(T) = \{1, 2, \ldots, m^\FW(T)\}$, and the set $\{\tau(m): m \in {\cal M}^\FW_\Psi(T)\cup {\cal M}^\FW_\nu(T)\}$ consists of the starting time of each episode. The sets ${\cal M}^\FW_\Psi(T)$, ${\cal M}^\FW_\nu(T)$ need not be disjoint. 
To prove the Lemma, it suffices to show that
\begin{align}
n_\Psi := | {\cal M}^\FW_\Psi(T) | & \leq  M^\FW_\Psi(T) := 1 + \frac{Q}{2\beta \|\mathbf{1}_K\|} + \sqrt{\frac{  32\beta \|\mathbf{1}_K\|}{Q}\cdot T} , \label{eq:fw_bd_M_Psi}\\
n_\nu := | {\cal M}^\FW_\nu(T) | & \leq M^\FW_\nu(T) := SA (1 + \log_2 T) , \label{eq:fw_bd_M_nu}
\end{align} 
hold with certainty. Proving inequality (\ref{eq:fw_bd_M_Psi}) is the main part of the proof of the Lemma. While the proof of inequality (\ref{eq:fw_bd_M_nu}) follows \citep{JakschOA10}, we prove (\ref{eq:fw_bd_M_nu}) for completeness sake. 

\textbf{Demonstrating inequality (\ref{eq:fw_bd_M_Psi}). }Let's express ${\cal M}^\FW_\Psi(T) = \{m_1, m_2, \ldots, m_{n_\Psi}\}$, where $m_1 < m_2 < \ldots < m_{n_\Psi}$. We also define $m_0 = 0$. We focus on an arbitrary but fixed episode index $m_j$ with $j\geq 1$, and consider for each time step $t\in \{\tau(m_j)+1, \ldots, \tau(m_j+1)\}$ the difference:
\begin{align}
& \left\| \theta_t - \theta_{\tau(m_j)}\right\|_* \nonumber\\ 
= & \left\| \nabla g\left(\frac{1}{t-1}\sum^{t-1}_{q = 1}V_q (s_q , a_q )\right)- \nabla g\left(\frac{1}{\tau(m_j)-1}\sum^{\tau(m_j)-1}_{q = 1}V_q (s_q , a_q )\right)\right\|_* \nonumber\\
\leq & \beta \left\|\frac{1}{t-1}\sum^{t-1}_{q = 1}V_q(s_q, a_q )- \frac{1}{\tau(m_j)-1}\sum^{\tau(m_j)-1}_{q = 1}V_q(s_q, a_q )\right\| \nonumber\\
= & \beta \left\| \frac{1}{t-1} \sum^{t-1}_{q =\tau(m_j) }V_q (s_q , a_q) - \left[\frac{1}{\tau(m_j) - 1} - \frac{1}{t-1}\right]\sum^{\tau(m_j) - 1}_{q = 1}V_q (s_q, a_q) \right\|\nonumber\\
= & \beta \cdot \frac{t - \tau(m_j)}{t - 1} \cdot \left\|\frac{1}{t - \tau(m_j) } \sum^{t}_{q =\tau(m_j) }V_q (s_q, a_q) - \frac{1}{\tau(m_j) - 1}\sum^{\tau(m_j) - 1}_{q =1 }V_q (s_q, a_q) \right\| \nonumber\\
\leq &2\beta \left\| \mathbf{1}_K \right\|  \cdot \frac{t - \tau(m_j)}{t-1} \leq 2\beta \left\| \mathbf{1}_K \right\| \cdot \frac{t - \tau(m_j )}{\tau(m_j)} . \label{eq:fw_M_step_1}
\end{align}
By the fact that $m_j\in {\cal M}^\FW_\Psi(T)$, we know that $\sum^{\tau(m_j +1) }_{t = \tau(m_j )} \| \theta_t - \theta_{\tau(m_j )}  \|_*  > Q$. By applying the upper bound (\ref{eq:fw_M_step_1})\footnote{While the upper bound looks loose at the first sight, a simple inspection shows that $\tau(m_j + 1) - \tau(m_j)= O(\tau(m_j))$. Therefore, we can use the seemingly coarse upper bound (\ref{eq:fw_M_step_1}) without deteriorating the dependence on $T$ for our final bound in (\ref{eq:fw_bd_M_Psi}).}, we have
\begin{equation*}
2 \beta \left\| \mathbf{1}_K \right\|  \frac{(\tau(m_j +1) - \tau(m_j ))^2}{\tau(m_j )}\geq  2\beta \left\| \mathbf{1}_K \right\|  \sum^{\tau(m_j +1)}_{t = \tau(m_j )} \frac{t - \tau(m_j )}{\tau(m_j )} > Q,
\end{equation*}
where the first inequality is by the fact that $n^2 \geq n(n+1)/2$ for $n\geq 1$. Thus we arrive at
\begin{align}
\tau(m_j + 1) & \geq \tau(m_j) + \sqrt{\frac{Q}{2\beta \left\| \mathbf{1}_K \right\| } \cdot \tau(m_j)} \geq \tau(m_{j-1} + 1) + \sqrt{\frac{Q}{2 \beta \left\| \mathbf{1}_K \right\| } \cdot \tau(m_{j-1} + 1)} ,\label{eq:fw_M_step_2}
\end{align}
since $\tau(m) \geq \tau(m')$ for $m \geq m'$, and clearly $m_j \geq m_{j - 1} + 1$. 

Now, we apply Claim \ref{claim:aux_oco} with $C = \sqrt{Q / (2\beta\|\mathbf{1}_K\|)}$, $\alpha  = 1/2$, and $\rho_j =  \tau(m_{\lceil C^2\rceil +j} + 1)$ for $j = 1, 2, \ldots$. The application is valid, since $\rho_1 = \tau(m_{\lceil C^2\rceil +1} + 1) \geq C^2 = C^{1 / (1-\alpha) }$, and we are equipped with the recursive inequality (\ref{eq:fw_M_step_2}). Consequently, we arrive at
\begin{equation}\label{eq:fw_M_step_3}
\tau\left( m_{\lceil C^2 \rceil +j} + 1 \right)\geq \frac{Q}{32\beta \|\mathbf{1}_K\|}(j - 1)^2.
\end{equation}
Finally, if $n_\Psi \leq Q / (2\beta\|\mathbf{1}_K\|)$, then clearly (\ref{eq:fw_bd_M_Psi}) is established. Otherwise, we put $j = n_\Psi - \lceil Q / (2\beta\|\mathbf{1}_K\|) \rceil - 1$ in inequality (\ref{eq:fw_M_step_3}), which gives
\begin{equation}\label{eq:fw_M_step_4}
T \geq \tau(m_{n_\Psi })\geq\tau\left(m_{\lceil  C^2 \rceil + [n_\Psi -\lceil  C^2 \rceil - 1] } + 1\right) \geq \frac{Q}{32 \beta \|\mathbf{1}_K\|}\cdot \left(n_\Psi -   \frac{Q}{2\beta\|\mathbf{1}_K \| } -  1\right)^2.
\end{equation}
Finally, unraveling the bound (\ref{eq:fw_M_step_4}) gives the required upper bound (\ref{eq:fw_bd_M_Psi}).

\textbf{Demonstrating inequality (\ref{eq:fw_bd_M_nu}). }The demonstration follows the proof of Proposition 18 of (\citep{JakschOA10}). For each state-action pair, consider the set 
\begin{align}
{\cal M}^\FW_\nu(s, a ; T) & := \left\{m\in \mathbb{N}: \tau(m) \leq T\text{, episode $m+1$ is started due to } \right. \nonumber\\
&\qquad\qquad\quad\;\, \left. \nu_m(s, a) \geq N^+_m(s, a)\text{ for some $t\geq \tau(m)$}\right\} \nonumber. 
\end{align}
Clearly, we know that ${\cal M}^\FW_\nu(T) = \bigcup_{s\in \SSS}\bigcup_{a\in \AAA_s}{\cal M}^\FW_\nu(s, a ; T)$. To ease the notations, define shorthand $n_{\nu}(s, a) := |{\cal M}^\FW_\nu(s, a ; T)|$. To prove (\ref{eq:fw_bd_M_nu}), it suffices to show that $n_{\nu}(s, a) \leq 1+\log_2 T$.  Now, let's express $ {\cal M}^\FW_\nu(s, a ; T) = \{m_1, \ldots, m_{n_{\nu}(s, a)}\}$, where $m_1 < m_2 < \ldots < m_{n_{\nu}(s, a)}$. Observe that, by the way we define the set ${\cal M}^\FW_\nu(s, a ; T)$, the count sequence $N_{m_1 + 1}(s, a),  \ldots,$ $ N_{m_{n_{\nu}(s, a)} + 1}(s, a)$ is strictly increasing. In addition, by the exiting criteria of the \textbf{while} loop in Line \ref{alg:oco-ucrl2-while} in Algorithm {\sc Toc-UCRL2}, we see that $\{N_{m_1 + 1}(s, a), N_{m_2 + 1}(s, a), \ldots, N_{m_{n_{\nu}(s, a)} + 1}(s, a)\}$ is a subset of $\{1, 2, 2^2, \ldots, 2^{\lfloor \log_2 T\rfloor} \}$. Thus, we must have  $n_{\nu}(s, a)\leq 1+\log_2 T$, and the required inequality (\ref{eq:fw_bd_M_nu}) is shown. This concludes the proof of the Lemma. \hfill $\blacksquare$

\subsection{Proof of Proposition \ref{prop:Err_bd}}\label{sec:pf_Err_bd}
To prove the Proposition, we first propose a lemma that helps decompose each $(-\theta_t)^\top [v^* - V_t(s_t, a_t)]$:
\begin{restatable}{lemma}{lemmadecompose}\label{lemma:decompose}
Consider an execution of {\sc Toc-UCRL2} with a general OCO oracle over a communicating MDPwGR instance.  
Let $t$ be a time index, and let $m$ be the episode index such that $\tau(m) \leq t < \tau(m + 1)$. 
Conditional on events ${\cal E}^v, {\cal E}^p$, the following inequality holds:
\begin{equation*}
(-\theta_t)^\top [v^* - V_t(s_t, a_t)] \leq   (\clubsuit_t) + (\diamondsuit_t) + (\heartsuit_t) + (\spadesuit_t) + (\P_t),
\end{equation*}
where $v^* = \sum_{s\in \SSS, a\in \AAA_s} v(s, a)x^*(s, a)$ with $x^*$ optimal for $(\mathsf{P}_{\cal M})$, and 
\begin{align}
(\clubsuit_t) & := \left[\theta_t - \theta_{\tau(m)}\right]^\top V_t(s_t, a_t) ,   & (\diamondsuit_t) := \tilde{r}_m(s_t, a_t) -  [-\theta_{\tau(m)} ]^\top V_t(s_t, a_t), \label{eq:club}\\ 
(\heartsuit_t) & := \left[\theta_{\tau (m)} - \theta_{t} \right]^\top v^*, &(\spadesuit_t) := \max_{\bar{p}\in H^p_m(s_t, a_t)}\left\{\sum_{s'\in \SSS}\tilde{\gamma}_m(s') \bar{p}(s')\right\} - \tilde{\gamma}_m(s_t), \label{eq:heart}\\
(\P_t) &:= 1 /\sqrt{\tau(m)}.&\label{eq:P}
\end{align}
\end{restatable}
A proof of Lemma \ref{lemma:decompose} is provided in Appendix \ref{app:pflemmadecompose}. The proof is based on relating $(\tilde{\phi}_m, \tilde{\gamma}_m)$, which is output by an EVI Oracle, to the dual of $(\primal_{\cal M})$ with linearized reward $\tilde{r}_m$. 
The error terms (\ref{eq:club} -- \ref{eq:P}) account for the shortfall of the global reward collected by {\sc Toc-UCRL2}, compared to the offline reward. To prove the Proposition, it suffices to bound the sum of each error term over $1\leq t\leq T$. We first bound $(\clubsuit, \heartsuit)$, which account for the error by the delay of gradient updates: 
\begin{restatable}{lemma}{lemmaclubheart}\label{lemma:club_heart}
Suppose that gradient threshold $Q > 0$, and $\Pr\left[m(T) \leq M(T)\right] = 1$ for some deterministic constant $M(T)$. With probability 1, 
$$
\sum^T_{t=1} (\clubsuit_t) \leq Q  \left\|\mathbf{1}_K\right\| \cdot M(T), \qquad  \sum^T_{t=1} (\heartsuit_t) \leq Q   \left\|\mathbf{1}_K\right\| \cdot M(T).
$$
\end{restatable}
Lemma \ref{lemma:club_heart} is proved in Appendix \ref{app:pfclubheart}. We next bound $(\P, \heartsuit, \spadesuit)$, similarly to the styles in \citep{JakschOA10, FruitPL18}, but with important changes to adapt to our episode schedule. The error terms  $(\P, \heartsuit, \spadesuit)$ account for the error due to optimistic exploration, and the term $(\spadesuit)$ also penalizes for episode changes, which lead to sub-optimality due to the switches in stationary policies, and disrupt learning. 
\begin{restatable}{claim}{claimp}\label{claim:P}
With certainty, we have
$
\sum^T_{t=1} (\P_t) \leq \left(\sqrt{2} + 1\right)\sqrt{T}.
$
\end{restatable}
\begin{restatable}{lemma}{lemmadiamond}\label{lemma:diamond}
Conditional on event ${\cal E}^v$, with probability at least $1-\delta$ we have:
$$
\sum^T_{t=1} (\diamondsuit_t) = \tilde{O}\left(L \left\| \mathbf{1}_K\right\| \sqrt{ SAT }  + L \left\| \mathbf{1}_K\right\| SA  \right).
$$
\end{restatable}
\begin{restatable}{lemma}{lemmaspade}\label{lemma:spade}
Suppose that ${\cal M}$ is communicating with diameter $D$, and $\mathbb{P}\left[m(T) \leq M(T)\right] = 1$ for some deterministic constant $M(T)$. Conditional on event ${\cal E}^p$, with probability at least $1-\delta$ we have
$$
\sum^T_{t=1} (\spadesuit_t) = \tilde{O}\left( L \|\mathbf{1}_K\| D\cdot M(T) \right) +  \tilde{O}\left(L \|\mathbf{1}_K\| D\sqrt{\Gamma SAT } + L \|\mathbf{1}_K\|  DS^2 A  \right).
$$
\end{restatable}
The proofs of Claim \ref{claim:P}, Lemmas \ref{lemma:diamond}, \ref{lemma:spade} are provided in Appendices \ref{app:pfclaimP}, \ref{app:pflemmadiamond}, \ref{app:pflemmaspade} respectively. Altogether, Proposition \ref{prop:Err_bd} is proved by summing the bounds for $(\clubsuit, \diamondsuit, \heartsuit, \spadesuit, \P)$. \hfill $\blacksquare$

\section{Extensions to General Concave Rewards}\label{sec:general_cr}
The framework in Algorithm \ref{alg:oco-ucrl2} is versatile, as it can incorporate different OCO oracles for different $g$ and $\|\cdot\|$. In this Section, we provide OCO oracles for handling non-smooth $g$. 
First, we provide the Tuned Gradient Descent oracle $\TGD$ for the case when $g$ is Lipschitz continuous w.r.t. to the Euclidean norm $\| \cdot\|_2$. In this  case, we have $\|\cdot\| = \|\cdot \|_* = \| \cdot\|_2$. 
The oracle $\TGD$ is based on \citep{Zinkevich03}, and the oracle involves the Fenchel dual $
g^*(\theta) := \max_{w\in [0, 1]^K} \{ g(w) + \theta^\top w \}$. In addition, $\TGD$ involves projecting a point $\theta$ to $ B(L, \|\cdot\|_*)$, denoted as $\text{Proj}_L (\theta) \in \text{argmin}_{\vartheta\in B(L, \|\cdot\|_*) }\{\|\theta - \vartheta\|_* \}$.  
The oracle $\TGD$ begins with an arbitrary $\theta_1\in B(L, \|\cdot\|_*)$. To prepare for time $t+1$, at time $t$ the oracle $\TGD$ outputs
$$
\theta_{t+1} = \text{Proj}_L \left( \theta_t  - \eta^\TGD_t\left[\nabla g^*(\theta_t) - V_t(s_t, a_t) \right] \right), \text{ where } \eta^\TGD_t := L / (\|\mathbf{1}_K\| t^{2/3}).
$$
The learning rate scales as $\Theta(1/t^{2/3})$ instead of $\Theta(1/ t^{1/2})$ in \citep{Zinkevich03}. By bounding the number of episodes for $\TGD$ and harnessing the framework in Section \ref{sec:ana}, we derive a performance guarantee for $\TGD$ in the following Theorem, which is proved in Appendix \ref{app:tgd}.
\begin{restatable}{theorem}{thmtgd}\label{thm:TGD}
Consider {\sc Toc-UCRL2} with OCO oracle $\TGD$ and gradient threshold $Q > 0$, applied on a communicating instance ${\cal M}$ with diameter $D$. Suppose the concave reward function $g$ is $L$-Lipschitz continuous w.r.t $\|\cdot\|_2$. With probability $1 - O(\delta)$, we have anytime regret bound
\begin{equation*}
\text{Reg}(T) = \tilde{O}\left(  \left[\sqrt{LQ} + L^{3/2}D / \sqrt{Q} \right]  \|\mathbf{1}_K\|_2 \middle/ T^{1/3} \right) + \tilde{O}\left( L \|\mathbf{1}_K\|_2 D\sqrt{\Gamma S A} \middle/ \sqrt{T}\right).
\end{equation*}
Choosing $Q =L$ gives regret bound $\tilde{O} ( L \|\mathbf{1}_K\|_2 D /T^{1/3} )  +\tilde{O}(L \|\mathbf{1}_K\|_2 D \sqrt{\Gamma S A} /\sqrt{T} ) $.
\end{restatable}

Finally, for the general case when $g$ is $L$-Lipschitz continuous w.r.t. a non-Euclidean norm $\|\cdot\|$, we provide the Tuned Mirror Descent oracle $\TMD(F, T)$, which is based on the Mirror Descent Algorithm by \citep{NemirovskiY83}. The oracle $\TMD(F, T)$ assumes a \emph{mirror map} $F$ for $(g, \|\cdot\|)$, as well as the horizon $T$. The mirror map $F$ for $(g, \|\cdot\|)$ is an extended value function $F: B(L, \|\cdot\|_*) \rightarrow (-\infty, \infty]$ with two properties: (1) $F$ is 1-strongly convex w.r.t norm $\|\cdot\|_*$, (2) the domain\footnote{We define $\text{dom}(F) :=\{\theta \in B(L, \|\cdot\|_*) : F(\theta) < \infty\}$.} $\text{dom}(F)\subseteq B(L, \|\cdot\|_*)$ contains $\partial g(w)$ as a subset for every $w\in [0, 1]^K$. 

Clearly, if $F$ satisfies (1) and $\text{dom}(F) = B(L, \|\cdot\|_*)$, then $F$ is a mirror map for $(g, \|\cdot\|)$. Nevertheless, for certain $g$ the domain $\text{dom}(F)$ can be smaller, see Appendix \ref{app:tmd_background}.  By incorporating the doubling trick that guesses $T$, the implementation of {\sc Toc-UCRL2} under the oracle is still an anytime algorithm (see Algorithm \ref{alg:tmd_doubling} in Appendix \ref{app:pfthmTMD}). The oracle $\TMD(F, T)$ begins with $\theta_1\in \text{argmin}_{\theta \in \text{dom}(F)} F(\theta)$. To prepare for time $t+1$, at the end of time $t$ the oracle $\TMD(F, T)$ outputs
$$
\theta_{t+1} = \underset{\theta \in \text{dom}(F)  }{\text{argmax}}\left\{-\theta^\top \left[\eta^\TMD_T \left(\sum^{t}_{q = 1} \nabla [g^*(\theta_q)] - V_q(s_q, a_q)\right)\right] - F(\theta)\right\}; \ \eta^\TMD_T := \frac{L'}{\|\mathbf{1}_K \| T^{2/3}},
$$
where $L'^2 := \max_{\theta \in \text{dom}(F)} \{F(\theta)\} - \min_{\theta \in \text{dom}(F)} \{F(\theta)\}$. The $\TMD(F, T)$ oracle assumes a learning rate of $\Theta(1/T^{2/3})$, different from the classical rate $\Theta(1/T^{1/2})$. 

\begin{restatable}{theorem}{thmtmd}\label{thm:TMD}
Suppose that the underlying instance ${\cal M}$ is communicating with diameter $D$, and $g$ is $L$-Lipschitz continuous w.r.t $\|\cdot\|$.
Consider the application of {\sc Toc-UCRL2} under OCO oracle $\TMD(F, T)$ and gradient threshold $Q > 0$, with a mirror map $F$ for $(g, \|\cdot\|)$ and the doubling trick that guesses $T$ (see Algorithm \ref{alg:tmd_doubling}). With probability $1 - O(\delta)$, we have anytime regret bound
\begin{equation*}
\text{Reg}(T) = \tilde{O}\left( \left[ L' + \sqrt{L'Q} + \sqrt{L'} LD / \sqrt{Q}  \right] \|\mathbf{1}_K\|  \middle/ T^{1/3} \right) + \tilde{O} \left( L \|\mathbf{1}_K\| D\sqrt{\Gamma S A}\middle/ \sqrt{T}\right).
\end{equation*}
Choosing $Q =L$ gives regret bound $\tilde{O} ( [L' + \sqrt{L' L}D ] \cdot  \|\mathbf{1}_K\| /T^{1/3}) + \tilde{O}( L\|\mathbf{1}_K\| D \sqrt{\Gamma S A} /\sqrt{T} ) $.
\end{restatable}
The Theorem is proved in Appendix \ref{app:tgd}. The analysis exploits the fact that the Fenchel dual $F^*$ is 1-smooth w.r.t. $\|\cdot\|_*$, which implies that the sub-gradients do not vary significantly across time. This means that the total number of episodes is $\tilde{O}(T^{2/3})$, leading to a $\tilde{O}(1/T^{1/3})$ regret. 

\section{Applications of MDPwGR, and how does {\sc Toc-UCRL2} fare in those?}\label{sec:applications_model}
In this Section, we demonstrate the versatility of the MDPwGR problem model in modeling different problems in online Markovian environments. We provide formulation for multi-objective optimization, maximum entropy exploration as well as MDPwSR with knapsack constraints in the ``large volume regime'', which is defined in the forthcoming discussions. We provide the corresponding regret bounds, and which are based on applying Theorems \ref{thm:fw_brief}, \ref{thm:TGD}, \ref{thm:TMD}. The supplementary details of the discussions are given in Appendix \ref{app:applications_model}.

\subsection{Multi-objective optimization }
MDPwGR can be used to model multi-objective optimization (MOO) in online Markovian environments. In an online multi-objective MDP problem, the outcome $V_{t, k}(s_t, a_t)$ represents the reward earned by the agent for the $k^\text{th}$ objective at time $t$. The agent wishes to simultaneously maximizes each of the $K$ average cumulative rewards:
\begin{equation}\label{eq:moo_rewards_1}
\max \left\{ \frac{1}{T}\sum^T_{t=1} V_{t, 1}(s_t, a_t), \frac{1}{T}\sum^T_{t=1} V_{t, 2}(s_t, a_t), \ldots,\frac{1}{T}\sum^T_{t=1}  V_{t, K}(s_t, a_t) \right\}.
\end{equation}
Since $V_1(s, a), \ldots, V_K(s, a)$ can be arbitrarily correlated, it is generally impossible to attain the optimum for each and every objective. Thus, problem (\ref{eq:moo_rewards_1}) is not well-defined in general. 

The desired objective (\ref{eq:moo_rewards_1}) can be modeled by a surrogate function $g$ that appropriately combines the $K$ objective together. We consider surrogate functions that model target set objectives and fairness objectives.

\textbf{Target Set Objectives. }We model MOO as minimizing a certain distance measure between $V_{1:T}$ and a certain target set $Z\in [0, 1]^K$. A prime example for $Z$ is $$Z := \{w\in [0, 1]^K : w_k \geq \zeta_k \text{ for all }k\},$$ where $\zeta = (\zeta_1, \ldots, \zeta_K) \in [0, 1]^K$ is a vector of key performance indexes (KPIs) for the agent. By staying in the target set $Z$, the agent achieves all $K$ KPIs, i.e.  $\sum^T_{t=1} V_{t, k}(s_t, a_t) / T \geq \zeta_k$ for all $k\in \{1, \ldots, K\} $. 

More concretely, let's consider minimizing the squared Euclidean distance between $V_{1:T}$ and $Z$. A way to express this objective is to set $g_{\text{SE}}(w) := 1- \frac{1}{K}\sum^K_{k=1} \max \{0, \zeta_k - w_k   \}^2 . $
Evidently, we have $g_{\text{SE}}(w) \in [0, 1]$ for all $w\in [0, 1]^K$.
Suppose the underlying norm is the Euclidean norm $\|\cdot\|_2$. Since $\nabla g_{\text{SE}}(w) = - (2/K) \cdot (\max\{0, \zeta_k - w_k \})^K_{k=1}$, the reward function $g_{\text{SE}}$ is $L = (2/\sqrt{K})$-Lipschitz and $\beta = (2 / K)$-smooth w.r.t.  $\|\cdot\|_2$. By Theorem \ref{thm:fw_brief}, Algorithm {\sc Toc-UCRL2} with OCO oracle $\FW$ and $Q = 2/\sqrt{K} $ satisfies:
$$\text{Reg$(T)$ for MOO with $g_{\text{SE}}$} = \tilde{O}\left(D\sqrt{\Gamma S A}\middle/\sqrt{T}\right).$$ The regret bound is of the same order (up to multiplicative constants and logarithmic terms) as the regret bound of the MDPwSR by \citep{JakschOA10}, if we replace $\Gamma$ by its upper bound $S$. Somewhat intriguingly, the dependence on $\beta, L, \|\mathbf{1}_K\|_2 = \sqrt{K}$ in the regret bound disappears, as these quantities nullifies each other in the fractional expressions in the Theorem.

Definitely, the distance measure can be defined differently. For a fixed norm $\|\cdot\|$ on $\mathbb{R}^K$, we can define $g_{\|\cdot\|}(w) := 1 - \min_{z\in Z}\| w - z \|$. Clearly, $g_{\|\cdot\|}$ is 1-Lipschitz w.r.t. $\|\cdot\|$, but $g_{\|\cdot\|}$ is not necessarily smooth w.r.t. $\|\cdot\|$. Thus, the OCO oracle $\TMD$ (or $\TGD$ if $\|\cdot\| = \|\cdot\|_2$) has to be employed, resulting in the regret bounds in Theorem \ref{thm:TMD}. 

\textbf{Fairness Objectives.} Another way to model MOO is to 
to ensure fairness among objectives. In the following, we consider a certain special case of \citep{Busa-feketeSWM17} on the surrogate reward function for fairness. Let $\kappa \in \{1, \ldots, K\}$, and set $g_{(\kappa)}(w) = \sum^\kappa_{k=1}w_{\pi_w(k)}$, where $\pi_w$ is a permutation of $\{1, \ldots, K\}$ such that $w_{\pi_w(1)}\leq w_{\pi_w (2)}\leq \ldots \leq w_{\pi_w (K)}$. Clearly, function $g_{(\kappa)}$ is concave for any $\kappa$. 
In addition, function $g_{(\kappa)}$ is $\kappa$-Lipschitz w.r.t. $\|\cdot\| = \|\cdot\|_\infty$, see Appendix \ref{app:fairness} for a proof. Interestingly, for all $w\in [0, 1]^K$ and all $\theta\in \partial g_{(\kappa)}(w)$, we have $\|\theta\|_* = \|\theta\|_1 = \kappa$. Unfortunately, $g_{(\kappa)}$ is in general not smooth w.r.t. to $\|\cdot\|_\infty$. By Theorem \ref{thm:TMD}, the application of the OCO-oracle $\TMD$ mirror map $F_\infty$ defined in (\ref{eq:entropy}) in Appendix \ref{app:tmd_background} gives $$\text{Reg$(T)$ for MOO with fairness} = \tilde{O} \left( \kappa D  \middle/T^{1/3}\right) + \tilde{O}\left( \kappa D \sqrt{\Gamma S A} \middle/\sqrt{T}\right) .$$ 
It is interesting to extend the subtle procedure for inducing exploration in \citep{Busa-feketeSWM17} to the current MDP setting, in order to optimize their more general surrogate functions for fairness.


\subsection{Maximum entropy exploration }\label{app:MaxEnt}
In certain applications, it is desirable to explore the latent MDP uniformly. Motivated by \citep{HazanKSS18}, we phrase the objective of enforcing uniform exploration as the maximum entropy exploration (MaxEnt) problem, which requires maximizing the entropy function on the \emph{empirical visit distribution}. The empirical visit distribution is the empirical distribution on the frequency of visiting each state. To avoid triviality, we assume $S > 1$.

We use MDPwGR to model MaxEnt in an online setting. For each $s, a$, the outcome vector $V(s, a)\in [0, 1]^\SSS$ is indexed by the state space $\SSS$, and we define $V(s, a) := \mathbf{e}_s$, which is deterministic. The average outcomes $\bar{V}_{1:T} = \sum^T_{t=1}V(s_t, a_t) / T$ represents the empirical visit distribution over $T$ time steps. Next, we define the objective. For a probability distribution $P = \{P_s\}_{s\in \SSS}\in \Delta^\SSS$, its entropy is $H(P) := - \sum_{s\in \SSS} P_s \log P_s$. By convention, we define $0\log 0 =0$. The online MaxEnt problem can then be modeled as MDPwGR, where we aim to maximize $H(\bar{V}_{1:T})$.  

The entropy function has an unbounded gradient, which hinders online optimization. Thus, we first consider a smoothed version $H_\mu$ of $H$ proposed by \citep{HazanKSS18}, where $0 < \mu \leq 1$:
\begin{equation*}
H_\mu(P) := \frac{1}{\log S}\sum_{s\in \SSS} P_s \cdot \log \frac{1}{P_s + \mu}.
\end{equation*}
By inspecting the gradient and the Hessian of $H_\mu$ (see Lemma 4.3 in \citep{HazanKSS18}), we find that $H_\mu$ is $L_\mu = (\log(1/\mu) / \log S)$-Lipchitz w.r.t. $\|\cdot\|_1$ and  is $\beta_\mu = (1 / (\mu \log S))$-smooth w.r.t. $\|\cdot\|_1$. Theorem \ref{thm:fw_brief} shows that Algorithm {\sc Toc-UCRL2} with OCO oracle $\FW$ and $Q = L_\mu$ satisfies regret bound:
\begin{align}
\text{Reg$(T)$ for MaxEnt on $H_\mu$} &= \tilde{O} \left( \sqrt{\beta_\mu L_\mu } \|\mathbf{1}_S\|^{3/2}_1 D \middle/\sqrt{T}\right)  +  \tilde{O} \left( L_\mu \| \mathbf{1}_S\|_1 D\sqrt{\Gamma S A} \middle/\sqrt{T} \right)\label{eq:reg_maxent_raw}\\
& = \tilde{O} \left( \sqrt{\log(1/\mu) / \mu }S^{3/2} D \middle/\sqrt{T}\right)  +  \tilde{O}\left( \log(1/\mu) D\sqrt{\Gamma S^3 A}  \middle/\sqrt{T} \right) \nonumber.
\end{align}
Nevertheless, in MaxEnt, there is no uncertainty about $v$. In addition, for certain we have $V(s, a) \in B(1, \|\cdot\|_1)$. Therefore, we can in fact implement a refined {\sc Toc-UCRL2}, where $H^v_m$ is replaced by a collection of singletons $ H^v_m(s, a) = \{\mathbf{e}_s\}$ for each episode $m$. By tracing the analysis, the $\|\mathbf{1}_S\|_1$ terms in (\ref{eq:reg_maxent_raw}) can be replaced by 1, leading to 
\begin{equation*}
\text{Refined Reg$(T)$ for MaxEnt on $H_\mu$} =  \tilde{O} \left( \sqrt{\log(1/\mu) / \mu } D \middle/\sqrt{T}\right)  +  \tilde{O}\left( \log(1 /\mu ) D\sqrt{\Gamma S A} ]\middle/\sqrt{T} \right).
\end{equation*}
If $\mu = \Omega(1/S)$, then the refined regret bound is $\tilde{O}(D\sqrt{\Gamma S A}/\sqrt{T})$, which matches the regret bound in \citep{JakschOA10} when we replace $\Gamma$ by its upper bound $S$. 

Let's go back to the objective of maximizing $H(\bar{V}_{1:T})$. Now, Lemma 4.3 in \citep{HazanKSS18} states that $| H_\mu(P) - H(P) | \leq S\mu $ for all $P\in \Delta^\SSS$. Fix $\mu(T) = 1 / (S^{2/3} T^{1/3})$. Assuming the knowledge of $T$, the execution of the refined {\sc Toc-UCRL2}, with OCO oracle $\FW$ and $Q = L_{\mu(T)}$,
on objective function $H_{\mu(T)}$ yields the following regret bound on the entropy $H$:
\begin{equation}\label{eq:reg_entropy_refined}
\text{Refined Reg$(T)$ for MaxEnt on $H$} =  \tilde{O} \left(  D S^{1/3} \middle/ T^{1/3}\right)  +  \tilde{O}\left( D\sqrt{\Gamma S A} \middle/\sqrt{T}\right).
\end{equation}
The assumption of knowing $T$ can be removed by a doubling trick that guesses $T$. In fact, we envision that an execution where we dynamically set $Q= L_{\mu(t)}$ and perform gradient update using $H_{\mu(t)}$ at time $t$ would give an anytime implementation with the same regret bound (\ref{eq:reg_entropy_refined}).

We compare the results above with \citep{HazanKSS18}, and spell out some subtle differences. \citep{HazanKSS18} focus on an offline and $\gamma$-discounted setting (where $0<\gamma<1$), which is different from our online and un-discounted setting. In addition, \citep{HazanKSS18} seek to maximize the entropy of the discounted state-visit distribution, which is not an empirical distribution, different from our online setting. 
In \citep{HazanKSS18}, they show that $\tilde{O}( S^2 A / [\epsilon^3 (1- \gamma)^2] + S^3 / \epsilon^6  )$ time steps are required to compute a policy with additive error of $\epsilon$ on their discounted objective. 

Finally, we remark that similar results for other entropy-based objectives, such as the minimization of KL-divergence or cross-entropy measure w.r.t. a fixed distribution on $\SSS$ (see Section 3.1 in \citep{HazanKSS18}), hold in our online setting. Nevertheless, the approach seems hard to be generalized to the setting by \citep{TarbouriechL19}, which makes their objective function smooth by inserting certain parameter that concerns the mixing time of any stationary policy.

\subsection{MDPwSR with knapsack constraints, a.k.a. MDPwK}\label{app:MDPwK}
We consider the following MDPwSR problem with knapsack constraints, dubbed the MDPwK problem. The MDPwK problem is a common generalization of MDPwSR \citep{JakschOA10} and the stochastic Bandits with Knapsacks (BwK), studied in  \citep{BadanidiyuruKS13,BadanidiyuruLS14,AgrawalD14,AgrawalD16, AgrawalDL16}. 

\textbf{Problem Model. }Essentially, MDPwK is the MDPwSR problem with $(K-1)$ resource constraints, where $K \geq 2$. MDPwK is defined same as MDPwGR with $K\geq 2$, apart from the following five differences:
\begin{itemize}
\item \textbf{Specialized definition of $V(s,a)$. }In MDPwK, the vectorial outcomes have a specific meaning. We have $V(s, a) = (R(s, a), C(s, a) = (C_k(s, a))^{K-1}_{k=1})\in [0, 1]^K$. The scalar random variable $R(s, a)\in [0, 1]$ is the reward for performing action $a$ at state $s$. For each $k\in \{1, \ldots, K-1\}$, the scalar random variable $C_k(s, a)\in \{0, 1\}$ is the amount of resource $k$ consumed in performing action $a$ at state $s$.
\item \textbf{Null actions $a_0$. }For each $s\in \SSS$, we assume that there is a null action $a_0\in \AAA_s$, where $R(s, a_0) = 0$ and $C(s, a_0) = \mathbf{0}_{K-1}$ with certainty. We do not pose any restriction on the associated state transition, and $p(\cdot | s, a_0)$ is a general distribution on $\SSS$. The null action ensures the feasibility of MDPwK. The agent knows $a_0$ and also knows that $V(s, a_0) = \textbf{0}_K$ for each $s\in \SSS$, but she does not know $p(\cdot | s, a_0)$ for any $s\in \SSS$. In addition, to avoid free actions that trivialize the problem, we assume that, if $r(s, a) > 0$, then $\Pr[C(s, a) = \textbf{0}_K ] =0  $.
\item \textbf{Dynamics. }The dynamics of MDPwK is the same as the dynamics of MDPwGR, with the additional notions of resource consumption and stopping time $\tau$. The agent is provided with $b T $ units of resource $k$, for each $k\in \{1\ldots, K -1 \}$, and $b\in (0, 1)$ is a constant independent of $T$.  These resources cannot be replenished during the subsequent $T$ rounds. At each time $t$, after $V_t(s_t, a_t) = (R_t(s_t, a_t), C_t(s_t, a_t))$ is revealed, the agent consumes $C_{t, k}(s_t, a_t)$ units of resource $k$.  The agent stops the process when any of the $K-1$ resource is depleted, or when time $T$ is reached. We denote this stopping time as $\tau$.
\item \textbf{Knowledge of $T$. } The agent knows the horizon $T$.
\item \textbf{Objective. } The agent's objective is to  maximize $\sum^\tau_{t=1} R_t(s_t, a_t)$, while satisfying all $K-1$ resource constraints. That is, $\sum^\tau_{t=1} C_{t, k}(s_t, a_t)\leq bT$ for all $k\in\{1, \ldots, K-1\}$.
\end{itemize}
We supplement with some remarks on the amount of initial inventory. Similar to \citep{BadanidiyuruKS18}, we normalize the resource levels to be equal across resources. 
By assuming $b\in (0, 1)$ and $b = \Theta(1)$, we are assuming the ``large volume regime''. This is a stronger assumption than the $m$-arm bandit with knapsack setting \citep{BadanidiyuruKS18, AgrawalD14}, who only require $b\in (0, 1)$ to satisfy $ b\geq  c' m \log (KT)  / T$ for some absolute constant $c'$, or \citep{AgrawalD16}, who only require $b\in (0, 1)$ to satisfy $ b\geq  c'' m  / T^{1/4}$ for some absolute constant $c''$, or \citep{ImmorlicaSSS18}
, who only require $b\in (0, 1)$ to satisfy $ b\geq  c''' m \log T / T^{1/4}$ for some absolute constant $c'''$.

By a modification of {\sc Toc-UCRL2} (see Appendix \ref{app:supp_MDPwK}), it is possible to achieve near optimality on MDPwK.
\begin{proposition}\label{prop:MDPwK}
Consider an instance of MDPwK, where the underlying MDP is communicating with diameter $D$. A modification of {\sc Toc-UCRL2}, displayed in Algorithm \ref{alg:MDPwK}, incurs a regret $\tilde{O}\left(K D \cdot T^{2/3}\right) +  \tilde{O}\left(K D \sqrt{\Gamma S A T} \right)$ with probability at least $1 - O(\delta)$.
\end{proposition}
We envision that the factor of $K$ can be saved by following \citep{AgrawalD16}, at the expense of more technical works and making the $\TMD$ oracle more transparent in the analysis. We provide the Algorithm \ref{alg:MDPwK} and its analysis in Appendix \ref{app:supp_MDPwK}.

\section{Some Concluding Thoughts}\label{sec:conclusion}
In the manuscript, we study online Markov decision processes with global concave reward functions (MDPwGR). We reveal that the reachability assumption on the underlying MDP has a profound impact on the optimization of the global objective. In addition, we flesh out a delicate trade-off in alternating among different actions, in order to balance the vectorial outcomes while maintaining near-optimality. We propose Algorithm {\sc Toc-UCRL2}, a no-regret algorithmic framework that is able to incorporate a wide variety of gradient descent algorithms. The framework hinges on a novel gradient threshold procedure, which handles the stated trade-off by delaying gradient updates appropriately. 

We highlight some interesting research topics for further study: 

\textbf{What about Dynamic $Q$s? }A natural generalization to {\sc Toc-UCRL2} is to set the gradient threshold $Q $ dynamically, possibly depending on $\tau(m)$, the starting time index of an episode $m$. Under $\FW$, as well as $\TGD, \TMD$ with generally learning rates, we have tried setting $Q(m) = \Theta(\tau(m)^a)$ for general $a$ (both negative or positive). While regret bounds can still be established by appropriately generalizing Lemma \ref{lemma:club_heart} and the corresponding Lemma for establishing the upper bound $M(T)$, the endeavor does not bring about any improvement compared to the current regret bounds (see Theorems \ref{thm:fw_brief}, \ref{thm:TGD}, \ref{thm:TMD}) in the manuscript. Nevertheless, we do not exclude the possibility of benefiting our algorihtmic framework with dynamic $Q$s under other OCO oracles. 

\textbf{Incorporating Recent Improvements. }While we incorporate UCRL2 \citep{JakschOA10} in our algorithmic framework, we note that recent works \citep{AgrawalJ17, FruitPLO18, FruitPL18} have brought about improvements on online Markov decision processes with scalar rewards. \citep{AgrawalJ17} shaves off a $\sqrt{S}$ factor in the regret bound by \citep{JakschOA10} under the same communicating MDP assumption. \citep{FruitPLO18} derive regret bounds that improves upon the dependence on $D$ under certain mild parametric assumptions, in the weakly-communicating MDP setting. \citep{FruitPL18} provide regret bounds under an even more general assumption of non-communicating MDPs, assuming that the agent starts at a recurrent state. While we focus on handling exploration-exploitation trade-off in face of a global reward function in Markovian environments, we envision that improved regret bounds can be achieved by incorporating the mentioned recent works. The incorporation of \citep{FruitPLO18} could require some new definitions of span in our setting where the scalarization of the objective changes every time step. The incorporation of \citep{FruitPL18} could require some additional ideas on how to distinguish recurrent states from transient states while maintaining both the balance amounts outcomes and near-optimality simultaneously.
 
\textbf{Understanding the Impact of Reachability Assumptions. }We note that the delicate trade-off in alternating among different actions is mainly due to the assumption of communicating MDPs, and the trade-off is absent when we impose a stronger assumption of unichain MDPs. Since the unichain assumption is quite commonly imposed in the literature, it is interesting to see if stronger regret bounds can be obtained under the unichain assumption. We believe that unichain and communicating MDPs set the boundary on the difficulty of MDPwGR. 

Under certain additional mixing time assumptions (see \citep{Ortner18,TarbouriechL19}), we envision a no-regret primal algorithm for MDPwGR for unichain MDPs, where the agent employs a policy constructed using the optimal solution of an estimated $(\primal_{\cal M})$. We remark that, in order to compare the performance of an algorithm directly $\text{opt}(\primal_{\cal M})$ (in contrast to our method that compares with the dual of $(\primal_{\cal M})$), certain mixing time assumptions are necessary, since $x^*$, an optimal solution of $(\primal_{\cal M})$, represents a stationary distribution. Reflecting upon the instance ${\cal M}_\text{MG}$, we believe that such an approach does not work in the communicating MDP case, where $x^*$ does not correspond to a stationary distribution in general, and stationary policies could be sub-optimal (see Claim \ref{claim:cont_small_eg}).

\textbf{Regret Lower Bounds. }Another important research question is to derive regret lower bounds. It is interesting to see if it is possible to ascertain a $\Omega(T^{-1/3})$ regret lower bound for the general concave reward case, by incorporating the regret lower bounds on learning MDPs \citep{JakschOA10} with our lower bound examples for balancing vectorial outcomes for MDPwGR, see Claims \ref{claim:small_eg}, \ref{claim:cont_small_eg}, \ref{claim:eg_Q0}.  In addition, it is interesting to investigate if the dependence on $\beta$ is tight in the case of $\beta$-smooth $g$, while we are not aware of any regret lower bound for the online Frank-Wolfe algorithm that involves $\beta$. 

\textbf{Refined Regret Bounds for Specific Applications. }We have highlighted the applications of MDPwGR on multi-objective optimization, maximum entropy exploration (MaxEnt) and constrained optimization (MDPwK), and we have provided the corresponding regret bounds under suitable applications of {\sc Toc-UCRL2}. It is plausible that these regret bounds can be refined in certain settings, as highlighted for the cases of MaxEnt in Appendix \ref{app:MaxEnt} and MDPwK in Appendix \ref{app:MDPwK}. 

We would like to elaborate further on MDPwK. The dual-based approach for MDPwK in general requires estimating the optimal value of the offline benchmark problem (see $\primal_\text{C}(b)$ in Appendix \ref{app:MDPwK}). While in the BwK it is possible to estimate such optimal value within a constant factor approximation in $O(K\log T)$ rounds by playing random actions, the approximation task in the case of MDPwK is much less straightforward, as the agent cannot freely alternate among actions. We envision that the large volume regime, which requires $b = \Omega(1)$, could be relaxed if we can accomplish such an approximation. 

\textbf{Practical Implementation. }The implementation of our gradient threshold procedure is independent of $S, A, K$, but only depends on the reward function $g$. We believe that the procedure has practical values, and can be incorporated with existing heuristics on reinforcement learning algorithms assuming scalar rewards for optimizing global objectives, even when the state or action space is large.

\bibliography{mdp_cr_arxiv_v1}


\newpage
\newpage
\appendix

\startcontents[appendices]
\printcontents[appendices]{l}{1}{\section*{Appendices}\setcounter{tocdepth}{3}}

\newpage


\section{Details about the MDPwGR Problem}
The Appendix section is organized as follows. First, we motivate the offline benchmark $\text{opt}(\primal_{\cal M})$ in Sections \ref{app:offlineremark} -- \ref{app:pflemmamc}. Then, we substantiate our claims in the discussion on the challenges in MDPwGR in Section  \ref{app:pfclaimsmalleg}. Finally, in Section \ref{app:applications_model}, we supplement the discussions on the applications of the MDPwGR problem model in Section \ref{sec:applications_model} in the main text.

\subsection{A remark on Theorem \ref{thm:benchmark}, and a dual of $(\primal_{\cal M})$}\label{app:offlineremark}
\thmbenchmark*
\textbf{A remark on the additive term.} Before we prove the Theorem, we first remark that an additive term to the offline benchmark $\text{opt}(\primal_{\cal M})$ in the Theorem is necessary, in the sense that there exists communicating MDPwGR instances for which $\mathbf{E}[g(\sum^T_{t=1}V_t(s_t, a_t)/T)] > \text{opt}(\primal_{\cal M})$. Essentially, the requirement of an additive term is due to the fact that we do not incorporate the initial state $s_1$ in the offline benchmark, rather than an artifact of our analysis. 

Consider the following deterministic instance ${\cal M_{\text{det}}}$ with scalar rewards, where the underlying MDP constitutes a directed cycle. In the instance, there are $D > 1$ states $1, \ldots, D$. Each state is associated with only one action (say denoted as action $a$), so the only feasible policy is to always take action $a$. The instance forms a directed cycle of length $D$, in the sense that $p(i + 1 | i, a) =1$ for $i = 1, \ldots, D-1$, and $p(1 | D, a) = 1$. When the agent takes action $a$ at state $1$, there is a scalar reward of 1, but when she takes action $a$ at any other state, there is scalar reward of $0$. Finally, we let $g(w) = w$, which means that we are maximizing the total scalar reward. 

On the one hand, we clearly have  $\text{opt}(\primal_{\cal M_{\text{det}}}) = 1/D$. On the other hand, supposing that the agent starts at state $s_1 = 1$, her global reward at time $T = (j-1)D + 1$ is equal to $j / [(j-1)D + 1]$, which is clearly larger than $1/D$, since $D > 1$. 

\textbf{Dual Formulation of $(\primal_{\cal M})$. }We prove the Theorem by considering a dual formulation to the primal problem $(\primal_{\cal M})$ which defines the offline benchmark. To formulate a dual, we first recall the 
Sion Minimax Theorem \citep{Sion58}:
\begin{proposition}{\citep{Sion58}}\label{prop:sion}
Let sets $\mathcal{X}\subset \mathbb{R}^m, \mathcal{Y} \subset\mathbb{R}^n$ be convex and compact. Consider function $f : \mathcal{X}\times \mathcal{Y}\rightarrow \mathbb{R}$, where $f(\cdot,  y)$ is a convex function of $ x\in \mathcal{X}$ for each fixed $y \in \mathcal{Y}$, and $f( x, \cdot)$ is a concave function of $ y\in \mathcal{Y}$ for each fixed $ x\in \mathcal{X}$. The following equality holds:
$$
\min_{x \in \mathcal{X}}\max_{y \in \mathcal{Y}}f(x,  y) =\max_{ y\in \mathcal{Y}} \min_{x\in \mathcal{X}}f( x, y).
$$
\end{proposition}
Next, we provide the definition of the Fenchel dual $g^*$ of $g$, which is defined on the domain $B(L, \|\cdot\|_* )$ since $g$ is $L$-Lipschitz w.r.t the norm $\|\cdot\|$. The Fenchel dual $g^*$ is defined as
$$
g^*(\theta) := \max_{w\in [0, 1]^K} g(w) + \theta^\top w.
$$
Now, $g(w) = \min_{\theta \in B(L, \|\cdot\|^*)} g^*(\theta) - \theta^\top w$, and recall that $R(\primal_{\cal M})$ denotes the feasible region of $(\primal_{\cal M})$. Note that $R(\primal_{\cal M})$  is defined by a finite set of linear constraints. By applying Proposition \ref{prop:sion}, we can rephrase the primal problem $(\primal_{\cal M})$ as follows:
\begin{align}
\text{opt}(\primal_{\cal M})  = \max_{x\in {\cal P}} g(x) & = \max_{x\in R(\primal_{\cal M} }\min_{\theta\in B(L, \|\cdot\|^*)} \left\{ g^*(\theta) - \theta^\top x\right\}\nonumber\\
& = \min_{\theta\in B(L, \|\cdot\|^*)}\max_{x\in   R(\primal_{\cal M} }\left\{ g^*(\theta) - \theta^\top x \right\} \nonumber\\
& = \min_{\theta\in B(L, \|\cdot\|^*)}\left\{ g^*(\theta) - \max_{x\in R(\primal_{\cal M} } \theta^\top x\right\}\label{eq:inner_max}
\end{align}
Note that the inner maximization problem in (\ref{eq:inner_max}) is a linear program. By taking the dual of this linear program, we formulate the dual minimization problem $(\dual_{\cal M})$, which is the dual of $(\primal_{\cal M})$:
\begin{subequations}
\begin{alignat}{2}
(\dual_{\cal M})\text{: \;} & \min_{\theta, \phi, \gamma} ~ g^*(\theta) + \phi  \nonumber\\
\text{s.t. } &\phi + \gamma(s) \geq - \theta^\top v(s, a) + \sum_{s'\in\SSS} p(s' | s,a) \gamma(s') &\quad &\forall s \in \SSS, a\in \AAA_s \label{eq:D.1}\\
               &\theta \in B(L, \|\cdot\|_*), \qquad \phi \text{ free, } \qquad \gamma(s) \text{ free } &\quad & \forall s\in \SSS \label{eq:D.2}
\end{alignat}
\end{subequations}
Our proof of Theorem \ref{thm:benchmark} requires analyzing the optimal solution to the dual problem $(\dual_{\cal M})$. The analysis hinges on the follwoing property of a feasible solution to the dual problem:
\begin{lemma}\label{lemma:mc}
Let $(\theta, \phi,\gamma)$ be a feasible solution to the dual problem $(\dual_{\cal M})$, where the underlying MDPwGR instance is communicating with diameter $D$. We have $$\max_{s, s'\in \mathcal{S}} \left\{ \gamma(s) - \gamma(s') \right\} \leq (L \|\mathbf{1}_K\|  + \phi ) D.$$ 
\end{lemma}
We delay the proof of Lemma \ref{lemma:mc} to Section \ref{app:pflemmamc}.
\subsection{Proof of Theorem \ref{thm:benchmark}}\label{app:pfthmbenchmark}
Let $(\theta^*, \phi^*, \gamma^*)$ be an optimal solution to the dual problem $(\dual_{\cal M})$. By our formulation, we know that 
$$
\text{opt}(\primal_{\cal M}) = \text{opt}(\dual_{\cal M}) = g^*(\theta^*) + \phi^*.
$$
Now, consider an arbitrary non-anticipatory policy, and denote the (random) sequence of states and actions across the first $T$ time steps as $\{(s_t, a_t)\}^T_{t=1}$. The expected reward under the policy in the first $T$ time steps can be bounded as:
\begin{align}
&\mathbb{E}\left[g\left(\frac{1}{T}\sum^T_{t=1}V_t(s_t, a_t)\right)\right] \nonumber\\
= &\mathbb{E}\left[\min_{\theta\in B(L, \|\cdot\|_*)} g^*( \theta ) - \theta^\top  \left(\frac{1}{T}\sum^T_{t=1}V_t(s_t, a_t)\right) \right]\nonumber\\
\leq &g^*(\theta^*) + (- \theta^*)^\top \mathbb{E}\left[ \frac{1}{T}\sum^T_{t=1}V_t(s_t, a_t) \right]\nonumber\\
= &g^*( \theta^*) + (- \theta^*)^\top \mathbb{E}\left[ \frac{1}{T}\sum^T_{t=1} \mathbb{E}\left[ V_t(s_t, a_t) \mid (s_t, a_t) \right] \right]\nonumber\\
= & g^*(\theta^*) + (- \theta^*)^\top \mathbb{E}\left[ \frac{1}{T}\sum^T_{t=1} v(s_t, a_t) \right]\nonumber\\
\leq & g^*( \theta^*) + \frac{1}{T}\sum^T_{t=1}\mathbb{E}\left[\phi^* + \gamma^*(s_t) - \sum_{s\in \SSS}\gamma^*(s)p(s | s_t, a_t)\right]\nonumber\\
= & g^*(\theta^*) + \phi^* + \frac{1}{T}\mathbb{E}\left[\gamma^*(s_1) - \sum_{s\in \SSS}\gamma^*(s)p(s | s_T, a_T ) + \sum^{T-1}_{t=1}\left[\gamma^*(s_{t+1}) - \sum_{s\in \SSS}\gamma(s)p(s | s_t, a_t ) \right]\right] \nonumber\\
= & g^*( \theta^*) + \phi^* + \frac{1}{T}\mathbb{E}\left[ \gamma^*(s_1) - \sum_{s \in \SSS}\gamma^*(s)p(s | s_T, a_T )\right]\nonumber\\
\leq & g^*( \theta^*) + \phi^* + \frac{(L\|\mathbf{1}_K\| + \phi^* )D}{T}= \text{opt}(\primal_{\cal M}) + \frac{(L\|\mathbf{1}_K\| + \phi^* )D}{T}.\label{eq:last_ineq}
\end{align}
The last step (\ref{eq:last_ineq}) uses Lemma \ref{lemma:mc}. Finally, by our application of linear duality in formulating $(\dual_{\cal M})$,  $$\phi^* = \max_{x\in R(\primal_{\cal M})}\left\{\sum_{s, a}(-\theta^*)^\top v(s, a) x(s, a)\right\} \leq \max_{s, a} \left\{(-\theta^*)^\top v(s, a)\right\} \leq L\|\mathbf{1}_K\|.$$  Applying this upper bound on $\phi^*$, the Theorem is proved.$\hfill \Box$

\subsection{Proof of Lemma \ref{lemma:mc}}\label{app:pflemmamc}
Let's fix a pair of states $\bar{s}, \bar{s}'$. For a stationary $\pi$, recall the random variable $$\Lambda(s' | \pi, s) := \min\left\{ t : s_{t+1} = s' , s_1 = s, s_{\tau + 1} \sim p(\cdot | s_\tau, \pi(s_\tau)) \text{ $\forall\tau$}\right\}.$$ 
Now, let stationary policy $\pi$ satisfy $\mathbb{E}[\Lambda(\bar{s}' | \pi, \bar{s})]\leq D$. In the following, we first define some notations for the stationary policy $\pi$, and then proceed with an induction argument that proves the Lemma. 

For each state $s\in \SSS$, the stationary policy $\pi$ defines a probability distribution $\{ q^\pi(a | s)\}_{a\in \AAA_s} \in \Delta^{\AAA_s}$ over the action set $\AAA_s$, where $q^\pi(a | s )$ is the probability of taking action $a$ when the agent is at state $s$ and employs policy $\pi$. The stationary policy $\pi$ gives rise to a Markov Chain $\{s_t\}^\infty_{t= 1}$ with initial state $s_1 = \bar{s}$, time homogeneous reward $r^\pi$ and transition probability $p^\pi$:
\begin{align}
r^\pi(s ) & := \sum_{a\in \AAA_s  } (-\theta)^\top v(s, a) q^\pi(a | s)\qquad \text{ for each $s \in \SSS$},\nonumber\\
p^\pi(s' | s) & := \sum_{a\in \AAA_s  }  p(s' | s, a) q^\pi(a | s)\qquad \text{ for each $s, s' \in \SSS$}\nonumber.
\end{align}
Note that $r^\pi(s)\geq  -r_{\text{max}}$ for each $s\in \SSS$. Now, for each $s$, we take linear combination over the inequalities (\ref{eq:D.1}) indexed by $\{(s, a)\}_{a\in \AAA_s}$ with the probabilities $\{ q^\pi(a | s)\}_{a\in \AAA_s}$, and arrive at the inequality 
$$
\gamma(s) \geq (r^\pi(s) - \phi) + \sum_{s'\in \mathcal{S}} \gamma(s') p^\pi(s' | s) .
$$ 

After setting up the notations, we proceed to an induction argument for proving the Lemma. We claim that, for any positive integer $t$, the following inequality holds true:
\begin{align}
\gamma(\bar{s}) &\geq -(r_\text{max} + \phi)\sum^{t}_{\tau=1}\Pr[ \Lambda(\bar{s}' | \pi, \bar{s}) \geq \tau] + \gamma(\bar{s}')\sum^t_{\tau=1}\Pr[\Lambda(\bar{s}' | \pi, \bar{s}) = \tau] \nonumber\\
&\qquad \qquad \qquad \qquad \qquad \qquad \qquad + \sum_{s\in \SSS\setminus\{\bar{s}'\}}\gamma(s)\Pr[\Lambda(\bar{s}' | \pi, \bar{s}) > t, s_t = s] \label{eq:lemma_markov_ind_claim} .
\end{align}
We prove the inequality (\ref{eq:lemma_markov_ind_claim}) by an induction on $t$. Now, for $t=1$, we clearly have
\begin{align}
\gamma(\bar{s}) &\geq -(r_\text{max} + \phi) + \sum_{s\in \SSS}p^\pi(s | \bar{s}) \gamma(s)\nonumber\\
& = -(r_\text{max} + \phi) + \gamma(\bar{s}')p^\pi(\bar{s}' | \bar{s}) + \sum_{s\in\mathcal{S}\setminus \{\bar{s}'\}} \gamma(s) p^\pi(s | \bar{s}) \nonumber\\
& = -(r_\text{max} + \phi)\Pr\left[\Lambda(\bar{s}' | \pi, \bar{s}) \geq 1\right] + \gamma(\bar{s}')\Pr\left[\Lambda(\bar{s}' | \pi, \bar{s}) = 1\right] \nonumber\\
&\qquad \qquad \qquad \qquad \qquad \qquad \qquad + \sum_{s\in\mathcal{S}\setminus \{\bar{s}'\}}\gamma(s) \Pr\left[\Lambda(\bar{s}' | \pi, \bar{s}) > 1, s_1 = s\right]\nonumber,
\end{align}
proving the case when $t=1$. Now, assuming that the inequality (\ref{eq:lemma_markov_ind_claim}) is true for $t$, we prove the inequality for the case with $t+1$ by expanding the last term in (\ref{eq:lemma_markov_ind_claim}):
\begin{align}
&\sum_{s\in \SSS\setminus\{\bar{s}'\}}\gamma(s)\Pr[\Lambda(\bar{s}' | \pi, \bar{s}) > t, s_t = s]\nonumber\\
\geq & \sum_{s\in \SSS\setminus\{\bar{s}'\}}\left[(r^\pi(s) - \phi) + \sum_{s'\in \SSS}p^\pi(s' | s)\gamma(s')\right]\cdot\Pr[\Lambda(\bar{s}' | \pi, \bar{s}) > t, s_t = s]\nonumber\\
\geq & -(r_\text{max} + \phi)\sum_{s\in \SSS\setminus\{\bar{s}'\}}\Pr[\Lambda(\bar{s}' | \pi, \bar{s}) > t, s_t = s] \nonumber\\
&\qquad + \sum_{s' \in \SSS} \gamma(s') \sum_{s\in \SSS\setminus\{\bar{s}'\}}  \Pr[\Lambda(\bar{s}' | \pi, \bar{s}) > t, s_t= s]p^\pi(s' | s)\nonumber\\
=& -(r_\text{max} + \phi)\Pr\left[ \Lambda(\bar{s}' | \pi, \bar{s}) > t \right] + \gamma(\bar{s}')\sum_{s\in \SSS\setminus\{\bar{s}'\}}  \Pr[\Lambda(\bar{s}' | \pi, \bar{s}) > t, s_t = s]p^\pi(\bar{s}' | s)\nonumber\\
&\qquad + \sum_{s' \in \SSS\setminus\{\bar{s}'\}} \gamma(s') \sum_{s\in \SSS\setminus\{\bar{s}'\}}  \Pr[\Lambda(\bar{s}' | \pi, \bar{s}) > t, s_t= s]p^\pi(s' | s)\nonumber\\
= & -(r_\text{max} + \phi)\Pr\left[ \Lambda(\bar{s}' | \pi, \bar{s}) \geq t+1 \right] + \gamma(\bar{s}') \Pr\left[\Lambda(\bar{s}' | \pi, \bar{s}) = t+1\right] \nonumber\\
&\qquad + \sum_{s' \in \SSS\setminus\{\bar{s}'\}}\gamma(s') \Pr[\Lambda(\bar{s}' | \pi, \bar{s}) > t+1, S_{t+1} = s'].\label{eq:lemma_ind_claim_step_1}
\end{align}
By applying the bound in (\ref{eq:lemma_ind_claim_step_1}) to the last term in inequality (\ref{eq:lemma_markov_ind_claim}), the induction argument is established. Thus, the inequality (\ref{eq:lemma_markov_ind_claim}) is true for all $t\in \mathbb{N}$. 

Finally, we prove the required inequality in the Lemma by tending $t$ to the infinity, which gives
\begin{align}
\gamma(\bar{s}) &\geq \underset{t\rightarrow \infty}{\text{liminf}} \left\{ -(r_\text{max} + \phi)\sum^{t}_{\tau=1}\Pr[ \Lambda(\bar{s}' | \pi, \bar{s}) \geq \tau] + \gamma(\bar{s}') \sum^t_{\tau=1}\Pr[\Lambda(\bar{s}' | \pi, \bar{s}) = \tau]\right. \nonumber\\
&\qquad \qquad \qquad \qquad \qquad \qquad \qquad\left. + \sum_{s\in \SSS\setminus\{\bar{s}'\}}\gamma(s)\Pr[\Lambda(\bar{s}' | \pi, \bar{s}) > t, s_t = s]\right\} \label{eq:lemma_ind_claim_step_2} .
\end{align}
Thanks to Assumption \ref{ass:communicating}, the limit in fact exists, so we can replace the ``liminf'' by ``lim''. Indeed, we have $\sum^\infty_{\tau=1}\Pr[\Lambda(\bar{s}' | \pi, \bar{s}) = \tau] = \mathbf{E}[\Lambda(\bar{s}' | \pi, \bar{s})]\leq D$ by Assumption \ref{ass:communicating} and the definition of policy $\pi$. In addition, the Assumption implies that $\lim_{t\rightarrow \infty}\sum^t_{\tau=1}\Pr[\Lambda(\bar{s}' | \pi, \bar{s}) = \tau] = 1$, and equivalently $\Pr[\Lambda(\bar{s}' | \pi, \bar{s}) =\infty] = 0$. By putting these three limits into (\ref{eq:lemma_ind_claim_step_2}), we arrive at
\begin{equation*}
\gamma(\bar{s}) \geq -(r_\text{max} + \phi) D + \gamma(\bar{s}').
\end{equation*}
Finally, since the states $\bar{s}, \bar{s}'$ are arbitrary, the Lemma is proved. $\hfill \Box$

\subsection{Proof of Claim \ref{claim:small_eg}, and remarks on reachabiltiy assumptions}\label{app:pfclaimsmalleg}
In this Appendix section, we first prove Claim \ref{claim:small_eg}, and then elaborate on some profound implications of the underlying reachability assumption on the structure of the optimal policies. 
\claimsmalleg*
\begin{proof}
Consider the instance ${\cal M}_\text{c}$ under which instance ${\cal M}_\text{MG} = {\cal M}_\text{star}$ is a communicating star graph with $K \geq 2$ branches and diameter $D\geq 2$, where $D$ is even. The states $s^1, \ldots, s^K$ are the leaves nodes in the star graph, and it takes $D$ time steps to travel from a state $s^k$ to another $s^{k'} \neq s^k$. 
More precisely, the state space of ${\cal M}_\text{star}$ is $S_\text{star} = \{s_\text{c}\} \cup \bigcup^K_{k=1} \{s^k_1, s^k_2, \ldots, s^k_{D / 2 -1}\} \cup \bigcup^K_{k=1} \{s^k\}$. For convenient sake, we identify $s^k$ with $s^k_{D / 2}$ . Note that $s^k_{D / 2 - 1}$ is $s'^k$. 
For two distinct states $s, s'$, we say that there are two-way arcs between $s, s'$ if there is an arc from $s$ to $s'$ and an arc from $s'$ to $s'$. In instance ${\cal M}_\text{star}$, there are two way arcs between $s_\text{c}$, $s^k_1$ for each $1\leq k\leq K$, as well as between $s^k_d, s^k_{d+1}$ for each $1\leq k\leq K$, $1\leq d\leq D/2 - 1$. Lastly, as defined in Section \ref{sec:alg}, there is a self loop $ \CircArrowRight{}^k$ with each $s^k$.

For each $k$, we clearly have $V_{1:T, k} = \frac{1}{T}\sum^T_{t=1}\mathsf{1}( a_t =  \CircArrowRight{}^k )$. Consider the quantity $N_{\not \CircArrowRight{}} = \sum^T_{s=1}\mathsf{1}( a_t \not\in\{ \CircArrowRight{}^k\}^K_{k=1}) = T (1 - \sum^K_{k=1} V_{1:T, k})$. Clearly, $N_{\not \CircArrowRight{}} \geq N_\text{alt}$. 
If $N_\text{alt} \geq T^\alpha$, then we also have $N_{\not \CircArrowRight{}}\geq T^\alpha$. By triangle inequality, we have
$$
g(\bar{V}_{1:T}) \leq 1 - \frac{1}{2}\left|  \sum^K_{k=1} \left( V_{1:T, k} - \frac{1}{K}\right) \right| = 1 - \frac{N_{\not \CircArrowRight{}} }{2T}  \leq 1 - \frac{1}{2 \cdot T^{1-\alpha}}.
$$
Since $\text{opt}({\cal M}_{\text{MG}}) = 1$, we have $\text{Reg}(T) = \Omega(1 / T^{1-\alpha})$ as required. 

Now, suppose the otherwise, that is, $N_\text{alt} < T^\alpha$. If it is still the case that $N_{\not \CircArrowRight{}}\geq T^\alpha$, then we have $\text{Reg}(T) = \Omega(1 / T^{1-\alpha})$ as before. In the remaining proof, we assume that $N_\text{alt}, N_{\not \CircArrowRight{}} < T^\alpha$. In fact, these conditions imply a stronger bound $N_\text{alt} < T^\alpha / D$, since an alternation between distinct $ \CircArrowRight{}^k, \CircArrowRight{}^{k'}$ requires visiting $D$ arcs (which are equivalent to actions) incident to ${\cal M}_\text{c}$ but not in $\{  \CircArrowRight{}^k\}^K_{k=1}$. 

From now on, suppose $T$ is so large that $T / 2 > D\max\{T^\alpha, T^{1-\alpha}\}$. From time step $T/2$ to $T$, the agent visits $\{  \CircArrowRight{}^k \}^K_{k=1}$ at least $ T/2 - T^\alpha > 0$ times, but she only alternate among $\{  \CircArrowRight{}^k\}^K_{k=1}$ at most $T^\alpha / D$ times. By the pigeonhole principle, for some $k$ the agent visits $ \CircArrowRight{}^k $ for at least $ (T/2 - T^\alpha ) / (T^\alpha / D) = D(T^{1 - \alpha} / 2 - 1) $ consecutive time steps, let say starting from $\tau \in\{T/2 , \ldots, T - D(T^{1 - \alpha} / 2 - 1)\}$, which is non-empty since $T$ is sufficiently large. In particular, she continuously visits $ \CircArrowRight{}^k $ during time steps $\{\tau, \tau + 1, \ldots, \tau' = \tau + D \tau^{1 - \alpha}/2\}$ (Note that $\tau'\leq T$, since $D \tau^{1 - \alpha}/2 \leq \tau \leq T/2$).  We claim 
\begin{equation}\label{eq:claim_small_eg}
\max\left\{ \left| V_{1:\tau, k} - \frac{1}{K}  \right|,  \left| V_{1:\tau', k} - \frac{1}{K}  \right|  \right\}\geq\frac{D}{8\tau^\alpha}. 
\end{equation}
The claim inequality (\ref{eq:claim_small_eg}) immediately implies that $\max\{\text{Reg}(\tau), \text{Reg}(\tau')\} = \Omega(D / \tau^\alpha)$, which proves the second part of the Claim since $\tau' = \Theta(\tau)$. To prove (\ref{eq:claim_small_eg}), first denote $V_{1:\tau, k} = \frac{\bar{N}}{\tau}$. Then we know that $V_{1:\tau', k} = \frac{\bar{N}+ D \tau^{1 - \alpha}/2 }{\tau+ D \tau^{1 - \alpha}/2 }$. Consequently, 
\begin{align}
\left| V_{1:\tau, k} - \frac{1}{K}  \right| +  \left| V_{1:\tau', k} - \frac{1}{K}  \right| \geq V_{1:\tau', k} - V_{1:\tau, k} = \frac{\bar{N}+ D \tau^{1 - \alpha}/2 }{\tau+ D \tau^{1 - \alpha}/2 } - \frac{\bar{N}}{\tau} \geq \frac{D}{4\tau^\alpha}\nonumber,
\end{align}
hence (\ref{eq:claim_small_eg}) is proved and the Claim is established.
\end{proof}

\textbf{Unichain vs Communicating MDPs. } The underlying reachability assumption has a profound impact on the difficulty of MDPwGR. \citep{Altman99} shows that certain mutli-objective MDP can be asymptotically optimized by a randomized stationary policy under the unichain assumption.  This result is inapplicable for ${\cal M}_{\text{MG}}$, which only satisfies the communicating assumption but not the unichain assumption. In fact, MDPwGR under the communicating assumption should be optimized by a non-stationary policy:
\claimcontsmalleg*
From the reachability assumption perspective, MDPwGR is different from MDPwSR, where a deterministic stationary policy achieves $O(D / T)$ anytime regret under the unichain or the communicating assumption. In fact, it is also different various bandit settings with global rewards and global constraints \citep{BadanidiyuruKS13, AgrawalD14, AgrawalDL16, AgrawalD16}. In each of these settings, an $O(1/\sqrt{T})$ anytime regret (hiding the dependence on other model parameters) can be achieved by a randomized stationary policy.

\begin{proof}[Proof of Claim \ref{claim:cont_small_eg}]
Consider the star instance ${\cal M}_{\text{star}}$ defined in the proof of Claim \ref{claim:small_eg}, with $K\geq 2$ and $D\geq 2$.  
A stationary policy, be it deterministic or randomized, induces a time homogeneous Markov chain $(S_\text{star}, \mathfrak{p})$ on ${\cal M}_{\text{star}}$, where $\mathfrak{p}(s' | s)$ is the probability of transiting from $s\in S_\text{star}$ to $s'\in S_\text{star}$ in the Markov chain. Every state transition occurs along an arc in ${\cal M}_{\text{star}}$. 

We prove the claim by inspecting $(S_\text{star}, \mathfrak{p})$ under different cases on $\mathfrak{p}$. If $\mathfrak{p}(s^k | s^k) = 0$ for some $k\in \{1, 2\}$, then clearly $V_{1:t, k} =0$ for every $t$, leading to $\text{Reg}(T) = \Omega(1)$ for all $T$. Else, suppose that $\mathfrak{p}(s^k_1 | s_\text{c}) = 0$ or $\mathfrak{p}(s^k_{d} | s^k_{d-1}) = 0$ for some $1\leq d\leq D / 2$ and some $k$. This means that the agent cannot reach $s^k$ from $s_\text{c}$, which still leads to $\text{Reg}(T) = \Omega(1)$ for all $T$. Else, suppose that $\mathfrak{p}(s_\text{c} | s^k_1 ) = 0$ or $\mathfrak{p}(s^k_{d-1} | s^k_{d}) = 0$ for some $1\leq d\leq D / 2$ and some $k$. This means that the agent cannot leave the branch containing state $s^k$ once she reaches $s^k$. Therefore, under the stationary policy, she either never visits $s^k$, or she does visit $s^k$, but she will not be able to visit  $\{s^{k'}\}_{k' = k}$ forever. This means that $\text{Reg}(T) = \Omega(1)$ for sufficiently large $T$.

The remaining case is when $\mathfrak{p}(s | s') >0$ for all arcs from $s$ to $s'$ in ${\cal M}_{\text{star}}$. In this case, all states in $S_\text{star}$ forms a single recurrent class. By either the Perron-Frobenius Theorem or Theorem 1.7.5 in \citep{Norris98}, the stationary distribution $\{\lambda_s\}_{s\in \SSS_{\text{star}}}$ is entry-wise positive, and in particular $\lambda_{s_\text{c}} > 0$. This implies that $\lim_{T\rightarrow \infty }\mathbb{E}[\sum^T_{t=1}\mathsf{1}(a_t =  a_\text{c})] / T > 0$ for some $a_\text{c} \in \AAA_{s_\text{c}}$, and further implying that $\lim_{T \rightarrow \infty} \mathbb{E} [ \sum^K_{k=1} |  \frac{1}{T}\sum^T_{t=1}  \mathsf{1}(a_t =  \CircArrowRight{}^k ) - \frac{1}{K} | ] < 1$. These time averages exist since the Markov chain $(S_\text{star}, \mathfrak{p})$ is recurrent and aperiodic. As a result, we have $\lim_{T\rightarrow \infty } g(\bar{V}_{1:T}) < 1 = \text{opt}(\primal_{\cal M})$, which means that $\text{Reg}(T) = \Omega(1)$ for sufficiently large $T$.
\end{proof}

\subsection{Supplement to the Applications of MDPwGR}\label{app:applications_model}
We provide supplementary details on the discussions on the fairness objectives and MDPwK.
\subsubsection{Supplements to the discussions on Fairness Objectives}\label{app:fairness}
We demonstrate that function $g_{(\kappa)}$ is $\kappa$-Lipschitz w.r.t. $\|\cdot\| = \|\cdot\|_\infty$. Indeed, for any $u, w\in [0, 1]^K$, we know that
\begin{align}
g_{(\kappa)}(w) - g_{(\kappa)}(u) &= \sum^\kappa_{k=1}w_{\pi_w(k)} - \sum^\kappa_{k=1}u_{\pi_u(k)} = \sum^\kappa_{k=1} [ w_{\pi_u(k)} - u_{\pi_u(k)} ] + \sum^\kappa_{k=1} [w_{\pi_w(k)} -  w_{\pi_u (k)}]\nonumber\\
&\leq \sum^\kappa_{k=1} [ w_{\pi_u(k)} - u_{\pi_u(k)} ]\label{eq:moo_by_def_of_pi_u} \leq \kappa \|w - u\|_\infty.
\end{align}
Step (\ref{eq:moo_by_def_of_pi_u}) is by the definition of $\pi_w$, which implies $\sum^\kappa_{k=1}w_{\pi_w(k)} = \min_{I\subseteq \{1\ldots, K\} : |I| = \kappa}  \sum_{k\in I}w_k$. By the symmetry between $u, w$, the claim is established.

\subsubsection{Supplements to the discussions on MDPwK}\label{app:supp_MDPwK}
\textbf{Algorithm.} We solve the problem by implementing {\sc Toc-UCRL2} with the Tuned Mirror Descent oracle, namely oracle $\TMD(F_\infty, T)$ on a surrogate function $g_\text{C}$, see Algorithm \ref{alg:MDPwK}. Algorithm \ref{alg:MDPwK} involves interacting with Algorithm {\sc Toc-UCRL2}. Essentially, at the start of time $t$ in Algorithm \ref{alg:MDPwK}, it first run {\sc Toc-UCRL2}, which is also at the start of time $t$, until {\sc Toc-UCRL2} recommends action $a_t$ (Line \ref{alg:oco-ucrl2-action} in Algorithm \ref{alg:oco-ucrl2}). It is instructive to note that, for {\sc Toc-UCRL2}, time 1 start from Line \ref{alg:oco-ucrl2-tau}. For $t > 1$, {\sc Toc-UCRL2} starts from checking the \textbf{while} condition in Line \ref{alg:oco-ucrl2-while} in Algorithm \ref{alg:oco-ucrl2}. If the \textbf{while} condition is violated, {\sc Toc-UCRL2} runs Lines \ref{alg:oco-ucrl2-tau}--\ref{alg:oco-ucrl2-initial} , only then it runs Line \ref{alg:oco-ucrl2-action} and recommends an action. Otherwise, {\sc Toc-UCRL2} jumps straight to Line \ref{alg:oco-ucrl2-action}. After {\sc Toc-UCRL2} recommends action $a_t$, {\sc Toc-UCRL2} completes running for time step $t$ by running Lines \ref{alg:oco-ucrl2-receive}--\ref{alg:oco-ucrl2-end}
\begin{algorithm}[t]
\caption{Solving MDPwK by {\sc Toc-UCRL2}}\label{alg:MDPwK}
\begin{algorithmic}[1]
\State Initialize resource level $I_k = bT$ for each $k\in \{1, \ldots, K - 1\}$, $\tau = 0$.\;
\State Denote $\mathsf{A}$ as Algorithm {\sc Toc-UCRL2} with OCO oracle $\TMD(F_\infty, T)$ and $Q = (1 + 2/b)$ and objective function $g_\text{C}$.  \;
\Comment We continuously seek action recommendation from $\mathsf{A}$ and update $\mathsf{A}$.\;
\State \textbf{for} $t = 1\ldots, T$ \textbf{do}
\State \hspace{0.5cm} \textbf{if} $I_k \geq 0$ for all $k\in \{1, \ldots, K - 1\}$ \textbf{do}
\State \hspace{0.5cm} \hspace{0.5cm} Run $\mathsf{A}$, which is at the start of time $t$, until $\mathsf{A}$ recommends an action $a_t$. 
\State \hspace{0.5cm} \hspace{0.5cm} Perform action $a_t$. Observe $V_t(s_t, a_t) = (R_t(s_t, a_t), C_t(s_t, a_t))$ and the next state $s_{t+1}$.
\State \hspace{0.5cm} \hspace{0.5cm} Feedback $V_t(s_t, a_t), s_{t+1}$ to $\mathsf{A}$, and run $\mathsf{A}$ until its end of its time $t$.
\State \hspace{0.5cm} \hspace{0.5cm} Update $I_k \leftarrow I_k - C_{t, k}(s_t, a_t)$ for each $k\in \{1, \ldots, K - 1\}. $ \Comment{Update inventory levels}
\State \hspace{0.5cm} \hspace{0.5cm} Update $\tau \leftarrow \tau + 1$.
\State \hspace{0.5cm} \textbf{else}
\State \hspace{0.5cm} \hspace{0.5cm} Perform null action $a_0$.
\end{algorithmic}
\end{algorithm}

The mirror map $F_\infty$, defined in (\ref{eq:entropy}) in Appendix \ref{app:tmd_background}, is designed for Lipschitz continuous function w.r.t.  norm $\|\cdot\|_\infty$. The surrogate function $g_{\text{C}}$ takes input $(\bar{r}, \bar{c})\in [0, 1] \times [0, 1]^{K-1}$, and is defined as
$$g_{\text{C}}(\bar{r}, \bar{c} ) := \bar{r} - \frac{2}{b}\max_{k\in \{1, \ldots,  K  -1\} } \left\{ ( \bar{c}_k - b )^+  \right\},$$
where $w^+ := \max\{0, w\}$. The function $g_{\text{C}}$ penalizes the agent when any of the $K$ resource constraints is violated. The surrogate function $g_\text{C}$ is close in spirit to \citep{AgrawalD14,AgrawalDL16}. 

\textbf{Analysis. } We consider the following benchmark for MDPwK:
\begin{subequations}
\begin{alignat}{2}
(\primal_{\text{C}}(b))\text{: \;} & \max_x ~ \sum_{s\in \SSS, a\in \AAA_s} r(s, a) x(s, a) & \nonumber\\
\text{s.t. } & \sum_{a\in \AAA_s}c_k(s, a) x(s,a) \leq  b &\quad &\forall k\in \{1, \ldots, K\} \nonumber\\
				&\sum_{a \in \AAA_s} x(s, a) = \sum_{s'\in\SSS, a'\in \AAA_{s'}} p(s | s',a') x(s', a')&\quad &\forall s \in \SSS \nonumber\\
               &\sum_{s\in \SSS, a\in \AAA_s}x(s, a) = 1       &\quad & \nonumber\\
               &x(s, a)\geq 0      &\quad &\forall s\in \SSS, a\in \AAA_s \nonumber
\end{alignat}
\end{subequations}
\begin{claim}\label{claim:MDPwSR_constr}
Consider an instance of MDPwK, where the underlying MDP is communicating with diameter $D$.  
For any $b\geq 0$, the linear program $(\primal_{\text{C}}(b))$ has a non-empty feasible region. In addition, (i) for any non-anticipatory policy that satisfies the resource constraints, it holds that $$\mathbb{E}\left[\sum^\tau_{t=1}R_t(s_t, a_t) \right] \leq T\cdot \text{opt}(\primal_{\text{C}}(b)) + 2D .$$ (ii) For any $\eta > 0$, we have $$\frac{b}{b + \eta} \text{opt}(\primal_{\text{C}}(b + \eta)) \leq  \text{opt}(\primal_{\text{C}}(b)).$$
\end{claim}
We postpone the proof of the Claim to the end of the Appendix. Now, we define the regret of an algorithm as
$$
\text{Reg}_{MDPwK}(T) = T \cdot \text{opt}(\primal_{\text{C}}(b))  -\sum^\tau_{t=1}R_t(s_t, a_t),
$$
where we recall that $\tau$ is the stopping time. 
We claim that Algorithm \ref{alg:MDPwK} is near optimal.
In the remaining, we prove Claim \ref{claim:MDPwSR_constr} and Proposition \ref{prop:MDPwK}.
 
\begin{proof}[Proof of Claim \ref{claim:MDPwSR_constr}]
To see that the feasible region of $(\primal_\text{C}(b))$ is non-empty, we consider the Markov chain induced by taking the null action in each state. Since $S$ is finite, there is at least one recurrent class in the Markov chain. Consequently, there exists a stationary distribution $\{\lambda_0(s)\}_{s\in \SSS}$ on the Markov chain. Now, define $x_0$ as: $x_0(s, a) = 0$ if $a\neq a_0$, and $x_0(s, a) = \lambda_0(s)$ otherwise. Clearly, $x_0$ is feasible, thus the non-emptiness is established.

Part (i) of the Claim follows from the same argument as the proof of Theorem \ref{thm:benchmark}, hence we omit the proof. For part (ii), suppose that $x^*_{b+\eta}$ be an optimal solution to $(\primal_\text{C}(b+\eta))$. Now, the solution  $ \frac{b}{b+\eta}  x^*_{b+\eta} + \frac{\eta}{b+\eta} x_0$ is feasible to $(\primal_\text{C}(b))$, and has objective value $\frac{b}{b+\eta}\text{opt}(\primal_\text{C}(b+\eta))$. Hence part (ii) is established.
\end{proof}

\begin{proof}[Proof of Proposition \ref{prop:MDPwK}]
Note that the function $g_{\text{C}}$ is $(1 + 2/b)$-Lipschitz w.r.t. $\|\cdot\|_\infty$ on $(\bar{r}, \bar{v})$, with the property that $\|\theta\|_1 = 1 + 2/b$ for all $\theta\in \partial g_\text{C}(\bar{r}, \bar{c})$ and for all $(\bar{r}, \bar{c})\in [0, 1]\times [0, 1]^{K-1}$. Now, denote the offline benchmark of the MDPwGR problem with reward function $g_\text{C}$ as $\text{opt}(\primal_\text{S-C})$. 
Now, let $\{V^\text{free}_t(s_t, a_t)\}^T_{t=1}$ be generated by just running $\mathsf{A}$ itself for $T$ steps. Clearly, we know that $\text{opt}(\primal_\text{S-C}) \geq \text{opt}(\primal_\text{C}(b))$. For brevity sake, let's denote 
$$
\eta^\text{free} := \max_{k\in \{1, \ldots,  K  -1\} } \left\{ \left( \sum^T_{t=1} C^\text{free}_{t, k}(s_t, a_t) - b T \right)^+  \right\} \Big / T . 
$$
By applying Theorem \ref{thm:TMD}, we know that with probability $1 - O(\delta)$, 
\begin{align}
 T\cdot \text{opt}(\primal_\text{C}(b)) - \sum^T_{t=1} R^\text{free}_t(s_t, a_t) + \frac{2}{b}  \cdot T \eta^\text{free} & \leq  \tilde{O}\left(D \cdot T^{2/3}\right) +  \tilde{O}\left(D \sqrt{\Gamma S A T} \right), \label{eq:MDPwSR_intermediate_1} 
\end{align}
where the $\tilde{O}(\cdot)$ notation hides logarithmic dependence on $K, S, A, T, 1/\delta$. 
Now, $\{V^\text{free}_t(s_t, a_t)\}^\tau_{t=1}$, $\{V_t(s_t, a_t)\}^\tau_{t=1}$ are identically distributed (The latter output by Algorithm \ref{alg:MDPwK}). 
To proceed, we show that the violation $T \eta^\text{free}$ in inventory constraints is bounded. Now, observe that
\begin{align}
\mathbb{E}\left[\sum^T_{t=1} R^\text{free}_t(s_t, a_t)\right] & \leq T\cdot \mathbb{E}\left[\text{opt}(\primal_{\text{C}}(b + \eta^\text{free})) \right] + 2 D \label{eq:MDPwK_by_claim_i}\\
& \leq T \cdot \text{opt}(\primal_{\text{C}}(b + \mathbb{E} [ \eta^\text{free}] ))  + 2 D\label{eq:MDPwK_by_claim_ii}\\
& \leq T\cdot \frac{b + \mathbb{E}\left[ \eta^\text{free}\right] }{b}\text{opt}(\primal_{\text{C}}(b))  + 2 D \label{eq:MDPwK_by_claim_ii_again}\\
& \leq T\cdot \text{opt}(\primal_{\text{C}}(b))  + T\cdot \frac{ \mathbb{E}\left[ \eta^\text{free}\right] }{b}  + 2 D\label{eq:MDPwK_by_trivail}.
\end{align}
Step (\ref{eq:MDPwK_by_claim_i}) is by (i) in Claim \ref{claim:MDPwSR_constr}. Step (\ref{eq:MDPwK_by_claim_ii}) is by the concavity of $\text{opt}(\primal_\text{C}(b))$ as a function of $b$, which is implied by (ii) in Claim \ref{claim:MDPwSR_constr}. Step (\ref{eq:MDPwK_by_claim_ii_again}) is by a direct application of (ii) in Claim \ref{claim:MDPwSR_constr}. Step (\ref{eq:MDPwK_by_trivail}) is evidently by $\text{opt}(\primal_\text{C}(b)) \leq 1$. 
Apply inequality (\ref{eq:MDPwK_by_trivail}) to the bound (\ref{eq:MDPwSR_intermediate_1}) provides 
\begin{align}
T\cdot \mathbb{E} [ \eta^\text{free} ] / b = \tilde{O}\left(D \cdot T^{2/3}\right) +  \tilde{O}\left(D \sqrt{\Gamma S A T} \right). 
\end{align}
By the Hoeffding inequality, we know that 
\begin{equation}\label{eq:MDPwSR_intermediate_2}
\Pr \left[ T\eta^\text{free} / b \leq \tilde{O}\left(D \cdot T^{2/3}\right) +  \tilde{O}\left(D \sqrt{\Gamma S A T} \right) \right] \geq  1 - O(\delta) .
\end{equation}
This implies that 
\begin{equation*}
\Pr\left[\text{Reg}_{\text{MDPwK}}(T) = \tilde{O}\left(K D \cdot T^{2/3}\right) +  \tilde{O}\left(K D \sqrt{\Gamma S A T} \right) \right] \geq 1 - O(\delta),
\end{equation*}
since for each unit of violation, the agent at most earns $K$ units of rewards. 
Hence, the regret bound follows.
\end{proof}

\section{Analysis of {\sc Toc-UCRL2}}
In this Appendix section, we first provide further details about {\sc Toc-UCRL2}, and then provide proofs for the results used in establishing the convergence of {\sc Toc-UCRL2}. In Section \ref{app:evi}, we state the EVI oracle $\textsf{EVI}$ by \citep{JakschOA10}. In Section \ref{app:pfclaimegQ0}, we demonstrate that running {\sc Toc-UCRL2} with $Q=0$ leads to  $\text{Reg}(T) = \Omega(1)$ for sufficiently large $T$. In Section \ref{app:aux_res}, we provide several auxiliary propositions from the literature for our analysis. In Section \ref{app:pflemmaconc}, we demonstrate that events ${\cal E}^v, {\cal E}^p$ hold with high probability. In Section \ref{app:pflemmadecompose}, we prove the decomposition lemma, Lemma \ref{lemma:decompose}. Finally, in Sections \ref{app:pfclubheart} -- \ref{app:pflemmaspade}, we establish bounds for the five error terms $(\clubsuit, \diamondsuit, \heartsuit, \spadesuit, \P)$. These bounds are conditional on $M(T)$, a deterministic upper bound on the number of episodes in $T$ time steps. Altogether, given an $M(T)$ for an OCO oracle, the framework provides a regret bound for {\sc Toc-UCRL2} under the OCO oracle. 
\subsection{EVI Oracle \citep{JakschOA10}}\label{app:evi}
\begin{algorithm}[t]
\caption{{\sc EVI}$(\tilde{r}, H^p; \epsilon)$, mostly extracted from \citep{JakschOA10}}\label{alg:evi}
\begin{algorithmic}[1]
\State Initialize VI record $u_0\in \mathbb{R}^{\SSS}$ as $u_0(s) = 0$ for all $s\in \SSS$.
\State \textbf{for} $i = 0, 1, \ldots$ \textbf{do}
\State \hspace{0.5cm} For each $s\in \SSS$, compute VI record $$u_{i+1}(s) = \max_{a\in \AAA_s}\tilde{\Upsilon}_i(s, a)\text{, where } \tilde{\Upsilon}_i(s, a) = \tilde{r}(s, a) + \max_{\bar{p}\in H^p(s, a)}\left\{\sum_{s'\in \SSS} u_i(s')\bar{p}(s')\right\}.$$
\State \hspace{0.5cm} \textbf{if} $\max_{s\in \SSS}\left\{ u_{i+1}(s) - u_i(s) \right\} - \min_{s\in \SSS}\left\{ u_{i+1}(s) - u_i(s) \right\} \leq \epsilon$ \textbf{do}\label{alg:evi_termination}
\State \hspace{0.5cm} \hspace{0.5cm} Break the \textbf{for} loop.
\State Define stationary policy $\tilde{\pi}:\SSS \rightarrow \AAA_s$ as $\tilde{\pi}(s) = \text{argmax}_{a\in \AAA_s}\tilde{\Upsilon}_i(s, a).$
\State Define an optimistic dual solution $\tilde{\phi} = \max_{s\in \SSS}\left\{ u_{i+1}(s) - u_i(s) \right\}$, $\tilde{\gamma} = u_i$.
\State Return policy $\tilde{\pi}$ and dual variables $(\tilde{\phi}, \tilde{\gamma})$.
\end{algorithmic}
\end{algorithm}
We present an EVI oracle from \citep{JakschOA10} in Algorithm \ref{alg:evi}. In the input, $\tilde{r}$ is an optimistic estimate of a certain latent scalar reward $r$ (which is $\{(-\theta_{\tau(m)})^\top v(s, a)\}_{s, a}$ in our setting), $H^p$ is a confidence region that contains the latent transition kernel $p$, and $\epsilon \in (0, 1)$ is an error parameter. The oracle is essentially a Value Iteration (VI) algorithm on an extended space of transition kernels.

\subsection{Proof of Claim \ref{claim:eg_Q0}: {\sc Toc-UCRL2} incurs $\text{Reg}(T) = \Omega(1)$ when $Q = 0$}\label{app:pfclaimegQ0}
In this subsection, we investigate the behavior of {\sc Toc-UCRL2} when $Q = 0$. 
We focus on $\FW$ bound, and the same conclusion holds for  $\TGD$. 

\begin{claim}\label{claim:eg_Q0}
There is an MDPwGR instance with a Lipschitz continuous and smooth $g$ w.r.t. $\|\cdot\|_2$,  such that {\sc Toc-UCRL2} with OCO oracle $\FW$ and $Q=0$ incurs $\text{Reg}(T) = \Omega(1)$ for all $T$. 
\end{claim}
\begin{proof}
Before diving into the proof, we first observe that, by setting $Q = 0$, the \textbf{while} condition in Line \ref{alg:oco-ucrl2-while} becomes vacuous, and each episode consists of only a single time step. Equivalently, the notions of episodes and time steps coincide in this case. {\sc Toc-UCRL2} behaves similarly to the dual based algorithm by \citep{AgrawalD14}, where the linearized reward $\{(-\theta_t)^\top \tilde{v}_t(s, a)\}_{s, a}$ is used at time $t$ to compute a corresponding policy $\tilde{\pi}_t$. Thus, $\tilde{\pi}_t$ in general varies across time steps.

Let's revisit the example ${\cal M}_\text{star}$ again, with the reward function $g(w) =   \sum^K_{k=1} w_k - \frac{1}{2}(w_k - \frac{1}{K})^2$, with $K\geq 2$ and $D \geq 2$. Clearly, the function $g$ is $\sqrt{K}$-Lipschitz and $1$-smooth w.r.t. $\|\cdot\|_2$, and we still have $\text{opt}(\primal_{\cal M}) = 1$. 
We argue that, when $T$ is sufficiently large, the agent visit each $ \CircArrowRight{}^k$ once, and then transit to another $ \CircArrowRight{}^{k'}$ with $k' \neq k$. Consequently, we have $N_{\not \CircArrowRight{}}  \geq (D - 1) T / D$ for sufficiently large $T$, resulting in a $\Omega(1)$ regret by following the argument for the first ``if'' case in Claim \ref{claim:small_eg}.

Without loss of generality, assume that $T$ is so large that $\rad^v_{m, k}(s, a)< 1/(1000 K)$ for all $s\in \SSS_\text{star}, a\in \AAA_s$ , and that $\rad^p_m(s' | s, a) < 1/(1000 S_\text{star})$ for each $s, s'\in \SSS_\text{star}, a\in \AAA_s$. Indeed, if the agent fails to visit each state-action pairs for say $10^8 \max\{K, S_\text{star}\}$ times for all finite $T$, then clearly we have $\text{Reg}(T) = \Omega(1)$ for sufficiently large $T$, by virtue of Claim \ref{claim:small_eg}. Now, these confidence radii are so small that the agent is ``almost'' certain about the instance ${\cal M}_\text{star}$:
\begin{itemize}
\item For $s\in \SSS_\text{star}, a\in \AAA_s$ with $v(s, a) = \textbf{0}_K$, we have $\theta^\top \bar{v}(s, a) \in [-0.001, 0.001] $ for all $\theta\in B(\sqrt{K}, \|\cdot\|_2)$ and all $\bar{v}(s, a)\in H^v_T$.
\item For each $(s^k, \CircArrowRight{}^k)$, we have $v_k(s^k, \CircArrowRight{}^k) \in [0.999, 1]$, and $v_{k'}(s^k, \CircArrowRight{}^k) \in [0, 0.001]$ for all $k'\neq k$.
\item For $s, s'\in \SSS_\text{star}, a\in \AAA_s$ where $p(s' | s, a) = 1$, we have $\bar{p}(s' | s, a) \in [0.999, 1]$ for all $\bar{p} \in H^p_T(s, a)$.
\item For $s, s'\in \SSS_\text{star}, a\in \AAA_s$ where $p(s' | s, a) = 0$, we have $\bar{p}(s' | s, a) \in [0, 0.001]$ for all $\bar{p} \in H^p_T(s, a)$.
\end{itemize} 
These items ensure that the agent knows that $\{(s^k, \CircArrowRight{}^k)\}^K_{k=1}$ are precisely the state-action pairs with non-null feedback, and $\CircArrowRight{}^k$ contributes to the outcome in dimension $k$. In addition, when the agent decides that to follow the stationary policy that ``walks to state $s^k$, and then self-loops at $\CircArrowRight{}^k$ indefinitely'', the agent would know that she needs to travel to $s_\text{c}$ in the correct path first, and then follow the correct branch to reach $s^k$.

Recall that $V_{1:T, k} = \frac{1}{T}\sum^T_{s=1}\mathsf{1}(a_t  =  \CircArrowRight{}^k)$, and note that $\nabla g(w) = (1 + 1/K) \mathbf{1}_K - w$. Consequently, the scalarized reward $(-\theta_t)^\top \tilde{v}_t(s, a) = ((1 + 1/K) \mathbf{1}_K - V_{1:t})^\top \tilde{v}_t(s, a)$ is maximized at the $(s^k, \CircArrowRight{}^k)$ for which the corresponding $V_{1:T, k}$ is the smallest. Altogether, the algorithm always ensures that $\max_k\{\sum^T_{t=1}\mathsf{1}(a_t = \CircArrowRight{}^k)\} - \min_k\{\sum^T_{t=1}\mathsf{1}(a_t = \CircArrowRight{}^k)\} \leq 1$, leading to the claimed dynamics, and proving the regret lower bound.
\end{proof}

\subsection{Auxiliary results for analyzing {\sc Toc-UCRL2}}\label{app:aux_res}
First, we state two Theorems on concentration inequalities. Theorem \ref{thm:emp_berstein} is useful for analyzing events ${\cal E}^v, {\cal E}^p$, and Theorem \ref{thm:hoeffding} is useful for analyzing the dynamics of the online process.
\begin{theorem}[\cite{AudibertMS09}]\label{thm:emp_berstein}
Let random variables $Y_1, \ldots, Y_N\in [0, 1]$ be independently and identically distributed. Consider their sample mean $\hat{Y}_N$ and their sample variance $\hat{\sigma}^2_{Y, N}$: 
    $$\hat{Y}_N = \frac{1}{N}\sum^{N}_{i=1}Y_i\text{, }\quad \hat{\sigma}^2_{Y, N} = \frac{1}{N}\sum^{N}_{i=1}(Y_i - \hat{Y})^2.$$  For any $\delta \in (0, 1)$, the following inequality holds:
        \[
           \qquad \qquad \qquad  \Pr\left(\left|\hat{Y}_N - \mathbb{E}[Y_1] \right| \leq \sqrt{\frac{2\hat{\sigma}^2_{Y, N}\log(1/\delta)}{N}} + \frac{3\log(1/\delta)}{N}\right)\geq 1 - 3\delta.  \qquad \qquad \qquad \text{$\blacksquare$}
        \]
        
\end{theorem}
\begin{theorem}[\cite{Hoeffding63}]\label{thm:hoeffding}
Let random variables $X_1, \ldots, X_T$ constitute a martingale difference sequence w.r.t. a filtration $\{\FFF_t\}^T_{t=1}$, that is, $\mathbb{E}[X_t | \FFF_{t-1}] = 0$ for all $1\leq t\leq T$. Also, suppose that $|X_t | \leq B$ almost surely for all $t$. Then the following inequality holds for any $0<\delta < 1$:
\[
 \qquad \qquad \qquad \qquad \quad \quad \; \ \Pr\left[\frac{1}{T}\sum^T_{t=1} X_t \leq B \sqrt{\frac{2\log(1/\delta)}{T}} \right]\geq 1 - \delta. \quad \qquad \qquad \qquad \qquad \quad \; \ \text{$\blacksquare$}
\]
\end{theorem}
Next, we present auxiliary results, mostly from \citep{JakschOA10}. Theorem \ref{thm:JOA} is useful for analyzing the EVI oracle $\textsf{EVI}$. Lemmas \ref{lemma:JOA2}, \ref{lemma:JOA} and Claim \ref{claim:inverse} are useful for proving the convergence of {\sc Toc-UCRL2}.
\begin{theorem}[\cite{JakschOA10}]\label{thm:JOA}
Consider applying $\textsf{EVI}$ (Algorithm \ref{alg:evi}) with input $(\tilde{r}, H^p; \epsilon)$, where the underlying transition kernel $p$ of lies in $H^p$, and the underlying instance is communicating with diameter $D$. Then (i)  $\textsf{EVI}(\tilde{r}, H^p; \epsilon)$ terminates in finite time, (ii) the output dual variables $\tilde{\gamma}$ satisfies $\max_{s\in \SSS} \tilde{\gamma}_s - \min_{s\in \SSS}\tilde{\gamma}_s \leq D \cdot \max_{s, \in \SSS, a\in \AAA} |\tilde{r}(s, a) |.$ \hfill  $\blacksquare$
\end{theorem}

\begin{lemma}[Lemma 19 in \citep{JakschOA10}]\label{lemma:JOA2}
For any sequence of numbers $z_1 , \ldots, z_n$ with $0 \leq z_m \leq Z_{m-1} := \max\{1, \sum^{m-1}_{i = 1}z_i\}$, we have
\[
 \qquad \qquad \qquad \qquad \qquad \qquad \; \ \sum^n_{m=1}\frac{z_m}{\sqrt{Z_{m-1}}} \leq \left(\sqrt{2} + 1\right) \sqrt{Z_n}. \qquad \qquad \qquad \qquad \qquad \qquad \; \ \text{$\blacksquare$}
\]
\end{lemma}

\begin{lemma}[\cite{JakschOA10}]\label{lemma:JOA}
The following inequality holds with certainty:
\[
 \qquad \qquad \qquad \qquad \quad \qquad \;  \sum^T_{t=1} \frac{1}{\sqrt{N^+_{m(t)}(s_t, a_t)}} \leq \left(\sqrt{2} + 1\right)\sqrt{SA T}.\quad \qquad \qquad \qquad \qquad \quad \;  \text{$\blacksquare$}
\]
\end{lemma}
\begin{claim}\label{claim:inverse}
The following inequality holds with certainty:
$$
\sum^T_{t=1} \frac{1}{N^+_{m(t)}(s_t, a_t)} \leq SA \left(1 + 2\log T\right).
$$
\end{claim}

\begin{proof}[Proof of Claim \ref{claim:inverse}]
To start the proof, first denote $\nu'_{m(T)}(s, a) = \sum^T_{t = \tau(m(T))} \mathsf{1}((s_t, a_t) = (s, a))$. Essentially $\nu'_{m(T)}(s, a)$ is $\nu_{m(T)}(s, a)$ capped at the end of time step $T$. In addition, denote $N^{+'}_{m(T) + 1}(s, a) = \sum^T_{t = 1} \mathsf{1}((s_t, a_t) = (s, a))$. Similar to $\nu'_{m(T)}(s, a)$, $N^{+'}_{m(T) + 1}(s, a)$ denotes the version of $N^{+}_{m(T) + 1}(s, a)$ capped at the end of time step $T$. Now, 
\begin{equation*}
\sum^T_{t=1} \frac{1}{N^+_{m(t)}(s_t, a_t)} = \sum^{m(T) - 1}_{m = 1} \sum_{s\in \SSS}\sum_{a\in \AAA_s}\frac{\nu_m(s, a)}{N^+_m(s, a)} + \sum_{s\in \SSS}\sum_{a\in \AAA_s}\frac{\nu'_{m(T)}(s, a)}{N^+_{m(T)}(s, a)}.
\end{equation*}
Now, for every state-action pair $s, a$, we assert that
\begin{equation}\label{eq:inverse_each_s_a}
\sum^{m(T) - 1}_{m = 1} \frac{\nu_m(s, a)}{ N^+_m(s, a)} +  \frac{\nu'_{m(T)}(s, a)}{ N^+_{m(T)}(s, a)} \leq 1 + 2 \log \left( N^{+'}_{m(T) + 1}(s, a)\right).
\end{equation}
Indeed, the asserted inequality can be proved by drawing the following general fact: For any sequence of numbers $z_1, \ldots, z_n$ with $0\leq z_m \leq Z_{m-1} := \max\{1, \sum^{m-1}_{i=1} z_i \}$, we have
\begin{equation}\label{eq:general_fact}
\sum^n_{m=1}\frac{z_m}{Z_{m-1}}\leq 1 + 2\log Z_n.
\end{equation}
We prove the inequality (\ref{eq:general_fact}) by induction on $n$. The case for $n = 1$ is clearly true. Now, suppose the inequality is true for $n$. Then it is also true for $n+1$, since
$$
\sum^{n+1}_{m=1}\frac{z_m}{Z_{m-1}} \leq 1 + 2\log Z_{n} + \frac{z_{n+1}}{Z_{n}} \leq  1 + 2\log Z_{n} + 2\log \left(1 + \frac{z_{n+1}}{Z_{n}}\right) = 1 + 2\log Z_{n+1},
$$
where we use the fact that $x\leq 2\log (1 + x)$ for $x\in [0, 1]$. Hence, the induction is established and the (\ref{eq:general_fact}) is proved for general $n$.

Given (\ref{eq:general_fact}) for general $n$, we can readily establish (\ref{eq:inverse_each_s_a}) by applying (\ref{eq:general_fact}) with $n = m(T)$, $z_m = \nu_m(s, a)$ for $1\leq m\leq n-1$ and $z_n = \nu'_n(s, a)$. Altogether, noting that $N^{+'}_{m(T) + 1}(s, a) \leq T$, we achieve the required inequality.
\end{proof}

\subsection{Proof of Lemma \ref{lemma:conc} for events ${\cal E}^v, {\cal E}^p$}\label{app:pflemmaconc}
\lemmaconc*
The proof of the Lemma uses Theorem \ref{thm:emp_berstein} by \citep{AudibertMS09}.

\begin{proof}[Proof of Lemma \ref{lemma:conc}]
We first analyze event ${\cal E}^v$. Consider a fixed objective index $k$, a fixed state $s$ and a fixed action $a$. We assert that 
\begin{equation}\label{eq:required_v}
\mathbb{P}\left[\left|\hat{v}_{m, k}(s, a) - v_k(s, a)\right| \leq \rad^v_{m, k}(s, a) \text{ for all $m$}\right]\geq 1 - \frac{\delta}{2 K SA }. 
\end{equation}
Assuming inequality (\ref{eq:required_v}), the bound $\Pr[{\cal E}^v]\geq 1 - \delta/2$ is established by taking a union bound over $s\in \SSS$, $a\in \AAA_s$ and $k\in [K]$.

We establish inequality (\ref{eq:required_v}) by applying Theorem \ref{thm:emp_berstein} and the union bound. First, note that $\hat{v}_{m, k}(s, a)$ is the sample mean of $N_m(s, a)$ i.i.d. random variables, which are distributed as $V_k(s, a)$. Let $\hat{\sigma}^2_{v, m, k}$ be the sample variance of these $N_m(s, a)$ i.i.d random variables. To apply the union bounds, we also consider $\Upsilon^V_1, \ldots, \Upsilon^V_T$, which are $T$ i.i.d samples with the same distribution as $V_k(s, a)$. Denote $\hat{\Upsilon}^V_t$, $\hat{\sigma}^2_{\Upsilon^V, t}$ respectively as the sample mean and variance of $\Upsilon^V_1, \ldots, \Upsilon^V_t$. 
Let $\delta^v(t) = \delta / (12K  SA t^2) $. Now, 
\begin{align}
&\mathbb{P}\left[\left|\hat{v}_{m, k}( s, a) - v_k( s, a)\right| \leq \sqrt{\frac{2\hat{\sigma}^2_{v, m, k}\log(1/\delta^v(N^+_m(s, a)) )}{N^+_m(s, a)}} + \frac{3\log(1/\delta^v(N^+_m(s, a)) )}{N^+_m(s, a)} \forall\text{ $m$} \right]\nonumber\\
\geq & \mathbb{P}\left[\left|\hat{\Upsilon}^V_t - v_k( s, a)\right| \leq \sqrt{\frac{2\hat{\sigma}^2_{\Upsilon^V, t}\log(1/\delta^v(t) )}{t}} + \frac{3\log(1/\delta^v(t) )}{t} \text{ for all $t\in [T]$} \right]\label{eq:v_by_union_bound}\\
\geq & 1- 3 \sum^T_{t=1}\delta^v(t) = 1 - \frac{\delta}{4KSA} \sum^T_{t=1}\frac{1}{t^2} \geq   1 - \frac{\delta}{2KSA}.\label{eq:v_by_conc}
\end{align}
Step (\ref{eq:v_by_union_bound}) is by applying a union bound over all possible values of $N^+_m(s, a)$s. Step (\ref{eq:v_by_conc}) is by applying Theorem \ref{thm:emp_berstein}. 
Finally, note that $\hat{\sigma}^2_{v, m, k} \leq \hat{v}_{m, k}(s, a)$, since $V(s_t, a_t)\in [0, 1]$. Putting in the definition of $\delta^v(t)$ yields 
$$
\rad^v_{m, k}(s, a) \geq \sqrt{\frac{2\hat{\sigma}^2_{v, m, k}\log(1/\delta^v(N^+_m(s, a)) )}{N^+_m(s, a)}} + \frac{3\log(1/\delta^v(N^+_m(s, a)) )}{N^+_m(s, a)}.
$$
Altogether, the required inequality for ${\cal E}^v$ is shown.

Next, we analyze the event ${\cal E}^p$ by in a similar way. Consider fixed states $s', s\in \SSS$ and a fixed action $a\in \AAA_s$. We assert that 
\begin{equation}\label{eq:required_p}
\mathbb{P}\left[\left|\hat{p}_m(s' | s, a) - p(s' | s, a)\right| \leq \rad^p_m(s' | s, a)  \text{ for all $m$} \right]\geq 1 - \frac{\delta}{2 S^2 A}.
\end{equation}
Assuming inequality (\ref{eq:required_p}), the bound $\Pr[{\cal E}^p]\geq 1 - \delta/2$ is established by taking a union bound over $s', s\in \SSS$ and $a\in \AAA_s$. 
Let $\Upsilon^p_1, \ldots, \Upsilon^p_T$ be $T$ i.i.d. Bernoulli random variables with the common mean $p(s' | s, a)$. For each $t\in [T]$, denote $\hat{\Upsilon}^p_t, \hat{\sigma}^2_{\Upsilon^p, t}$ respectively as the sample mean and sample variance of $\Upsilon^p_1, \ldots \Upsilon^p_t$.
In addition, let $\delta^p(t) = \delta/ (12S^2 A t^2)$. We have
\begin{align*}
& \mathbb{P}\left[\left|\hat{p}_m(s' | s, a) - p(s' | s, a)\right| \leq \rad^p_m(s' | s, a)  \text{ for all $m$} \right] \\
\geq & \mathbb{P}\left[\left|\hat{\Upsilon}^p_t - p(s' | s, a)\right| \leq \sqrt{\frac{2\hat{\sigma}^2_{\Upsilon^p, t}\log(1/\delta^p(t) )}{t}} + \frac{3\log(1/\delta^p(t) )}{t} \text{ for all $t\in [T]$} \right]\\
\geq & 1- 3 \sum^T_{t=1}\delta^p(t) = 1 - \frac{\delta}{4S^2 A} \sum^T_{t=1}\frac{1}{t^2} \geq   1 - \frac{\delta}{2S^2 A}.
\end{align*}
Hence, the Lemma is proved.
\end{proof}

\subsection{Proof of Lemma \ref{lemma:decompose}, which decomposes the regret}\label{app:pflemmadecompose}
\lemmadecompose*
\begin{proof}[Proof of Lemma \ref{lemma:decompose}]
First, by the definitions of $(\clubsuit_t), (\diamondsuit_t)$, it is clear that 
$$
(-\theta_t)^\top V_t(s_t, a_t) = \tilde{r}_m(s_t, a_t) - \left[ (\clubsuit_t) + (\diamondsuit_t) \right].
$$
Thus, it suffices to show that 
\begin{equation}\label{eq:decompose_required_ineq}
\tilde{r}_m(s_t, a_t) \geq (-\theta_t)^\top \sum_{s\in \SSS, a\in \AAA_s} v(s, a)x^*(s, a) -\left[ (\heartsuit_t) + (\spadesuit_t) + (\P_t)\right].
\end{equation}
To prove (\ref{eq:decompose_required_ineq}), we first focus on the application of the EVI oracle for episode $m$. By  Assumption \ref{ass:communicating} and by assuming the event ${\cal E}^p$, we know that the oracle terminates in finite time, by virtue of item (i) in Theorem \ref{thm:JOA}. Thus, the output policy $\tilde{\pi}_m$ and the output dual variables $(\tilde{\phi}_m, \tilde{\gamma}_m)$ from Eq. (\ref{eq:evi_m}) are well-defined. Now, we assert that 
\begin{equation}\label{eq:decompose_step_1}
\tilde{r}_m(s_t, a_t)\geq \tilde{\phi}_m -\left[(\spadesuit_t) + (\P_t)\right].
\end{equation}
To show (\ref{eq:decompose_step_1}), we let $\tilde{u}_{\iota +1}, \tilde{u}_{\iota}\in \mathbb{R}^\SSS$ respectively be the terminating and the penultimate VI records, when $\text{EVI}(\tilde{r}_m, H^m_p, 1/\sqrt{\tau(m)})$ is applied. Now, we have
\begin{align}
\tilde{\phi}_m - (\P_t) &= \max_{s\in \SSS}\left\{\tilde{u}_{\iota + 1}(s) - \tilde{u}_{\iota}(s)\right\}- \frac{1}{\sqrt{\tau(m)}} \label{eq:decompose_step_2}\\
&\leq \min_{s\in \SSS}\left\{\tilde{u}_{\iota + 1}(s) - \tilde{u}_{\iota}(s)\right\}\label{eq:decompose_step_3}\\
& \leq \tilde{u}_{\iota + 1}(s_t) - \tilde{u}_{\iota}(s_t)\nonumber\\
& = \max_{a\in \AAA_{s_t}}\left\{ \tilde{r}_m(s_t, a) + \max_{\bar{p}\in H^p_m(s_t, a)}\left\{\sum_{s'\in \SSS} \tilde{u}_{\iota}(s')\bar{p}(s')\right\}\right\} - \tilde{u}_{\iota}(s_t)\nonumber\\
& = \tilde{r}_m(s_t, a_t) + \max_{\bar{p}\in H^p_m(s_t, a_t)}\left\{\sum_{s'\in \SSS} \tilde{u}_{\iota}(s')\bar{p}(s')\right\} - \tilde{u}_{\iota}(s_t) \label{eq:decompose_step_4}\\
&= \tilde{r}_m(s_t, a_t) + (\spadesuit_t) \label{eq:decompose_step_5},
\end{align} 
where step (\ref{eq:decompose_step_2}) is by the definition of $\tilde{\phi}_m$, step (\ref{eq:decompose_step_3}) is by the terminating condition of EVI, and step (\ref{eq:decompose_step_4}) is by the definition of $\tilde{\pi}_m$, and step (\ref{eq:decompose_step_5}) is by the definition of $\tilde{\gamma}_m$.

In order to prove the inequality (\ref{eq:decompose_required_ineq}) and complete the proof of the Lemma, it suffices to show
\begin{equation}\label{eq:decompose_step_6}
\tilde{\phi}_m \geq (-\theta_t)^\top \sum_{s\in \SSS, a\in \AAA_s} v(s, a)x^*(s, a) - (\heartsuit_t).
\end{equation}
To this end, we first claim that the output dual variables $(\tilde{\phi}_m, \tilde{\gamma}_m)$ are feasible to the following linear program $(\lineardual_m)$:
\begin{subequations}
\begin{alignat}{2}
(\lineardual_m) \text{:\; min }    & \phi & \nonumber\\
\text{s.t. }   &\phi + \gamma(s) \geq \tilde{r}_m(s, a) + \sum_{s'\in\SSS} p(s' | s,a) \gamma(s')      &\quad &\forall s \in \SSS, a\in \AAA_s \nonumber\\
               &\phi, \gamma(s)\text{ free}      &\quad & \forall s\in \SSS. \nonumber
\end{alignat}
\end{subequations}
Indeed, for any $s\in \SSS$, $a\in \AAA_s$, we have
\begin{align}
 \tilde{\phi}_m + \tilde{\gamma}_m(s) & \geq  \tilde{u}_{\iota + 1}(s) - \tilde{u}_\iota(s) + \tilde{u}_\iota(s) = \tilde{u}_{\iota + 1}(s)\nonumber\\
& \geq \tilde{r}_m(s, a) + \max_{\bar{p}\in H^p_m(s, a)}\left\{\sum_{s'\in \SSS} \tilde{u}_{\iota}(s')\bar{p}(s')\right\} \label{eq:decompose_step_7}\\
& \geq \tilde{r}_m(s, a) + \sum_{s'\in \SSS} \tilde{u}_{\iota}(s')p(s' | s, a) = \tilde{r}_m(s, a) + \sum_{s'\in \SSS} \tilde{\gamma}_m(s')p(s' | s, a)\nonumber,
\end{align}
where step (\ref{eq:decompose_step_7}) is by the assumption that $p \in H^p_m$, since we condition on the event ${\cal E}^p$. Therefore, we have $\tilde{\phi}_m \geq \text{opt}(\lineardual_m) = \text{opt}(\linearprimal_m)$, where the linear program 
\begin{subequations}
\begin{alignat}{2}
(\linearprimal_m) \text{:\; max }    & \sum_{s\in \SSS, a\in \AAA_s}\tilde{r}_m(s, a) x(s, a) & \nonumber\\
\text{s.t. }   &\sum_{a \in \AAA_s} x(s, a) = \sum_{s'\in\SSS, a'\in \AAA_{s'}} P(s | s',a')x(s', a')      &\quad &\forall s \in \SSS \nonumber\\
               &\sum_{s\in \SSS, a\in \AAA_s}x(s, a) = 1       &\quad & \nonumber\\
               &x(s, a)\geq 0      &\quad &\forall s\in \SSS, a\in \AAA_s \nonumber
\end{alignat}
\end{subequations}
is a dual formulation of $(\lineardual_m)$. The optimal solution $x^*$ of the offline benchmark problem $(\primal_{\cal M})$ is feasible to the problem $(\linearprimal_m)$, since both $(\primal_{\cal M})$, $(\linearprimal_m)$ have the same feasible region. 

Finally, we prove the inequality (\ref{eq:decompose_step_6}), and hence completing the proof of the Lemma. In the following derivation, we denote $\tilde{v}_m(s, a)$ as an optimal solution to the optimization problem (\ref{eq:opt_reward}) for computing the optimistic reward $\tilde{r}_m(s,a)$:
\begin{align}
\tilde{\phi}_m &\geq \sum_{s\in \SSS, a\in \AAA_s} \tilde{r}_m(s, a) x^*(s, a) \nonumber\\
&= (-\theta_{\tau(m)})^\top \sum_{s\in \SSS, a\in \AAA_s} \tilde{v}_m(s, a) x^*(s, a)\nonumber\\
&= \sum_{s\in \SSS, a\in \AAA_s} x^*(s, a)\left[(-\theta_{\tau(m)})^\top  \tilde{v}_m(s, a) - (-\theta_{\tau(m)})^\top v(s, a)\right] \nonumber\\
&\qquad + \left[ \theta_{t} -\theta_{\tau (m)}  \right]^\top \sum_{s\in \SSS, a\in\AAA_s} v(s, a) x^*(s, a) + (-\theta_{t})^\top \sum_{s\in \SSS, a\in\AAA_s} v(s, a) x^*(s, a) \label{eq:decompose_step_8}\\
& \geq - (\heartsuit_t) + (-\theta_{t})^\top \sum_{s\in \SSS, a\in\AAA_s} v(s, a) x^*(s, a)\nonumber,
\end{align}
where step (\ref{eq:decompose_step_8}) holds, since we condition on the event ${\cal E}^v$, which ensures that 
$(-\theta_{\tau(m)})^\top  \tilde{v}_m(s, a) - (-\theta_{\tau(m)})^\top v(s, a) \geq 0$ for each $s\in \SSS, a\in \AAA_s$. Therefore, the first sum in (\ref{eq:decompose_step_8}) is non-negative, hence the step is justified. Altogether, inequality (\ref{eq:decompose_step_6}) is shown, and the Lemma is proved.
\end{proof}

\subsection{Proof of Lemma \ref{lemma:club_heart}, which bounds $(\clubsuit,\heartsuit)$}\label{app:pfclubheart}
\lemmaclubheart*
\begin{proof}[Proof of Lemma \ref{lemma:club_heart}]
Now,
\begin{align}
\sum^T_{t=1} (\clubsuit_t) = & \sum^{m(T)-1}_{m = 1} \sum^{\tau(m+1) - 1}_{t = \tau(m)} \left[ \theta_t -\theta_{\tau(m)}\right]^\top V_t(s_t, a_t) + \sum^{T}_{t = \tau(m(T))} \left[ \theta_t - \theta_{\tau(m)}\right]^\top V_t(s_t, a_t)\nonumber\\
\leq & \sum^{m(T)-1}_{m = 1} \sum^{\tau(m+1) - 1}_{t = \tau(m)} \left\| \theta_t - \theta_{\tau(m)}\right\|_* \| V_t(s_t, a_t)\| + \sum^{T}_{t = \tau(m(T))} \left\| \theta_t -  \theta_{\tau(m)} \right\|_* \| V_t(s_t, a_t)\|\label{eq:fw_cs}\\
\leq & \sum^{m(T)-1}_{m = 1} Q \max_{t\in \{1, \ldots, T\}} \|V_t(s_t, a_t)\|   +  Q \max_{t\in \{1, \ldots, T\}} \|V_t(s_t, a_t)\|  \quad \text{ w.p. 1} \label{eq:fw_R_club_bd}\\
= & ~ Q \max_{t\in \{1, \ldots, T\}} \|V_t(s_t, a_t)\|  \cdot M(T) ~ \leq ~ Q \|\mathbf{1}_K\| \cdot M(T).  \label{eq:club_refined}
\end{align}
Step (\ref{eq:fw_cs}) is by the triangle inqeuality and the Cauchy-Scwhartz inequality. 
Step (\ref{eq:fw_R_club_bd}) is by our terminating criteria, which require $\Psi \leq Q$ for each episode. 
Similar to the above, we also have: 
\begin{align}
&\sum^T_{t=1} (\heartsuit_t) =  \left\{\sum^{m(T)-1}_{m = 1} \sum^{\tau(m+1) - 1}_{t = \tau(m)}\left[ \theta_{\tau(m)} - \theta_t \right] + \sum^{T}_{t = \tau(m(T))} \left[ \theta_{\tau(m)} - \theta_t  \right]\right\}^\top \sum_{s\in \SSS, a\in\AAA_s} v(s, a) x^*(s, a)   \nonumber\\
\leq & \left\{\sum^{m(T)-1}_{m = 1} \sum^{\tau(m+1) - 1}_{t = \tau(m)}\left\| \theta_{\tau(m)} - \theta_t  \right\|_* + \sum^{T}_{t = \tau(m(T))} \left\| \theta_{\tau(m)} - \theta_t  \right\|_* \right\} \left\| \sum_{s\in \SSS, a\in\AAA_s} v(s, a) x^*(s, a) \right\| \nonumber\\
\leq & \sum^{m(T)-1}_{m = 1} Q \max_{s\in \SSS, a\in \AAA_s} \|v(s, a)\|   +  Q \max_{s\in \SSS, a\in \AAA_s} \|v(s, a)\|  \quad \text{ w.p. 1} \nonumber\\
= & ~ Q \max_{s\in\SSS, a\in \AAA_s} \|v(s, a)\|  \cdot M(T) ~ \leq ~ Q \|\mathbf{1}_K\| \cdot M(T).  \label{eq:heart_refined}
\end{align}
Altogether, the Lemma is proved.
\end{proof}

\subsection{Proof of Claim \ref{claim:P}, which bounds $(\P)$}\label{app:pfclaimP}
\claimp*
\begin{proof}[Proof of Claim \ref{claim:P}]
The proof uses Lemma \ref{lemma:JOA2}. Let's apply $n = m(T)$, as well as 
\begin{equation*}
z_m = 
  \begin{cases} 
   \tau(m + 1) - \tau(m) & \text{if } 1\leq m < m(T) \\
   T - \tau(m)       & \text{if } m = m(T)
  \end{cases},
\end{equation*}
where we set $\tau(0) = 0$. 
Now, $Z_0 = 1$, $Z_m = \tau(m)$ for $1\leq m < m(T)$, and $Z_{m(T)} = T$. Therefore, 
\begin{align*}
\sum^T_{t=1}(\P_t) & = \sum^{m(T) -1}_{m  =1}\sum^{\tau(m+ 1) - 1}_{t = \tau(m)}\frac{1}{\sqrt{\tau(m)}} + \sum^T_{t = \tau(m(T))}\frac{1}{\sqrt{\tau(m(T))}}\nonumber\\
& = \sum^n_{m=1}\frac{z_m}{\sqrt{Z_{m-1}}} \leq \left(\sqrt{2} + 1\right) \sqrt{Z_{m(T)}} = \left(\sqrt{2} + 1\right) \sqrt{T}. 
\end{align*}
Hence the claim is proved.
\end{proof}

\subsection{Proof of Lemma \ref{lemma:diamond}, which bounds $(\diamondsuit)$}\label{app:pflemmadiamond}
\lemmadiamond*
The proof of the Lemma uses the Azuma-Hoeffding inequality in Theorem \ref{thm:hoeffding}, as well as Lemma \ref{lemma:JOA} by \citep{JakschOA10} and Claim \ref{claim:inverse}.

\begin{proof}[Proof of Lemma \ref{lemma:diamond}]
To start the proof, we define $\tilde{v}_m(s, a)$ and $m(t)$. We express $\tilde{r}_m(s, a) = (-\theta_{\tau(m)})^\top \tilde{v}_m(s, a)$, where $\tilde{v}_m(s, a)$ is an optimal solution to the optimization problem (\ref{eq:opt_reward}). For each $t$, we define $m(t)$ to be the episode index such that $\tau(m(t)) \leq t < \tau(m(t) + 1) - 1$. 
We first decompose $\sum^T_{t=1} (\diamondsuit_t) $ as follows:
\begin{align*}
\sum^T_{t=1} (\diamondsuit_t) & \leq \sum^T_{t=1}\tilde{r}_{m(t)}(s_t, a_t) - (-\theta_{\tau(m(t))})^\top V_t(s_t, a_t) \nonumber\\
& = \underbrace{\sum^T_{t=1}(-\theta_{\tau(m(t))})^\top \left[ \tilde{v}_{m(t)}(s_t, a_t) - v(s_t, a_t)\right]}_{(\dagger_v)}  + \underbrace{\sum^T_{t=1} (-\theta_{\tau(m(t))})^\top\left[v(s_t, a_t) -  V_t(s_t, a_t)\right]}_{(\ddagger_v)}.
\end{align*}
We bound the sums $(\dagger_v, \ddagger_v)$ as follows:

\textbf{Bounding $(\dagger_v)$. }We bound this term by invoking the confidence bounds asserted by the event ${\cal E}^v$. Define the notation $(\text{log-$v$}) := \log (12KSAT^2 / \delta )$. We have
\begin{align}
(\dagger_v) &= \sum^T_{t=1}(-\theta_{\tau(m(t))})^\top \left[ \tilde{v}_{m(t)}(s_t, a_t) - \hat{v}_{m(t)}(s_t, a_t) +  \hat{v}_{m(t)}(s_t, a_t) - v(s_t, a_t)\right] \nonumber\\
&\leq \sum^T_{t=1} \left\|-\theta_{\tau(m(t))}\right\|_* \left[\left\|\tilde{v}_{m(t)}(s_t, a_t) - \hat{v}_{m(t)}(s_t, a_t) \right\| + \left\|\hat{v}_{m(t)}(s_t, a_t) - v(s_t, a_t)\right\| \right]\label{eq:dagger_v_cs}\\
&\leq 2 L \sum^T_{t=1} \left\|(\rad^v_{m(t), k}(s_t, a_t ))^K_{k=1}\right\| \label{eq:dagger_v_Ev}\\
& \leq 4 L \left\| \mathbf{1}_K\right\| \left[\sqrt{(\text{log-$v$})}\cdot \sum^T_{t=1} \frac{1}{\sqrt{N^+_{m(t)}(s_t, a_t)}} + 3\cdot (\text{log-$v$})\cdot \sum^T_{t=1}\frac{1}{N^+_{m(t)}(s_t, a_t)}\right]\label{eq:diamond_step_1}\\
&\leq 4 L \left\| \mathbf{1}_K\right\| \left[\left(\sqrt{2} + 1\right)\sqrt{SA T \cdot (\text{log-$v$})} +   3\cdot (\text{log-$v$})\cdot SA \left(1 + 2\log T\right)\right]
\nonumber.
\end{align}
Step (\ref{eq:dagger_v_cs}) is by the Cauchy-Schwartz inequality, step (\ref{eq:dagger_v_Ev}) is by the assumption that the event ${\cal E}^v$ holds. Step (\ref{eq:diamond_step_1}) is by Lemma \ref{lemma:JOA} and Claim \ref{claim:inverse}, as well as $(\text{log-$v$})\geq (\text{log-$v$})_m$ for all $m$.

\textbf{Bounding $(\ddagger_v)$. } Consider random variable $X_t = (-\theta_{\tau(m(t))})^\top\left[v(s_t, a_t) -  V_t(s_t, a_t)\right]$ and filtration ${\cal F}_t = \sigma(\{ s_t, a_t, V_t(s_t, a_t), \theta_{t+1} \}^t_{\tau = 1}).$ 
Now, $|X_t| \leq L\|\mathbf{1}_K\|$, $X_t$ is ${\cal F}_t$-measurable with $\mathbf{E}[X_t | {\cal F}_{t-1} ]=0$. Thus, we apply Theorem \ref{thm:hoeffding} to conclude that, with probability $ \geq 1-\delta$,
$$
(\ddagger_v) \leq L\|\mathbf{1}_K\| \sqrt{2T \log(1/\delta)}.
$$

Altogether, we have, with probability at least $1-\delta$, 
\begin{align}
\sum^T_{t=1} (\diamondsuit_t) &\leq L \left\| \mathbf{1}_K\right\| \left[\left(5 \sqrt{2} + 4\right)\sqrt{\sum_{s\in \SSS}|\AAA_s| T \cdot(\text{log-$v$}) } +   12\cdot (\text{log-$v$}) \cdot \sum_{s\in \SSS} |\AAA_s| \log T \right] \label{eq:diamond_precise_bound}\\
& = O\left(L \left\| \mathbf{1}_K\right\| \sqrt{ SAT \log \frac{KSAT}{\delta}}  + L \left\| \mathbf{1}_K\right\| SA \log^2\frac{KSAT}{ \delta} \right). 
\nonumber
\end{align}
Hence, the Lemma is proved.
\end{proof}

\subsection{Proof of Lemma \ref{lemma:spade}, which bounds $(\spadesuit)$}\label{app:pflemmaspade}
\lemmaspade*
\begin{proof}[Proof of Lemma \ref{lemma:spade}]
First, observe that $$\max_{s\in \SSS}\{\tilde{\gamma}_m(s)\} - \min_{s\in \SSS}\{ \tilde{\gamma}_m(s)\}\leq L \max_{s, a}\max_{w\in H^v_m(s, a) } \|w\| \cdot D \leq L \|\mathbf{1}_K\| \cdot D$$ for all $m$. 
Indeed, conditioned on ${\cal E}^p$, we have $p \in H^p_m$ for all $m$. In addition, $\max_{s, a}\tilde{r}_m(s, a) \leq L \max_{s, a}\max_{w\in H^v_m(s, a) } \|w\| \leq L\|\mathbf{1}_K\|$. The desired inequality follows from item (ii) in Theorem \ref{thm:JOA}.    
For each episode $m$ and state $s$, consider replacing $\tilde{\gamma}_m(s)$ by $\gamma_m(s) := \tilde{\gamma}_m(s) - \min_{s'\in \SSS}\{ \tilde{\gamma}_m(s')\}$. Now, $0 \leq \max_{m, s}\{\gamma_m(s)\}\leq L\|\mathbf{1}_K\| \cdot D$, and the value of each $(\spadesuit_t)$ is preserved: 
\begin{align*}
(\spadesuit_t) &= \max_{\bar{p}\in H^p_{m(t)}(s_t, a_t)}\left\{\sum_{s'\in \SSS}\tilde{\gamma}_{m(t)}(s') \bar{p}(s')\right\} - \tilde{\gamma}_{m(t)}(s_t) \nonumber\\
&=  \max_{\bar{p}\in H^p_{m(t)}(s_t, a_t)}\left\{\sum_{s'\in \SSS} \gamma_{m(t)}(s') \bar{p}(s')\right\} - \gamma_{m(t)}(s_t) ,
\end{align*}
where $m(t)$ is the episode index such that $\tau(m(t)) \leq t <\tau(m(t)+1)$. 

Consider the following decomposition:
\begin{align*}
\sum^T_{t=1} (\spadesuit_t) &= \underbrace{\sum^T_{t=1} \left[ \max_{\bar{p}\in H^p_m(s_t, a_t)}\left\{\sum_{s'\in \SSS} \gamma_{m(t)}(s') \bar{p}(s')\right\} - \sum_{s\in \SSS}\gamma_{m(t)}(s)p(s | s_t, a_t) \right]}_{(\dagger_p)}  \nonumber\\
&\qquad + \underbrace{\sum^T_{t=1} \left[\sum_{s\in \SSS}\gamma_{m(t)}(s)p(s | s_t, a_t) - \gamma_{m(t)}(s_t)\right]}_{(\ddagger_p)}.
\end{align*}

\textbf{Bounding $(\dagger_p)$.} We proceed by unraveling $H^p_m$. Now, denote $(\text{log-$p$}) := \log (12S^2AT^2 / \delta )$. 
\begin{align}
(\ddagger_p) &\leq \sum^T_{t=1} \left[ \max_{\bar{p}\in H^p_{m(t)}(s_t, a_t)}\left\{\sum_{s\in \SSS} \gamma_{m(t)}(s) \bar{p}(s)\right\} - \min_{\bar{p}\in H^p_{m(t)}(s_t, a_t)}\left\{\sum_{s\in \SSS} \gamma_{m(t)}(s) \bar{p}(s)\right\} \right] \nonumber\\
&\leq 2 \sum^T_{t=1}\sum_{s\in \SSS} \gamma_{m(t)}(s)\rad^p_{m(t)}(s | s_t, a_t) \nonumber\\
&\leq 2 L\|\mathbf{1}_K\| \cdot D \sum^T_{t=1}\sum_{s\in \SSS}  \left[\sqrt{\frac{2\hat{p}_{m(t)}(s' | s, a) \cdot (\text{log-$p$}) }{N^+_{m(t)}(s, a)}} + \frac{3(\text{log-$p$})}{N^+_{m(t)}(s, a)}\right]\nonumber\\
&\leq  2 L\|\mathbf{1}_K\| \cdot D \sum^T_{t=1}\left[\sqrt{\frac{2\Gamma\cdot (\text{log-$p$}) }{N^+_{m(t)}(s, a)}} + \frac{3\cdot S(\text{log-$p$})}{N^+_{m(t)}(s, a)}\right]\label{eq:spade_step_1}\\
&\leq 2 (\sqrt{2} + 1) L \|\mathbf{1}_K\|\cdot  D\sqrt{2\Gamma SAT\cdot (\text{log-$p$})} + 6L \|\mathbf{1}_K\| \cdot D S^2 A (1 + 2\log T)(\text{log-$p$}) \label{eq:spade_step_2}.
\end{align}
We justify step (\ref{eq:spade_step_1}) as follows. Now, recall $\Gamma = \max_{s\in \SSS ,a\in \AAA_s} \| p(\cdot | s, a) \|_0$. With certainty, we have $\| \hat{p}_m(\cdot | s, a) \|_0 \leq \| p(\cdot | s, a) \|_0 \leq \Gamma$. Indeed, for each $s'\in \SSS$, $p(s' | s, a) = 0$ implies that $\hat{p}_m(s' | s, a) = 0$ with certainty. By the Cauchy-Schwartz inequality,
\begin{align}
\sum_{s'\in \SSS} \sqrt{\hat{p}_m(s' | s, a)} & = \sum_{s'\in \SSS} \sqrt{\hat{p}_m(s' | s, a)\cdot\mathsf{1}(p(s' | s, a) > 0)}\nonumber\\
\leq & \sqrt{\left[\sum_{s'\in \SSS} \hat{p}_m(s' | s, a)\right]\left[\sum_{s'\in \SSS}\mathsf{1}(p(s' | s, a) > 0)\right]}= \sqrt{ \| p(\cdot | s, a) \|_0} = \sqrt{\Gamma}.\nonumber
\end{align}
Step (\ref{eq:spade_step_2}) is by Proposition \ref{lemma:JOA} and Claim \ref{claim:inverse}.

\textbf{Bounding $(\ddagger_p)$.} We analyze the term by accounting for the number of episodes:
\begin{align}
(\ddagger_p) &= \left[\gamma_{m(T+1)}(s_{T+1})  - \gamma_{m(1)}(s_1) \right] + \sum^T_{t=1} \left[\sum_{s\in \SSS}\gamma_{m(t)}(s)p(s | s_t, a_t) - \gamma_{m(t+1)}(s_{t+1})\right] \nonumber\\
& = \left[\gamma_{m(T+1)}(s_{T+1})  - \gamma_{m(1)}(s_1) \right] + \sum^T_{t=1} \left[\gamma_{m(t)}(s_{t+1}) - \gamma_{m(t+1)}(s_{t+1})\right] \nonumber\\
& \qquad + \sum^T_{t=1} \left[\sum_{s\in \SSS}\gamma_{m(t)}(s)p(s | s_t, a_t) - \gamma_{m(t)}(s_{t+1}) \right] \quad \text{ w.p. $1$} \label{eq:spade_step_2.5} \\
& \leq \max_{t, s}\{\gamma_{m(t)}(s)\}  (M(T) + 1)  +\sum^T_{t=1} \left[\sum_{s\in \SSS}\gamma_{m(t)}(s)p(s | s_t, a_t) - \gamma_{m(t)}(s_{t+1}) \right] \label{eq:spade_step_3} \\
& \leq \max_{t, s}\{\gamma_{m(t)}(s)\} (M(T) + 1)  + \max_{t, s}\{\gamma_{m(t)}(s)\} \sqrt{2T\log(1/\delta)}\nonumber \quad \text{ w.p. $1-\delta$}\\
& \leq L \|\mathbf{1}_K\|\cdot  D (M(T) + 1)  + L \|\mathbf{1}_K\|\cdot D\sqrt{2T\log(1/\delta)}\nonumber .
\end{align}
Step (\ref{eq:spade_step_3}) is shown by analyzing the second summation in (\ref{eq:spade_step_2.5}), which is $\sum^T_{t=1} \gamma_{m(t)}(s_{t+1}) - \gamma_{m(t+1)}(s_{t+1})$. In the summation, at most $m(T) \leq M(T)$ summands are non-zero, and each non-zero summand is less than or equal to $\max_{t, s}\{\gamma_{m(t)}(s)\} \leq L \|\mathbf{1}_K\|\cdot  D$. 

Combining the bounds for $(\dagger_p, \ddagger_p)$, with probability at least $1-\delta$ we have 
\begin{align}
& \sum^T_{t=1} (\spadesuit_t) \leq  (2\sqrt{2} + 3) L  \|\mathbf{1}_K\| \cdot D\sqrt{2\Gamma SAT\cdot (\text{log-$p$})} + L  \|\mathbf{1}_K\| \cdot D(M(T) + 1)\label{eq:explicit_bd_spade}\\
&\qquad\qquad \qquad  + 6L \|\mathbf{1}_K\| \cdot D S^2 A (1 + 2\log T)\cdot(\text{log-$p$})\nonumber\\
 =  & O \left(L \|\mathbf{1}_K\| D\cdot M(T)  \right) + O\left(L \|\mathbf{1}_K\| D\sqrt{\Gamma SAT\log\frac{SAT}{\delta}} +  L \|\mathbf{1}_K\| DS^2 A \log^2\frac{SAT}{\delta}\right)\nonumber.
\end{align}
Altogether, the Lemma is proved.
\end{proof}

\section{OCO Oracles and Their Analyses}
In this Appendix section, we analyze the performance of {\sc Toc-UCRL2} under several specific OCO oracles, which are designed for different classes of reward functions. En route, we provide the full $O(\cdot)$ expressions for the regret bounds under those OCO oracles. In Appendix \ref{app:fw}, we provide supplementary details on the analysis of {\sc Toc-UCRL2} with OCO oracle $\FW$. In Appendix \ref{app:tgd}, we analyze {\sc Toc-UCRL2} with OCO oracle $\TGD$. In Appendix \ref{app:tmd}, we provide an anytime implementation of {\sc Toc-UCRL2} with OCO oracle $\TMD$, as well as its analysis.

To start, we provide the full version of Proposition \ref{prop:Err_bd} in $O(\cdot)$ notation, by summarizing the full $O(\cdot)$ bounds for $(\clubsuit, \diamondsuit, \heartsuit, \spadesuit, \P)$:
\begin{proposition}\label{prop:Err_bd_full}[Prop. \ref{prop:Err_bd} in full]
Consider an execution of {\sc Toc-UCRL2} with a general OCO oracle, over a communicating MDPwGR instance with diameter $D$. For each $T \in \mathbb{N}$, suppose that there is a deterministic constant $M(T)$ s.t. $\Pr[m(T)\leq M(T)] = 1$. Conditioned on events ${\cal E}^v, {\cal E}^p$, with probability at least $1-O(\delta)$
\begin{align*}
\sum^T_{t=1} (-\theta_t)^\top  [ v^* - V_t(s_t, a_t) ] &= O\left( (L D +  Q) \|\mathbf{1}_K\|  M^\FW(T) + L \| \mathbf{1}_K \| D\sqrt{\Gamma SAT\log\frac{KSAT}{\delta}} \right.\nonumber\\
& \qquad \qquad \qquad \qquad \qquad \qquad \qquad \left. + L \| \mathbf{1}_K \| DS^2A\log^2\frac{KS AT}{\delta} \right). \text{$\blacksquare$}
\end{align*}
\end{proposition}
In addition, we provide an auxiliary claim that is useful for establishing an upper bound $M(T)$ on $m(T)$ under an OCO oracle.
\begin{claim}\label{claim:aux_oco}
Let $\alpha \in [0, 1)$ and $C > 0$. Suppose the sequence $\{\rho_j\}^\infty_{j=1}$ satisfies $\rho_1 \geq C^{\frac{1}{1-\alpha}}$, and $\rho_{j+1} \geq \rho_j + C \rho_j^\alpha$. Then for all $j\geq 1$ we have 
\begin{equation}\label{eq:aux_oco}
\rho_j \geq 2^{-\frac{\alpha}{(1-\alpha)^2}} \left[ C \cdot (1-\alpha )\cdot (j-1) \right]^\frac{1}{1-\alpha} .
\end{equation} 
\end{claim}
\begin{proof}[Proof of Claim \ref{claim:aux_oco}]
We prove the required inequality (\ref{eq:aux_oco}) by induction on $j$. Inequality (\ref{eq:aux_oco}) is clearly true for $j = 1, 2$. Now, suppose that inequality (\ref{eq:aux_oco}) is true for some $j \geq 2$, we aim to show that it is also true for $j + 1$. To ease the notation, let's denote $\upsilon  := 2^\frac{\alpha}{1- \alpha} / (1 - \alpha)$, so that the lower bound in (\ref{eq:aux_oco}) is equal to $[(C/\upsilon) \cdot (j - 1)]^\frac{1}{1-\alpha}$. Now, 
\begin{align}
\rho_{j+1} &\geq (C/\upsilon)^\frac{1}{1 - \alpha} (j-1)^\frac{1}{1 - \alpha} + C \cdot (C / \upsilon)^\frac{\alpha}{1 - \alpha} (j-1)^\frac{\alpha}{1 - \alpha}\nonumber\\
& = (C/\upsilon)^\frac{1}{1 - \alpha} \cdot\left[ (j-1)^\frac{1}{1 - \alpha} + \upsilon \cdot  (j - 1)^\frac{\alpha}{1 - \alpha} \right] \nonumber\\
& \geq (C/\upsilon)^\frac{1}{1 - \alpha} \cdot j^\frac{1}{1 - \alpha}. \label{eq:by_mvt}
\end{align} 
Step (\ref{eq:by_mvt}) is justified by considering function $\varphi(j) = j^\frac{1}{1-\alpha}$. By the mean value Theorem, for $j \geq 2$, 
\begin{align}
j^\frac{1}{1 - \alpha} - (j - 1)^\frac{1}{1 - \alpha} &= \varphi'(j - \xi) \cdot (j - (j - 1)) \qquad \text{for some $\xi\in (0, 1)$} \nonumber\\
& = \frac{1}{1 - \alpha} (j - \xi)^\frac{\alpha}{1-\alpha} \leq  \frac{2^\frac{\alpha}{1-\alpha} }{1 - \alpha} (j - 1)^\frac{\alpha}{1-\alpha} = \nu \cdot (j - 1)^\frac{1}{1 - \alpha}\nonumber.
\end{align}
The last inequality holds since $(j - \xi)/(j - 1) \leq 2$, as  $j\geq 2$. Altogether, the claim is proved.
\end{proof}

\subsection{Details for the Frank-Wolfe Oracle $\FW$}\label{app:fw}
Combining the explicit upper bound $M^\FW(T)$ in the proof of Lemma \ref{lemma:fw_bd_M_brief} with Proposition \ref{prop:Err_bd_full} and inequality (\ref{eq:fw_intermediate_bound4}),  we arrive at an $O(\cdot)$ regret bound under OCO oracle $\FW$, for the case of $\beta$-smooth reward functions. With probability at least $1 - 3\delta$,
\begin{align}
\text{Reg}(T) & = O \left(  
\left[  \sqrt{\beta} \left( \frac{LD}{\sqrt{Q}} + \sqrt{Q} \right) \|\mathbf{1}_K\|^{3/2} + L \| \mathbf{1}_K\| D \sqrt{ \Gamma S A \log\frac{KSAT}{\delta}}\right] \cdot \frac{1}{\sqrt{T}}
 \right.\nonumber\\
&\qquad\qquad   + \left. \left[  \beta \|\mathbf{1}_K\|^2 \log T + 
Q \|\mathbf{1}_K\| S A  \log T + L \| \mathbf{1}_K \| D S^2 A \log^2 \frac{KSAT}{\delta}\right] \cdot \frac{1}{T} \right).\label{eq:fw_bd_full}
\end{align}


\subsection{Analysis for the Tuned Gradient Descent Oracle $\TGD$}\label{app:tgd}
In this Appendix Section, we analyze the Tuned Gradient Oracle $\TGD$. In Section \ref{app:pfthmTGD}, we prove Theorem \ref{thm:TGD}. In Section \ref{app:pftgdbdM}, we prove the following Lemma that establishes the upper bound $M^\TGD(T)$ on $m^\TGD(T)$, which is the index of the episode that contains time step $T$:
\begin{restatable}{lemma}{lemmafgdbdM}\label{lemma:fgd_bd_M}
Consider an execution of Algorithm {\sc Toc-UCRL2} with the OCO oracle $\TGD$ and gradient threshold $Q>0$. With certainty, we have
$$m^\TGD(T) \leq M^\TGD(T) :=   1 + \left( \frac{Q}{2L}\right)^{3/4} +  \sqrt{\frac{9L}{Q}} \cdot T^{2/3} +  SA ( 1+ \log_2 T ). $$
\end{restatable}
To start, let's recall Theorem \ref{thm:TGD}:
\thmtgd* 
\subsubsection{Proof of Theorem \ref{thm:TGD} for the regret bound under $\TGD$}\label{app:pfthmTGD}
The proof of Theorem \ref{thm:TGD} requires the following Theorem on online gradient descent by \citep{Zinkevich03}:
\begin{theorem}[\cite{Zinkevich03}]\label{thm:zinkevich}
Let $f_1, \ldots, f_T$ be a sequence of convex and $R'$-Lipschitz function w.r.t. the Euclidean norm $\|\cdot\|_2$ on domain $B(L, \|\cdot\|_2)$. In addition, let $\{\eta_t\}^\infty_{t=1}$ be a sequence of non-increasing learning rates. Consider the OGD algorithm: For each $t = 1, 2, \ldots, $ perform
$$
\theta_{t+1}\leftarrow \text{Proj}_L\left(\theta_t - \eta_t \nabla f_t(\theta_t)   \right).
$$
The following inequality holds:
\[
\qquad \qquad \qquad
\frac{1}{T}\sum^T_{t=1}f_t(\theta_t) - \min_{\theta\in B(L, \|\cdot\|_2)} \left\{\frac{1}{T}\sum^T_{t=1}f_t(\theta)\right\}\leq \frac{L^2}{2T\eta_T} + \frac{R'^2 \sum^T_{t=1}\eta_t}{2T}. \qquad \qquad \qquad \blacksquare
\]
\end{theorem}
We now prove Theorem \ref{thm:TGD}. Recall the notations $v^* = \sum_{s\in \SSS, a\in \AAA_s}v(s,a)x^*(s, a)$, and define the notation $(\text{Err}_t) := (-\theta_t)^\top [v^* - V_t(s_t, a_t) ]$. Conditioned on events ${\cal E}^v, {\cal E}^p$ (see eqns (\ref{eq:event_v}, \ref{eq:event_p})),
\begin{align}
g\left(\frac{1}{T}\sum^T_{t=1}V_t(s_t, a_t)  \right)= & \min_{\theta \in B(L, \|\cdot\|_*)}\left\{ g^*(\theta) - \theta^\top\left[ \frac{1}{T}\sum^T_{t=1}V_t(s_t, a_t) \right] \right\} \nonumber\\
\geq & \frac{1}{T}\left[\sum^T_{t=1}g^*(\theta_t) - \theta_t^\top V_t(s_t, a_t)\right] - \left[\frac{L \|\mathbf{1}_K\| }{2 T^{1/3}} + \frac{2L \|\mathbf{1}_K\| }{T}\sum^{T}_{t=1} \frac{1}{t^{2/3}}\right] \label{eq:fgd_by_zinkevich}.\\
\geq & \frac{1}{T}\left[\sum^T_{t=1}g^*(\theta_t) - \theta_t^\top V_t(s_t, a_t)\right] - \frac{3 L \|\mathbf{1}_K\| }{ T^{1/3}} \nonumber\\
= & \frac{1}{T}\left[\sum^T_{t=1}g^*(\theta_t) - \theta_t^\top v^* \right] - \frac{1}{T}\sum^T_{t=1}(\text{Err}_t) - \frac{3L \|\mathbf{1}_K\| }{ T^{1/3}} \nonumber\\
\geq & \min_{\theta\in B(L, \|\cdot\|_*)}\left\{g^*(\theta) -\theta^\top v^* \right\}  - \frac{1}{T}\sum^T_{t=1}(\text{Err}_t) - \frac{3 L \|\mathbf{1}_K\| }{ T^{1/3}} \nonumber\\
= & \text{opt}(\primal_{\cal M})  - \frac{1}{T}\sum^T_{t=1}(\text{Err}_t) - \frac{3 L \|\mathbf{1}_K\| }{ T^{1/3}}.\label{eq:tgd_by_duality}
\end{align}
Step (\ref{eq:fgd_by_zinkevich}) is by applying Theorem \ref{thm:zinkevich} with $f_t(\theta) = g^*(\theta) - \theta^\top V_t(s_t, a_t)$, and  $\eta_t = L / (\left\| \mathbf{1}_K \right\|  t^{2/3})$. The function $g^*(\theta) - \theta^\top V_t(s_t, a_t)$ is $2\|\mathbf{1}_K\| $-Lipschitz w.r.t. the norm $\|\cdot\|_*$, since 
\begin{align}
&\left\| \nabla [ g^*(\theta) - \theta^\top V_t(s_t, a_t) ] \right\|_* \leq \left\|\underset{w\in [0, 1]^K}{\text{argmax}}\left\{\theta^\top w + g(w)\right\}\right\| + \left\|V_t(s_t, a_t)\right\|\leq 2\|\mathbf{1}_K\| \label{eq:dual_lip} .
\end{align}
The domain of each $f_t$ is $B(L, \|\cdot\|_*) = B(L, \|\cdot\|)$. Altogether, we can put $L' = L$, $R' = 2\|\mathbf{1}_K\| $ for applying the Theorem, which achieves the step. By Proposition \ref{prop:Err_bd_full}, with probability at least $1 - 3\delta$ we have
\begin{align}
&\sum^T_{t=1} (\text{Err}_t) = O\left( (L D +  Q) \|\mathbf{1}_K\|  M^\TGD(T)+ L \| \mathbf{1}_K \| D\sqrt{\Gamma SAT\log\frac{KSAT}{\delta}} \right.\nonumber\\
&\qquad\qquad \qquad \qquad \qquad \qquad \qquad \qquad\qquad \qquad \qquad \qquad \quad \left.  + L \| \mathbf{1}_K \| DS^2A\log^2\frac{KS AT}{\delta} \right).\label{eq:detailed_symbols_bd_tgd}
\end{align}
Finally, we combine the bound (\ref{eq:detailed_symbols_bd_tgd}) with the bound (\ref{eq:tgd_by_duality}), along with the definition of $M^\TGD(T)$ in Lemma \ref{lemma:fgd_bd_M}, yielding the final regret bound 
\begin{align}
\text{Reg}(T) & = O \left(
\left[ \frac{L^{3/2}D}{\sqrt{Q}} + \sqrt{LQ} \right] \cdot \|\mathbf{1}_K\|  \cdot \frac{1}{T^{1/3}} +  L \| \mathbf{1}_K\| D \sqrt{ \Gamma S A \log\frac{KSAT}{\delta}} \cdot \frac{1}{\sqrt{T}}
 \right.\nonumber\\
&\qquad\qquad \qquad\qquad  + \left. \left[
Q \|\mathbf{1}_K\| S A  \log T + L \| \mathbf{1}_K \| D S^2 A \log^2 \frac{KSAT}{\delta}\right] \cdot \frac{1}{T}
\right)\label{eq:tgd_bd_full}
\end{align}
that holds with probability at least $1-3\delta$, hence proving the Theorem \ref{thm:TGD}. \hfill $\blacksquare$

\subsubsection{Proof of Lemma \ref{lemma:fgd_bd_M} on $M^\TGD(T)$}\label{app:pftgdbdM}
\begin{proof}[Proof of Lemma \ref{lemma:fgd_bd_M}]
We consider the following two sets of episode indexes:
\begin{align}
{\cal M}^\TGD_\Psi(T) &:= \left\{m\in \mathbb{N}: \tau(m) \leq T\text{, episode $m+1$ is started due to } \Psi \geq Q\right\}, \nonumber\\
{\cal M}^\TGD_\nu(T) & := \left\{m\in \mathbb{N}: \tau(m) \leq T\text{, episode $m+1$ is started due to } \right. \nonumber\\
&\qquad\qquad\quad\;\, \left. \nu_m(s_t, \tilde{\pi}_m(s_t)) \geq N^+_m(s_t, \tilde{\pi}_m(s_t))\text{ for some $t\geq \tau(m)$}\right\} \nonumber. 
\end{align}
To ease the notation, we define the shorthand $n_\Psi :=| {\cal M}^\TGD_\Psi(T) |$. To prove the Lemma, it suffices to show that
\begin{align}
n_\Psi = | {\cal M}^\TGD_\Psi(T) | & \leq  M^\TGD_\Psi(T) := 1 + \left( \frac{Q}{2L}\right)^{3/4} +  \sqrt{\frac{9L}{Q}} \cdot T^{2/3} , \label{eq:fgd_bd_M_Psi}\\
| {\cal M}^\TGD_\nu(T) | & \leq M^\TGD_\nu(T) := SA (1 + \log_2 T) , \label{eq:fgd_bd_M_nu}
\end{align} 
hold with certainty. The proof of inequality (\ref{eq:fgd_bd_M_nu}) is identical to the proof of inequality (\ref{eq:fw_bd_M_nu}) in the proof of Lemma \ref{lemma:fw_bd_M}. Thus, in the remaining, we focus on proving inequality (\ref{eq:fgd_bd_M_Psi}).

Let's express ${\cal M}^\TGD_\Psi(T) = \{m_1, m_2, \ldots, m_{n_\Psi}\}$, where $m_1 < m_2 < \ldots < m_{n_\Psi}$, and define $m_0 = 0$. We focus on a fixed episode index $m_j$ with $j\geq 1$, and consider for each time step $t\in \{\tau(m_j)+1, \ldots, \tau(m_j+1)\}$ the difference: 
\begin{align}
\left\| \theta_t - \theta_{\tau(m_j)}\right\|_2 & \leq \sum^{t-1}_{q = \tau(m_j)} \left\| \theta_{q + 1} - \theta_q\right\|_2 \nonumber\\
& = \sum^{t-1}_{q = \tau(m_j)} \left\| \text{Proj}_L \left( \theta_q  - \frac{L}{\|\mathbf{1}_K\| q^{2/3}}\left[\nabla_\theta g^*(\theta_q) - V_q(s_q, a_q) \right] \right) - \theta_q\right\|_2 \nonumber\\
& \leq \sum^{t-1}_{q = \tau(m_j)} \left\|\theta_q - \frac{L}{\|\mathbf{1}_K\| q^{2/3}}\left[\nabla_\theta g^*(\theta_q) - V_q(s_q, a_q) \right] - \theta_q \right\|_2 \nonumber\\
& \leq 2 L \sum^{t-1}_{q = \tau(m_j)}  \frac{1}{q^{2/3}} \leq 2L \cdot \frac{t -\tau(m_j)}{\tau(m_j)^{2/3}}.\label{eq:fgd_bd_step_1}
\end{align}
By the fact that $m_j\in {\cal M}^\TGD_\Psi(T)$, we know that $\sum^{\tau(m_j +1) }_{t = \tau(m_j )} \| \theta_t - \theta_{\tau(m_j )}  \|  > Q$. By applying the upper bound (\ref{eq:fgd_bd_step_1})
, we have
\begin{equation*}
2 L\cdot  \frac{(\tau(m_j +1) - \tau(m_j ))^2}{\tau(m_j )^{2/3}} \geq 2L  \sum^{\tau(m_j +1)}_{t = \tau(m_j )} \frac{t -\tau(m_j)}{\tau(m_j)^{2/3}}  \geq Q,
\end{equation*}
which implies that
\begin{align}
\tau(m_j + 1) & \geq \tau(m_j) + \sqrt{\frac{Q}{2L} } \cdot \tau(m_j)^{1/3}  \geq \tau(m_{j-1} + 1) + \sqrt{ \frac{Q}{2L} }\cdot \tau(m_{j-1} + 1)^{1/3} \label{eq:fgd_bd_step_2},
\end{align}
since $\tau(m) \geq \tau(m')$ for $m \geq m'$, and clearly $m_j \geq m_{j - 1} + 1$. 

Now, we apply Claim \ref{claim:aux_oco} with $C = \sqrt{Q / (2 L )}$, $\alpha  = 1/3$, and $\rho_j =  \tau(m_{\lceil C^{3/2}\rceil +j} + 1)$ for $j = 1, 2, \ldots$. The application is valid, since $\rho_1 = \tau(m_{\lceil C^{3/2}\rceil +1} + 1) \geq C^{3/2} = C^{1 / (1-\alpha) }$, and we are equipped with the recursive inequality (\ref{eq:fgd_bd_step_2}). Consequently, we arrive at
\begin{equation}\label{eq:fgd_bd_step_3}
\tau\left( m_{\lceil C^{3/2} \rceil +j} + 1 \right)\geq \left(\frac{Q}{9L}\right)^{3/4}(j - 1)^{3/2}.
\end{equation}
Finally, if $n_\Psi \leq (Q / (2L))^{3/4}$, then clearly (\ref{eq:fgd_bd_M_Psi}) is established. Otherwise, we put $j = n_\Psi - \lceil (Q / (2L))^{3/4} \rceil - 1$ in inequality (\ref{eq:fgd_bd_step_3}), which gives
\begin{equation}\label{eq:fgd_bd_step_4}
T \geq \tau(m_{n_\Psi })\geq\tau\left(m_{\lceil  C^{3/2} \rceil + [n_\Psi -\lceil  C^{3/2} \rceil - 1] } + 1\right) \geq \left(\frac{Q}{9L}\right)^{3/4}\cdot \left(n_\Psi -  \left( \frac{Q}{2L }\right)^{3/4} -  1\right)^{3/2}.
\end{equation}
Finally, unraveling the bound (\ref{eq:fgd_bd_step_4}) gives the upper bound (\ref{eq:fgd_bd_M_Psi}), and proves the Lemma.
\end{proof}

\subsection{Analysis for the Tuned Mirror Descent Oracle $\TMD$}\label{app:tmd}
In this Appendix Section, we analyze the Tuned Gradient Oracle $\TMD(F, T)$. In Section \ref{app:tmd_background}, we provide a brief review on online mirror descents, and provide some examples of mirror maps. In Section \ref{app:pfthmTMD}, we provide an anytime implementation of {\sc Toc-UCRL2} with $\TMD(F, T)$, as well as a proof of Theorem \ref{thm:TMD}. In Section \ref{app:pflemmatmdbdM}, we establish an upper bound $M^\TMD(T)$ on $m^\TMD(T)$ for each $T$.
\subsubsection{A Brief Background on Mirror Descent}\label{app:tmd_background}
We briefly review the Online Mirror Descent (OMD) algorithm and state the necessary results in the literature for our analysis. 
Throughout the discussion, we follow the presentation in Chapter 2 in \citep{Shalev-Shwartz12}, who surveys OMD. We first recall the definition of strong convexity, which is required for a mirror map $F$. The forthcoming discussions on OMD assume the variable $\theta$ and the norm $\|\cdot\|_*$, since the $\TMD(F, T)$ oracle is applied on the dual space. Recall the notation $\text{dom}(F) = \{\theta\in B(L, \|\cdot\|_*): F(\theta) < \infty\}$.
\begin{definition}
Let $\alpha \geq 0$. 
A function $F: B(L, \|\cdot\|_*)\rightarrow (-\infty ,  \infty]$ is $\alpha$-strongly convex w.r.t. the norm $\|\cdot\|_*$, if the following inequality holds for any $\theta, \vartheta \in  \text{dom}(F)\subseteq B(L, \|\cdot\|_*)$ and $w \in \partial F(\vartheta)$:
$$
F(\theta) \geq F(\vartheta) + 	w^\top (\theta - \vartheta) + \frac{\alpha}{2}\|\theta - \vartheta\|^2_*. $$
\end{definition}
We provide the definition and the performance guarantee of the OMD algorithm, by slightly altering the presentation in \citep{Shalev-Shwartz12} and focusing on the duals.
\begin{theorem}[\cite{NemirovskiY83,Shalev-Shwartz12}]\label{thm:omd}
Let $f_1(\theta), \ldots, f_T(\theta)$ be a sequence of convex and $R'$-Lipschitz function w.r.t. the norm $\|\cdot\|_*$ on $B(L, \|\cdot\|_*)$. In addition, let $F(\theta)$ be a 1-strongly convex function over $B(L, \|\cdot\|_*)$ w.r.t. norm $\|\cdot\|_*$. Consider the following OMD algorithm with learning rate $\eta > 0$: For each $t = 1, 2, \ldots, T$, perform
\begin{equation}\label{eq:omd_looklikedual}
\theta_t  \leftarrow \text{argmax}_{\theta \in \text{dom}(F)}\left\{\theta^\top \left[- \eta\sum^{t-1}_{q = 1} w_q\right] - F(\theta)\right\},
\end{equation}
where $w_q \in \partial f_q(\theta_q)$. 
Let $L'^2 := \max_{\theta\in \text{dom}(F)}F(\theta) - \min_{\theta\in \text{dom}(F) } F(\theta)$. Then
\[
\qquad \qquad \qquad \qquad
\frac{1}{T}\sum^T_{t=1}f_t(\theta_t) - \min_{\theta\in \text{dom}(F) } \left\{\frac{1}{T}\sum^T_{t=1}f_t(\theta)\right\}\leq \frac{L'^2}{T\eta} + \eta R'^2. \qquad \qquad \qquad \qquad \blacksquare
\]
\end{theorem} 
As pointed out in \citep{Shalev-Shwartz12}, the OMD step (\ref{eq:omd_looklikedual}) can be viewed in terms of the Fenchel dual $F^*: [0,1]^K\rightarrow \mathbb{R}$ of $F$; the former is defined as $F^*(w) := \max_{\theta\in \text{dom}(F)} w^\top \theta - F(\theta)$:
\begin{equation}\label{eq:omd_equi}
 \nabla F^*\left(- \eta\sum^{t-1}_{q = 1} w_q\right) = \text{argmax}_{\theta \in \text{dom}(F) }\left\{\theta^\top \left[- \eta\sum^{t-1}_{q = 1} w_q\right] - F(\theta)\right\}.
\end{equation} 
To upper bound the total number of episodes, it is useful to recall the following folklore that relates the notions of strong convexity and smoothness of convex functions. The following fact is established for example in Lemma 15 in \citep{Shalev-Shwartz07}:
\begin{proposition}\label{prop:omd_dual}
Suppose function $F$ is $\alpha$-strongly convex over $B(L, \|\cdot\|_*)$ w.r.t the norm $\|\cdot\|_*$. Then its Fenchel dual $F^*$, is convex and $(1/\alpha)$-smooth w.r.t. $\|\cdot\|$: $F^*$ is differentiable on $[0, 1]^K$, and for any $w_1, w_2\in [0, 1]^K$ ,
$$
\left\|\nabla F^*( w_1) - \nabla F^*( w_2)\right\|_* \leq (1/\alpha ) \left\|w_1 - w_2\right\|.
$$
\end{proposition}
Now, we provide the following Lemma, which provides an upper bound $M^\TMD(T)$ on random variable $m^\TMD(T)$, the index of the episode that contains time step $T$.
\begin{restatable}{lemma}{lemmatmdbdM}\label{lemma:tmd_bd_M}
Consider an execution of Algorithm {\sc Toc-UCRL2} with the OCO oracle $\TMD(F, T)$, over $T$ time steps, with $F$ a mirror map for $(g, \|\cdot\|)$ and gradient threshold $Q>0$ . 
With certainty, 
$$m^\TMD(T) \leq M^\TMD(T) :=  1 + \sqrt{L' / Q} \cdot T^{2/3} +  SA ( 1+ \log_2 T ). $$
\end{restatable}
The Lemma is proved in Appendix \ref{app:pflemmatmdbdM}, after we prove Theorem \ref{thm:TMD} assuming Lemma \ref{lemma:tmd_bd_M} in Appendix \ref{app:pfthmTMD}. Before we dive into the analysis, we briefly review some mirror maps for typical applications. For ${\cal D}\subseteq \mathbb{R}^K$, define the function $I_{{\cal D}}$ as $I_{{\cal D}}(\theta) = 0$ if $\theta \in {\cal D}$, and $I_{{\cal D}}(\theta) = \infty$ if $\theta \not\in {\cal D}$.

\textbf{TGD with Lazy Projection for $\|\cdot\|_2$-Lipschitz Continuity. }Consider the case when $g$ is $L$-Lipschitz continuous w.r.t. the Euclidean norm $\|\cdot\| = \|\cdot\|_2$, which requires a mirror map $F$ for $(g, \|\cdot\|_2)$. An eligible candidate is $F_2(\theta) =  \theta^\top \theta / 2 + I_{B(L, \|\cdot\|_*)}(\theta)$, which has $L' = L / \sqrt{2}$. Now, recall that $\eta^\TMD_T = L' / (\|\mathbf{1}_K\|_2 T^{2/3})$, and the resulting OCO oracle is
\begin{equation}\label{eq:lazy_update}
\theta_{t+1} \leftarrow \text{Proj}_L \left( \frac{L}{\sqrt{2K} T^{2/3} } \left(\sum^{t}_{q = 1} \nabla [g^*(\theta_q)] - V_q(s_q, a_q)\right)\right).
\end{equation}
Although the OCO oracle in (\ref{eq:lazy_update}) is markedly different from $\TGD$, both oracles yield the same regret order bound. Nevertheless, oracle $\TGD$, which does not require the doubling trick, should be empirically better. 

\textbf{Multiplicative Weight Update for $\|\cdot\|_\infty$-Lipschitz Continuity.} In certain applications, such as MOO with fairness among objectives and MDPwK (see Appendix \ref{app:applications_model}), the underlying reward function $g$ is $L$-Lipschitz w.r.t. $\|\cdot\| = \|\cdot\|_\infty$. Additionally, for all $w\in [0, 1]^K$, we have $\partial g(w)\subseteq S_{\geq 0 }(L, \|\cdot\|_* ):= \{\theta: \|\theta\|_* = L, \theta\geq 0\}$, where $\|\cdot\|_* = \|\cdot\|_1.$ In such applications, it suffices to define a mirror map $F$ for $(g, \|\cdot\|_\infty)$ with $\text{dom}(F) \supseteq S_{\geq 0 }(L, \|\cdot\|_1 )$. An eligible candidate is the negative entropy function
\begin{equation}\label{eq:entropy}
F_\infty(\theta) = L \sum^K_{k=1}\theta_k\log\theta_k + I_{S_{\geq 0 }(L, \|\cdot\|_1 )}(\theta),
\end{equation}
where $0\log 0 := 0$. Consequently, $L' = \Theta(L\sqrt{\log K})$, and the resulting OCO oracle is a multiplicative weight update procedure with rate $\Theta(1/T^{2/3})$:
\begin{equation*}
\theta_{t+1} \leftarrow  \frac{(L e^{w_{t, 1} }, \ldots, L e^{w_{t, K}}) }{\sum^K_{k=1}e^{w_{t, k}}}, \text{ where } w_{t, k} := -\frac{\sqrt{\log K}}{ T^{2/3} } \left(\sum^{t}_{q = 1} \nabla [g^*(\theta_q)] - V_q(s_q, a_q)\right).
\end{equation*}

\subsubsection{An anytime implementation with $\TMD(F, T)$, and Proof of Theorem \ref{thm:TMD} for the regret bound}\label{app:pfthmTMD}
When Algorithm {\sc Toc-UCRL2} is applied with the OCO oracle $\TMD(F, T)$, the doubling trick can be applied to guess $T$, which is required for tuning the learning rate $\eta^\TMD_T$. We provide an implementation in Algorithm {\sc Toc-UCRL2-Anytime-$\TMD$}, as displayed in Algorithm \ref{alg:tmd_doubling}. In fact, Algorithm \ref{alg:tmd_doubling} can be empirically improved by allowing a mega-episode to inherit the confidence regions on $v, p$ from the previous one, while retaining the same theoretical guarantee. Nevertheless, we omit the discussion for brevity sake.
\begin{algorithm}[t]
\caption{Algorithm {\sc Toc-UCRL2-Anytime-$\TMD$}}\label{alg:tmd_doubling}
\begin{algorithmic}[1]
\State Input: Parameter $\delta$, mirror map $F$, gradient threshold $Q$, starting state $s_1$.\;
\State Compute starting gradient $\theta_\textsf{start}= \min_{\theta\in \text{dom}(F)}F(\theta)$.
;
\State Initialize $t=1$, mega-episode index $h = 1$, starting state $s_\textsf{start} = s_1$.\;
\State \textbf{for} mega-episode $h = 1, 2, \ldots$ \text{do}
\State \hspace{0.5cm} Set mega-episode length as $2^h$.
\State \hspace{0.5cm} Run {\sc Toc-UCRL2} for $2^h$ time steps with: 
\begin{itemize}
\item Inputs: parameter $\delta_\textsf{start}$, gradient $\theta_\textsf{start}$, gradient threshold $Q$, initial state $s_\textsf{start} $,
\item Input oracles: OCO oracle $\TMD(F, H)$, EVI oracle from Algorithm \ref{alg:evi},
\end{itemize}
\State \hspace{0.5cm} Update $t\leftarrow t + 2^h$.
\State \hspace{0.5cm} Update $s_\textsf{start} \leftarrow s_t$.
\State \hspace{0.5cm} Update $\delta_\textsf{start} \leftarrow \delta/ (2^{h+1})$.
\end{algorithmic}
\end{algorithm}

\thmtmd*
\begin{proof}[Proof of Theorem \ref{thm:TMD}] We start by defining notations for using the doubling trick. Let $H$ be the mega-episode index during time step $T$. For $h = \{1, \ldots H\},$ let events ${\cal E}^v_h, {\cal E}^p_h$ be the events ${\cal E}^v, {\cal E}^p$ on the run of {\sc Toc-UCRL2} during mega-episode $h$. In addition, denote $\text{l}(h)$ as the length of mega-episode $h$ in the first $T$ time steps, and denote $\text{i}(h)$ as the first time step index in mega-episode $h$. Consequently, we have $\text{i}(1) = 1$, and $\text{i}(h + 1) - \text{i}(h ) = \text{l}(h) = 2^h$ for $1\leq h < H$, and $\text{l}(H) = T - (2^H - 2) \leq 2^H$. In addition, note that $2^H -2\leq T\leq 2^{H+1}-2$. 

Consider the error term $(-\theta_t)^\top [ v^* - V_t(s_t, a_t)]$ at time $t$ in Algorithm {\sc Toc-UCRL2-Anytime-$\TMD$}. We analyze each of these error terms by mega-episode. For each mega-episode index $1\leq h \leq H$ and each $t\in \{ 1, \ldots, \text{l}(h)  \}$, we define $(\text{Err}_{h, t}) := (-\theta_q)^\top [ v^* - V_q(s_q, a_q)]$, where $q = i(h) + t - 1$. Supposing that we count time steps starting from time $\text{i}(h)$ , $(\text{Err}_{h, t}) $ is the error term at the time step $t$.  Finally, recall the notation $v^* = \sum_{s\in \SSS, a\in \AAA_s}v(s,a)x^*(s, a)$. Conditioned on events $\{{\cal E}^v_h, {\cal E}^p_h\}^H_{h=1}$, which simultaneously hold with probability $1 - O(\delta)$, we have
\begin{align}
&g\left(\frac{1}{T}\sum^T_{t=1}V_t(s_t, a_t)  \right) = \min_{\theta \in B(L, \|\cdot\|_*)}\left\{ g^*(\theta) - \theta^\top\left[ \frac{1}{T}\sum^T_{t=1}V_t(s_t, a_t) \right] \right\} \nonumber\\
 = &\min_{\theta \in \text{dom}(F)}\left\{ g^*(\theta) - \theta^\top\left[ \frac{1}{T}\sum^T_{t=1}V_t(s_t, a_t) \right] \right\} \label{eq:tmd_by_property_2}\\
 \geq & \frac{1}{T} \sum^H_{h=1} \min_{\theta \in \text{dom}(F)}\left\{ \text{l}(h)\cdot g^*(\theta) - \theta^\top\left[ \sum^{\text{i}(h) + \text{l}(h) - 1}_{t=\text{i}(h)}V_t(s_t, a_t) \right] \right\} \nonumber\\
\geq &  \frac{1}{T} \sum^H_{h=1} \sum^{\text{i}(h) + \text{l}(h) - 1}_{t=\text{i}(h)} \left[ g^*(\theta_t) - \theta_t^\top V_t(s_t, a_t)  \right] - \frac{1}{T}\sum^H_{h=1}\left[\frac{L'^2 }{ \eta^\TMD_{2^h}} + 4\cdot 2^h \cdot \eta^\TMD_{2^h} \|\mathbf{1}_K\|^2 \right]\label{eq:tmd_by_omd}\\ 
\geq & \frac{1}{T} \sum^H_{h=1} \sum^{\text{i}(h) + \text{l}(h) - 1}_{t=\text{i}(h)} \left[ g^*(\theta_t) - \theta_t^\top V_t(s_t, a_t)  \right] - \frac{5}{ 2^{2/3} - 1}\cdot  L' \|\mathbf{1}_K\| \cdot \frac{2^{2(H+1) / 3}}{T} \label{eq:tmd_by_lr}\\
\geq & \frac{1}{T}\sum^H_{h=1} \sum^{\text{i}(h) + \text{l}(h) - 1}_{t=\text{i}(h)} \left[ g^*(\theta_t) - \theta_t^\top v^*  \right] - \frac{1}{T}\sum^H_{h=1} \sum^{ \text{l}(h)}_{t=1} (\text{Err}_{h, t})- \frac{10 L' \|\mathbf{1}_K\| }{ T^{1/3}} \text{ w.p. $\geq 1-\delta$}\label{eq:tmd_by_decompose}\\ 
\geq & \min_{\theta\in \text{dom}(F)}\left\{g^*(\theta) -\theta^\top v^* \right\}  - \frac{1}{T}\sum^H_{h=1} \sum^{\text{l}(h)}_{t=1} (\text{Err}_{h, t})- \frac{10 L' \|\mathbf{1}_K\| }{ T^{1/3}} \nonumber\\
= & \text{opt}(\primal_{\cal M})  - \frac{1}{T}\sum^H_{h=1} \sum^{\text{l}(h)}_{t= 1} (\text{Err}_{h, t}) - \frac{ 10 L' \|\mathbf{1}_K\| }{ T^{1/3}}.\label{eq:tmd_by_duality}
\end{align}
Step (\ref{eq:tmd_by_property_2}) is by Property (2) of a mirror map $F$.
Step (\ref{eq:tmd_by_omd}) is by applying Theorem \ref{thm:omd} for each mega-episode. We apply the Theorem on the sequence of functions $\{f_t\}^{\text{i}(h) + 2^h - 1}_{t = \text{i}(h)}$ defined as $f_t(\theta) = g^*(\theta) - \theta^\top V_t(s_t, a_t)$, and learning rate $\eta = \eta^\TMD_{2^h}$. As established in  Appendix \ref{app:tgd}, each $f_t$ is $2\|\mathbf{1}_K\| $-Lipschitz w.r.t. $\|\cdot\|_*$, resulting in (\ref{eq:tmd_by_omd}). Step (\ref{eq:tmd_by_lr}) is by $\eta^\TMD_q = L' / (\|\mathbf{1}_K\| q^{2/3})$. Step (\ref{eq:tmd_by_decompose}) is by Lemma \ref{lemma:decompose}. Step (\ref{eq:tmd_by_duality}) is by duality and by Property (2) of a mirror map $F$.

By applying Proposition \ref{prop:Err_bd_full}, with probability $1 - O(\delta)$, for each $1\leq h\leq H$ we have
\begin{align}
&\sum^{\text{l}(h)}_{t= 1} (\text{Err}_{h, t}) = O\left( (L D +  Q) \|\mathbf{1}_K\|  M^\TMD(2^h) + L \| \mathbf{1}_K \| D\sqrt{\Gamma SA\cdot 2^h \log\frac{KSAT}{\delta}} \right.\nonumber\\
& \qquad \qquad \qquad \qquad \qquad \left. + L \| \mathbf{1}_K \| DS^2A\log^2\frac{KS AT}{\delta} \right).\nonumber
\end{align} 
Applying the definition of $M^\TMD(T)$ and summing across $1\leq h\leq H$ give the final regret bound that holds with probability $1 - O(\delta)$
\begin{align}
\text{Reg}(T) & = O \left(\left[ L' + \frac{\sqrt{L'}LD}{\sqrt{Q}} + \sqrt{L' Q} \right] \cdot
\|\mathbf{1}_K\| \cdot \frac{1}{T^{1/3}} +  L \| \mathbf{1}_K\| D \sqrt{ \Gamma S A \log\frac{KSAT}{\delta}} \cdot \frac{1}{\sqrt{T}}
 \right.\nonumber\\
&\qquad\qquad \qquad\qquad  + \left. \left[ Q
\|\mathbf{1}_K\| S A  \log T + L \| \mathbf{1}_K \| D S^2 A \log^2 \frac{KSAT}{\delta}\right] \cdot \frac{1}{T}
\right),\label{eq:tmd_bd_full}
\end{align}
hence proving the Theorem \ref{thm:TMD}. 
\end{proof}

\subsubsection{Proof of Lemma \ref{lemma:tmd_bd_M} on $M^\TMD(T)$}\label{app:pflemmatmdbdM}
The proof framework follows the proofs of Lemmas \ref{lemma:fw_bd_M_brief}, \ref{lemma:fgd_bd_M}. We consider the following two sets:
\begin{align}
{\cal M}^\TMD_\Psi(T) &:= \left\{m\in \mathbb{N}: \tau(m) \leq T\text{, episode $m+1$ is started due to } \Psi \geq Q\right\}, \nonumber\\
{\cal M}^\TMD_\nu(T) & := \left\{m\in \mathbb{N}: \tau(m) \leq T\text{, episode $m+1$ is started due to } \right. \nonumber\\
&\qquad\qquad\quad\;\, \left. \nu_m(s_t, \tilde{\pi}_m(s_t)) \geq N^+_m(s_t, \tilde{\pi}_m(s_t))\text{ for some $t\geq \tau(m)$}\right\} \nonumber. 
\end{align}
To ease the notation, we define the shorthand $n_\Psi :=| {\cal M}^\TMD_\Psi(T) |$. To demonstrate the required inequality in the Lemma, it suffices to show that the inequalities
\begin{align}
n_\Psi = | {\cal M}^\TMD_\Psi(T) | & \leq  M^\TMD_\Psi(T) := 1 + \sqrt{ L' / Q }\cdot  T^{2/3} , \label{eq:tmd_bd_M_Psi}\\
| {\cal M}^\TMD_\nu(T) | & \leq M^\TMD_\nu(T) := SA (1 + \log_2 T) , \label{eq:tmd_bd_M_nu}
\end{align} 
hold with certainty. The proof of inequality (\ref{eq:tmd_bd_M_nu}) is identical to the proof of inequality (\ref{eq:fw_bd_M_nu}) in the proof of Lemma \ref{lemma:fw_bd_M}. Thus, in the remaining, we focus on proving inequality (\ref{eq:tmd_bd_M_Psi}). 

Let's express ${\cal M}^\TMD_\Psi(T) = \{m_1, m_2, \ldots, m_{n_\Psi}\}$, where $m_1 < m_2 < \ldots < m_{n_\Psi}$. Let's denote $z_q = \nabla [g^*(\theta_q)] - V_q(s_q, a_q)$, which is the gradient increment at time $q$. We focus on a fixed episode index $m_j$ with $j\geq 1$, and consider for each time step $t\in \{\tau(m_j)+1, \ldots, \tau(m_j+1)\}$ the difference:
\begin{align}
\left\|\theta_t - \theta_{\tau(m_j)}\right\|_* &\leq \sum^{t-1}_{q = \tau(m_j)} \left\|\theta_{q+1} - \theta_q\right\|_* \nonumber\\
& = \sum^{t-1}_{q = \tau(m_j)} \left\|\nabla F^*\left(- \eta^\TMD \sum^{q}_{\ell = 1} z_\ell\right) - \nabla F^*\left(- \eta^\TMD \sum^{q-1}_{\ell = 1} z_\ell\right)\right\|_* \nonumber\\ 
&\leq \sum^{t-1}_{q = \tau(m_j)}\left\| - \eta^\TMD \sum^{q}_{\ell = 1} z_\ell  - \left(- \eta^\TMD \sum^{q-1}_{\ell = 1} z_\ell\right)\right\| \label{eq:by_dual_prop}\\
& = \eta^\TMD \sum^{t-1}_{q = \tau(m_j)} \left\| z_q\right\| \nonumber\\
& \leq 2 (t - \tau(m_j))\cdot  \eta^\TMD \|\mathbf{1}_K\| = \frac{2 L' (t - \tau(m_j)) }{T^{2/3}}.\label{eq:tmd_by_dual_lip}
\end{align}
Step (\ref{eq:by_dual_prop}) is by an application of Proposition \ref{prop:omd_dual} on the mirror map $F$, which is 1-strongly convex w.r.t. $\|\cdot\|$. Step (\ref{eq:tmd_by_dual_lip}) is by the fact that $\| z_q\| = \| \nabla [g^*(\theta_q)] - V_q(s_q, a_q)\| \leq 2\|\mathbf{1}_K\|$.  

Since $m_j\in {\cal M}^\TMD_\Psi(T)$, we have $\sum^{\tau(m_j +1) }_{t = \tau(m_j )} \| \theta_t - \theta_{\tau(m_j )}  \|_*  > Q$. Applying (\ref{eq:tmd_by_dual_lip}), we then have
\begin{align}
\frac{ L' (\tau (m_{j + 1} ) - \tau(m_j) )^2 }{T^{2/3}} \geq &\frac{ L' (\tau (m_j + 1) - \tau(m_j)  )^2 }{T^{2/3}} \nonumber\\
\geq &\frac{2 L' }{T^{2/3}} \sum^{\tau(m_j +1) }_{t = \tau(m_j )} (t - \tau(m_j)) \geq \sum^{\tau(m_j +1) }_{t = \tau(m_j )} \| \theta_t - \theta_{\tau(m_j )}  \|_*  \geq Q, \nonumber
\end{align}
leading to
\begin{equation*}
\tau (m_{j + 1} ) - \tau(m_j ) \geq \sqrt{Q / L'  } \cdot T^{1/3} ~ \Rightarrow ~ \tau (m_j ) \geq (j - 1)\sqrt{ Q / L'  }\cdot T^{1/3} +1
\end{equation*}
for each $j \geq 1$. Consequently, putting $j = n_\Psi$ gives
$$
T \geq \tau(m_{n_\Psi}) \geq (n_\Psi - 1)\sqrt{ Q / L' }\cdot T^{1/3}.
$$
Unraveling for $n_\Psi$ gives the bound (\ref{eq:tmd_bd_M_Psi}), and proves the Lemma.\hfill $\blacksquare$


\end{document}